%% file: MAIN_Full.tex
\documentclass[11pt]{article}

\usepackage{times}
\usepackage{latexsym}
\usepackage{amsfonts,amsthm,amssymb}
\usepackage{amsmath}
\usepackage{euscript}
\usepackage{amstext}
\usepackage{graphicx}
\usepackage{xcolor}
\usepackage{lipsum}
\usepackage{url,hyperref}
\usepackage{enumerate}
\usepackage{subcaption}
\usepackage{wrapfig}
\usepackage[square,numbers]{natbib}
\usepackage{caption}

\allowdisplaybreaks
\usepackage{mathtools}
\mathtoolsset{showonlyrefs}

\usepackage[ruled,noend]{algorithm2e}

\newtheorem{property}{Property}%[section]
\newtheorem{lemma}{Lemma}%[section]
\newtheorem{theorem}[lemma]{Theorem}

\newtheorem{definition}[lemma]{Definition}
\newtheorem{corollary}[lemma]{Corollary}

\newtheorem{claim}[lemma]{Claim}

\renewcommand{\proofname}{\textbf{Proof}}
{\hspace*{\fill}$\Box$\par}

\newcommand{\etal}{et al.\ }

\newcommand{\predm}{\hat{m}}
\newcommand{\predI}{\hat{I}}
\newcommand{\predG}{\hat{G}}
\newcommand{\predw}{\hat{\alpha}}
\newcommand{\predx}{\hat{x}}
\newcommand{\predA}{\widehat{Alloc}}
\newcommand{\learnA}{\hat{A}}

\newcommand{\cA}{{\cal A}}

\newcommand{\cD}{{\cal D}}

\newcommand{\cG}{{\cal G}}

\newcommand{\cI}{{\cal I}}

\newcommand{\cS}{{\cal S}}

\newcommand{\cW}{{\cal W}}

\newcommand{\g}{{\cal G}}
\newcommand{\C}{{\cal C}}
\newcommand{\D}{{\cal D}}

\newcommand{\E}{\mathbb{E}}
\newcommand{\R}{\mathbb{R}}

\newcommand{\Z}{\mathbb{Z}}

\newcommand{\bfo}{\mathbf{1}}

\newcommand{\alg}{\mathrm{ALG}}
\newcommand{\opt}{\textsc{OPT}{}}
\newcommand{\eps}{\epsilon}

\newcommand{\Gr}{\tilde{G}}
\newcommand{\Vr}{\tilde{V}}
\newcommand{\Er}{\tilde{E}}
\newcommand{\vr}{\tilde{v}}
\newcommand{\ur}{\tilde{u}}

\DeclareMathOperator{\val}{Val}
\DeclareMathOperator{\obj}{Obj}
\newcommand{\true}{\mathrm{True}}
\newcommand{\false}{\mathrm{False}}
\newcommand{\cond}{\mathbf{Cond}}
\newcommand{\poly}{\mathrm{poly}}
\newcommand{\lev}{\mathrm{Lev}}

\begin{document}

\title{Learnable and Instance-Robust Predictions for Online Matching, Flows and Load Balancing} %TODO Please add

\author{Thomas Lavastida\thanks{Carnegie Mellon University. 5000 Forbes Ave, Pittsburgh, PA 15213. Email: \{tlavasti, moseleyb\}@andrew.cmu.edu. Supported in part by a Google Research Award, an Infor Research Award, a Carnegie Bosch Junior Faculty Chair and NSF grants CCF-1824303,  CCF-1845146, CCF-1733873 and CMMI-1938909. }  \and Benjamin Moseley\footnotemark[1] \and R. Ravi\thanks{Carnegie Mellon University. 5000 Forbes Ave, Pittsburgh, PA 15213. Email: ravi@andrew.cmu.edu. Supported in part by the U. S. Office of Naval Research under award number N00014-21-1-2243 and the Air Force Office of Scientific Research under award number FA9550-20-1-0080.} \and Chenyang Xu\thanks{Zhejiang University. Hangzhou, Zhejiang, China 310007. xcy1995@zju.edu.cn.  Chenyang Xu is the corresponding author. Supported in part by Science and Technology Innovation 2030 –"The Next Generation of Artificial Intelligence" Major Project No.2018AAA0100902 and China Scholarship Council No.201906320329.}}
\date{}
\maketitle


\input{abstract}

\clearpage
\setcounter{page}{1}
\newpage

%\input{intro}
%\newpage
%\setcounter{page}{1}
\input{intro_v2}

\input{model}

% main body on flows
\input{flows_existence}
\input{flows_learnability}
\input{flows_robustness}

\input{scheduling_results}

\newpage

\bibliographystyle{plainnat}
\bibliography{proportional_flows}
\newpage
\appendix

\input{existweight}

\input{learnability}

\input{robustness}

% need to add lower bound for integral flows
% need to add scheduling stuff

\input{apdxexistweight}

\input{apdxlearnability}

\input{robustness_app}	

\input{apdx_scheduling}

\end{document}

%% file: abstract.tex
\begin{abstract}

%This paper proposes a new model for augmenting algorithms with useful predictions that go beyond worst-case bounds on the algorithm performance.  By  refining existing models, our model ensures predictions are formally learnable and instance robust.  
We propose a new model for augmenting algorithms with predictions by requiring that they are formally learnable and instance robust. 
Learnability ensures that predictions can be efficiently constructed from 
a reasonable amount of past data.  
Instance robustness ensures that the prediction is robust to modest changes in the problem input, where the measure of the change  may be problem specific. Instance robustness insists on a smooth degradation in performance as a function of the change. Ideally, the performance is never  worse than worst-case bounds. This also allows predictions to be  objectively compared.

%We propose a new model for augmenting algorithms with predictions by requiring that they are formally learnable and instance robust. Learnability ensures that predictions can be efficiently constructed from a reasonable amount of past data.  Instance robustness ensures that the prediction is robust to modest changes in the problem input: it insists on a smooth degradation in performance as a function of the change while never being worse than worst-case bounds. \ben{We remark that this measure of the change in the input is naturally problem specific.}  This also allows predictions to be more objectively compared to prior models of robustness.

We design online algorithms with predictions for a network flow allocation problem and restricted assignment makespan minimization. For both problems, two key properties are established:  high quality predictions can be learned from a small sample of prior instances and  these predictions are robust to errors that smoothly degrade as the underlying problem instance changes.

\end{abstract} 

%% file: intro_v2.tex
\section{Introduction}

Inspired by advances in machine learning, there is an interest in augmenting algorithms with predictions, especially in online algorithm design~\cite{DBLP:journals/siamcomp/GuptaR17,DBLP:conf/icml/BalcanDSV18,DBLP:conf/nips/BalcanDW18,ChawlaGTTZ19,Kraska,MLCachingLykouris,PurohitNIPS,DBLP:conf/soda/LattanziLMV20}. Algorithms augmented with predictions have had  empirical success in domains such as look-up tables~\cite{Kraska}, caching~\cite{MLCachingLykouris}, and bloom-filters~\cite{Mitzenmacher}.
%\ravi{These successes and the availability of predictions motivate the development of a possibility of a new model that goes beyond worst-case bounds when supplied with accurate predictions.}
These successes and the availability of data to make predictions using machine learning have motivated the development of new analysis models for going beyond worst-case bounds where an algorithm is supplied with accurate predictions.
%These empirical successes have motivated the development of a new model to theoretically analyze algorithms with predictions, with the goal of going beyond worst-case bounds in the presence of accurate predictions.  
In these models, an algorithm is given access to a prediction about the problem instance.  
The algorithm's performance is bounded in terms of the quality of this prediction.
%Typically in such a model, an algorithm is given access to an \emph{error prone} prediction.  This prediction can be leveraged to make algorithmic decisions.   Ideally, (1) the predictions result in better performance than the best worst-case bound.  We say such predictions are \emph{useful}; (2) the algorithm never performs asymptotically worse than the best worst-case algorithm even if the prediction error is large; and (3) in-between,  there is a graceful degradation in performance as the error in the prediction becomes worse. For example, the competitive ratio or running time can be parameterized by the prediction error.   See \cite{MitzenmacherVassilvitskii} for a survey.
Typically the algorithm learns such predictions from a limited amount of past data leading to  \emph{error-prone} predictions.
%thatcan be leveraged to make decisions. 
The algorithm with accurate predictions should result in better performance than the best worst-case bound. Ideally, the algorithm never performs asymptotically worse than the best worst-case algorithm even if the prediction error is large.  In-between,  there is a graceful degradation in performance as the prediction error increases. For example, competitive ratio or running time can be parameterized by prediction error. See \cite{MitzenmacherVassilvitskii} for a survey.
% changing to plural "models" here is confusing, previous paragraph alludes to one model

%This model is a widely applicable framework for beyond-worst-case analysis.  Still, our understanding of this model is in its infancy.  There is a need to refine this model to ensure predictions can be formally learned and there is room to explore new and potentially better ways to measure prediction error.
%This paper presents a new refinement of the algorithms augmented with predictions model.

\subparagraph*{Learnable and Instance Robust Predictions:}  
%\ravi{Note that the definition of what should be predicted to improve the performance of an algorithm is itself a novel creative contribution in many of these models. In addition,}

The model proposed in this paper has two pillars for augmenting algorithms with predictions. (\textbf{Learnability}:) Predictions should be learnable from representative data.  (\textbf{Instance Robustness}\footnote{Note that the robustness here is different from the definition of robustness mentioned in previous work, which we refer to as parameter robustness. 
See Section~\ref{sec:model} for a discussion.
%A further discussion is deferred to Section~\ref{sec:model}.
}:) Predictions should be robust to minor changes in the problem instance. As in prior models, determining what to predict remains a key algorithmic challenge.

Suppose  there is an unknown distribution $\cD$ over instances $\cI$.  Building on  data driven algorithm design~\cite{DBLP:journals/siamcomp/GuptaR17} and PAC-learning models, we require that predicted parameters are provably learnable  using a small number of sample instances from $\cD$.  The sample complexity of this task can be used to compare how difficult different predictions are to construct.  Practically, the motivation is that parameters are learned from  prior data (e.g.  instances of the problem). 

%In practice, future problem instances may not come from the distribution $\cD$ the parameters are learned from. 
In practice, future problem instances may not come from the same distribution used to learn the parameters.
Therefore, we also desire predictions that are robust to modest changes in the input. 
In particular, if the predictions perform well on some instance $\cI$, then the performance on a nearby instance $\cI'$ should be bounded as a function of the distance between these instances.  This measure of the distance between instances is necessarily problem specific.
We note that learnability is rarely addressed in prior work and our robustness model differs from many prior works by bounding the error by differences in problem instances (instance robustness), rather than by the differences in the predictions themselves (parameter robustness).
We present learnable and instance-robust predictions for two concrete online problems.
%Next, we consider two concrete online problems in this new model to demonstrate what can be achieved under it.

%\smallskip
%\noindent \textbf
\subparagraph*{Online Flow Allocation Problem:} We consider a general flow and matching problem.   The input is a Directed-Acyclic-Graph (DAG) $G$. Each node $v$  has an associated capacity $C_v$.  There is a sink node $t$, such that all nodes in the DAG can reach the sink.
%For the development of our results, it is instructive to consider layered DAGs where all arcs go from one layer to the next, where the layers are defined by the longest path to the sink (or from a dummy source node connected to all the sources).
Online source nodes arrive that have no incoming edges (and  never have any incoming edges in the future) and the other nodes are offline and fixed.  We will refer to online nodes $I$ as impressions.  
When impression $i$ arrives, it is connected to a subset of nodes $N_i$.  At arrival, the algorithm must decide a (fractional) flow from $i$ to $t$ of value at most 1 obeying the node capacities.  This flow is fixed the moment the node arrives.  The goal is to maximize the total flow that reaches $t$ without exceeding node capacities.
Instances are defined by the number of each \emph{type} of impression.  The type of an impression is given by the subset of the nodes of $G$ to which it has outgoing arcs.
We may consider specific worst-case instances or a distribution over types in our analysis.
%In the latter case, we make two assumptions: the unknown distribution is a product distribution, and the optimal solution of the “expected instance” (to be defined later) has at least a constant amount of flow routed through each node.
This problem captures fractional versions of combinatorial problems such as online matching, unweighted Adwords, and finding a maximum independent set in a laminar matroid or gammoid.  We call the problem the Online Flow Allocation Problem.

%\smallskip
%\noindent\textbf
\subparagraph*{Restricted Assignment Makespan Minimization:} In this problem, there are $m$ machines and $n$ jobs arrive in an online order.  When job $j$ arrives it must be immediately and irrevocably assigned to a machine.  The job has size $p_j$ and can only be assigned to a subset of machines specific to that job.  After $j$ is assigned the next job arrives.  A machine's load is the total size of all jobs assigned to it.  The goal is to minimize the makespan, or maximum load, of the assignment.

%This  theoretical model holds the promise of giving a systematic approach for going beyond-worst-case analysis.  While exciting, results shown in the model need to be carefully considered as the model can be abused.   Models developed so far are critically missing formal approaches for establishing which predictions are reasonable, the ability to formally compare different measures of prediction error.

%This paper presents a refinement of the algorithms augmented with predictions model designed to overcome these shortcomings.
%We consider a general class of flow allocation problems and load balancing in the restricted assignment settings where fractional assignments are allowed. Te flow allocation problem captures unweighted bipartite matching   Building on the work of \cite{LattanziLMV20} we establish that there is a vector of weights $\beta$ that can be predicted.  There is one weight per node in the flow problem or one weight per machine in the scheduling problem.

%\medskip \noindent \textbf{Open Questions:} The algorithms augmented with predictions model has revealed there is a broad potential to circumvent worst-case lower bounds, especially in the online setting using predictions.  The new question is that now that we have restricted ourselves to predictions that are learnable and distributionally robust, can we achieve similar results?

\subsection{Overview of Results for Flow Allocation and Restricted Assignment}

%This paper shows that there are %predictions 
%algorithmic parameters
%that can be efficiently learned and are instance robust for both applications. 
We first focus on the flow allocation problem and then we give similar results for the makespan minimization problem.

\subparagraph*{Node Parameters:} Our results on learnability and robustness are enabled by showing the existence of node weights which capture interesting combinatorial properties of flows.  
Inspired by the weights proven in~\cite{PropMatchAgrawal} for bipartite matching, we establish that there is a single weight 
%that can be predicted 
for each node in the DAG that completely describe near optimal flows on a single problem instance.  
%These weights are in the same spirit as weights used in other works on matching~\cite{PropMatchAgrawal} and scheduling~\cite{DBLP:conf/soda/LattanziLMV20}.  
%The weights can be used to allocate flow for each impression such that the flow for each impression is deterministic and computed independently of other impressions.
The weights determine an allocation of flow for each impression which is independent of the other impressions.
%This flow is routed by splitting the ongoing flow to the sink in a proportional manner based on the weights of a node's descendants and their distances from a dummy source connected to all the source nodes (generalizing their layer number).
Each node in the DAG routes the flow leaving it proportionally to the weights of its outgoing neighbors.
Moreover, the flow is near optimal, giving a $(1-\eps)$-approximate solution for any constant $\eps >0$ (but requiring time polynomial in $1/\epsilon$ to compute). Given these weights, the flow can be computed in one forward pass for each impression in isolation.  Thus they can be used online if given as a prediction. These weights are also efficiently computable offline given the entire problem instance (see Theorem~\ref{thm:weight_existence_3_layer}).

%These predictions are useful and give better than worst-case bounds.
%We show the following theorem regarding online flow allocation.  
%The predictions break through worst-case bounds as a $1-\frac{1}{e}$ bound is known on any randomized algorithm for online fractional matching that the problem generalizes.  Moreover, for integral flows we show a $1/(d+1)$ bound on any deterministic online algorithm for the case of a $d$-layer graph.
 % For any problem instance, we show there exist node weights that can be used to construct a near optimal allocation. Moreover, if given the weights in advanced, the allocation can be constructed online.  Knowing such weights exist,  we predict a single weight for each node of the DAG and the online algorithm uses these weights to construct the allocation.  
  
 % We show that these weights are an ideal set of parameters to predict by proving the weights are learnable and instance robust, which are defined as follows. 
%}

%\smallskip
%\noindent \textbf
\subparagraph*{Instance Robustness:}    %Instances are defined by the number of each \emph{type} of impression.  The type of an impression is given by the subset of $V$ to which it has outgoing arcs. 
We measure the distance of the two instances as the difference of the number of impressions of each type (see Theorem~\ref{thm:dist_robustness}). We show that if the weights are near optimal for one instance, the performance degrades gracefully according to the distance between the two instances.  
This distance is defined for any two instances irrespective of whether they are generated by specific distributions
%and hence instance robustness is a worst-case guarantee.
\footnote{We also show that our predictions for the online flow allocation problem have ``parameter robustness", the kind of robustness that has been considered in prior work (Theorem~\ref{thm:flow_param-rob}).}.

\subparagraph*{Learnability:} 

For learnability it is assumed that impressions are drawn from an  unknown distribution over types.  We show that learning near-optimal weights for this distribution 
%weights that give a near  optimal solution in expectation 
has low sample complexity under two assumptions.  First, we assume the unknown distribution is a product distribution.  Second, we assume that the optimal solution of the ``expected instance'' (to be defined later) has at least a constant amount of flow routed through each node.  In the 2-layer case, this assumption can be simplified to requiring each node's capacity to be at least a constant (depending on $\frac{1}{\epsilon}$).\footnote{This is similar to the lower bound requirement on budgets in the online analysis of the AdWords problem~\cite{DBLP:conf/sigecom/DevenurH09}.} The number of samples is polynomial in the size of the DAG without the impressions. Note that in problems such as Adwords, the impressions are usually much larger than the fixed portion of the graph.

% in Section~\ref{sec:flows_robustness}. 
%\rr{More generally, we can insist that if there are single set of predictions for a class of distributions over the instances, then the same predictions will perform reasonably well on instances from a nearby distribution. We call such types of robustness requirements as ``instance robustness".}

We now present our main theorem on the flow allocation problem. 

\begin{theorem}[Flow Allocation - Informal]
There exist algorithmic parameters for the Online Flow Allocation problem with the following properties:
%\begin{itemize}
%\itemsep0em
\begin{enumerate}[(i)]
%\item \ben{(Beyond Worst-Case Bounds) Optimal parameters can be used by an algorithm to give a $(1-\eps)$-approximation for any  constant $\eps >0$. 
\item (Learnability) Learning near-optimal parameters has  sample complexity polynomial in $\frac{1}{\epsilon}$ and the size of the graph  excluding the impressions. These parameters result in an online algorithm that is a  $(1-\eps)$-approximate solution in expectation as compared to the expected optimal value on the distribution  for any  constant $\eps >0$. (Theorem~\ref{thm:DAG-learnability})
\item (Instance Robustness) 
%\tl{There exists a metric on the instances such that} using the optimal parameters for any instance on a nearby instance gives a competitive ratio that improves as their distance (w.r.t. the metric) decreases. (Theorem~\ref{thm:dist_robustness})
Using the optimal parameters for an instance on another instance gives a competitive ratio that improves as their distance decreases, where the distance is proportional to the difference of impressions (Theorem~\ref{thm:dist_robustness}).
%and never worse than $\frac{1}{d+1}$ no matter how different the instances are for a $d$-layer graph. 
\item (Worst-Case Robustness) The competitive ratio of the online algorithm using the parameters is never worse than $\frac{1}{d+1}$, regardless of the distance between the two instances, where $d$ is the diameter of $G$. (Theorem~\ref{thm:dist_robustness})
%%Predicting the parameters gives an instance-robust algorithm, whose competitive ratio is bounded by the difference of the two instances and never worse than $\frac{1}{d+1}$ no matter how different the instances are for a $d$-layer graph.
%The predictions are instance robust and the objective is never smaller than a  $\frac{1}{d+1}$ of optimal no matter how inaccurate the error is for a $d$-layer graph. (Theorem~\ref{thm:dist_robustness}) 
%whose competitive ratio gives a  never worse than $\frac{1}{d+1}$ no matter how inaccurate the error is for a $d$-layer graph.
%Predicting the parameters gives an instance-robust algorithm, whose competitive ratio is bounded by the difference of the two instances and never worse than $\frac{1}{d+1}$ no matter how different the instances are for a $d$-layer graph.
%(Theorem~\ref{thm:dist_robustness})
%\end{itemize}
\end{enumerate}
\end{theorem}

The theorem states that weights are learnable and only a small number of samples are required to construct weights that are near optimal. 
These predictions break the worst-case $1-\frac{1}{e}$ bound on the competitive ratio for any randomized algorithm for online fractional matching, a special case.
Moreover, the difference in the types of impressions between two instances gives a metric under which we can demonstrate instance robustness.
%and show that the optimal parameters of an instance have similar performance on nearby instances, which is referred to as ``Instance Robustness''. 
Further the algorithm has worst-case guarantees, i.e.
%that the problem generalizes.  
%Moreover, the ratio is never worse than the $1/(d+1)$ bound that we prove that any deterministic integral online algorithm cannot beat (Theorem~\ref{thm:worst_case_bound}) for the case of a $d$-layer graph.
the ratio is never worse than $\frac{1}{d+1}$, which is tight for deterministic integral online algorithms and $d$-layer graphs (see Theorem~\ref{thm:worst_case_bound}) even though we output fractional allocations.

%We now extend these results to the makespan minimization problem.
We now discuss our results for makespan minimization.

\begin{theorem}[Restricted Assignment - Informal]
There exist algorithmic parameters for the Restricted Assignment Makespan Minimization problem with the following properties:
%\begin{itemize}
%\itemsep0em
\begin{enumerate}[(i)]
    \item (Learnability) Learning the near optimal parameters has  sample complexity polynomial in $m$, the number of machines, and $\frac{1}{\epsilon}$. These parameters result in an online algorithm that is a  $(1+\eps)$ approximate solution in expectation as compared to the expected optimal value on the distribution  for any  constant $\eps >0$. (Theorem~\ref{thm:learnability_makespan})
    \item (Instance Robustness) 
    %\tl{There exists a metric on the instances such that} 
    Using the optimal parameters for any instance on a nearby instance gives a competitive ratio for fractional assignment that is proportional to their distance, where the distance is proportional to the relative difference of job sizes of the same type.
    %(w.r.t the metric). 
    %and never worse than $O(\log m)$, matching the known $\Omega(\log m )$ lower-bound on any integral online algorithm.
    %\xcy{By predicting the parameters, the algorithm is instance robust, ensuring the competitive ratio is never larger than $O(\log m)$.
    % factor of optimal. 
    (Theorem~\ref{thm:makespan-robustness})
    %The predictions are instance robust and using them the algorithm 
    \item (Worst-Case Robustness) The competitive ratio of the algorithm using the parameters is never worse than $O(\log m)$, matching the known $\Omega(\log m )$ lower-bound on any integral online algorithm. (Theorem~\ref{thm:makespan-robustness})
\end{enumerate}
%\end{itemize}
\end{theorem}

This theorem shows that the predictions of \cite{DBLP:conf/soda/LattanziLMV20} have much stronger properties than what is known and are learnable and instance robust. That paper left open the question if their predictions can be formally learned in any model.   Moreover, it was not known if they are instance robust.  
We remark that this theorem assumes fractional assignments, whereas the original problem (and the lower bound~\cite{OnlineAssignmentsAzar}) requires integer assignments.  Lattanzi et al.~\cite{DBLP:conf/soda/LattanziLMV20} shows that any fractional assignment can be rounded online while losing a $O((\log \log m)^3)$ factor in the makespan.   

\input{related}

%% file: related.tex
\subsection{Related Work}\label{sec:relatedwork}

%There is growing interest in augmenting algorithms with predictions as well as data-driven approaches to algorithm design.  We highlight some of the work that has been done in these nascent areas of research.

% Main topic 1 Predictions

%An online algorithm with predictions of a problem is $\alpha$-consistent if its competitive ratio is $\alpha$ whenever there is no error in the predictions.
%Usually, the ratio $\alpha$ is better than the worst-case bound on the competitive ratio of that problem in the traditional setting, which means that the predictions can help to go beyond the worst case bound if they are accurate or have small errors (the property we call \emph{useful}).
%An online algorithm with predictions of a problem is $\beta$-robust if it is never worse than $\beta$ competitive, even with highly erroneous predictions.
%Usually, the ratio $\beta$ is not very far away from the competitive ratio of the best online algorithm.
%\noindent \textbf
\subparagraph*{Algorithms with Predictions:}
In this paper, we consider augmenting the standard model of online algorithms with erroneous predictions.
Several online problems have been studied in this context, including caching~\cite{MLCachingLykouris,DBLP:conf/soda/Rohatgi20,DBLP:conf/icalp/JiangP020,WeiCaching20}, page migration~\cite{IMMR20}, metrical task systems~\cite{ACEPS20}, ski rental~\cite{PurohitNIPS,DBLP:conf/icml/GollapudiP19,anand2020customizing}, scheduling~\cite{PurohitNIPS}, load balancing~\cite{DBLP:conf/soda/LattanziLMV20}, online linear optimization~\cite{BhaskarOnlineLearning20}, speed scaling~\cite{bamas2020learning}, set cover~\cite{bamas2020primaldual}, and bipartite matching and secretary problems~\cite{AGKK20}.

 Antoniadis \etal\cite{AGKK20} studies online weighted bipartite matching problems with predictions.  The main aspect of this work which distinguishes it from ours is that it considers the random order arrival model, rather than adversarial orders.  

%{\color{red}Thomas: added this paragraph, not sure where else this citation should go.} 
 Mahdian~et~al.~\cite{DBLP:journals/talg/MahdianNS12} focuses on the design of robust algorithms.  Rather than considering online algorithms which use a prediction, they consider two black box online algorithms, one optimistic and the other pessimistic.  The goal is to give an online algorithm which never performs much worse than the better of these two algorithms for any given instance.  This is shown for problems such as load balancing, facility location, and ad allocation.

%Using predictions to improve other aspects of algorithm design has also been considered.  Kraska \etal\cite{Kraska} uses machine learned predictions to improve the query time for index data structures.  There has also been work on how to reduce the space complexity of certain data structures, such as bloom filters~\cite{Mitzenmacher}, and heavy hitters sketches~\cite{HsuIndykICLR}.  Medina and Vassilvitskii \cite{MedinaV17} utilize bid predictions in the context of auctions.

% proportional allocation
The predictions utilized in our algorithm come in the form of vertex weights that guide a proportional allocation scheme.  Agrawal \etal\cite{PropMatchAgrawal} first studied proportional allocations for maximum cardinality fractional matching as well as weighted fractional matchings with high entropy.
%In their setting there is a bipartite graph $G = (I,A,E)$ and non-negative capacities $C_a$ on the vertices of $A$. They want to compute a maximum cardinality fractional matching respecting the capacities on $A$.
%, i.e. a vector $x \in \R_+^E$ such that $\sum_{a\in A} x_{ia} \leq 1$ and $\sum_{i\in I} x_{ia} \leq C_a$.
%They show that a proportional allocation scheme
%, in which each vertex in $A$ has a weight $\beta_a$ and each vertex in $I$ proportionally assigns itself according the weights of its neighbors in $A$,
%is sufficient to give a near optimal fractional matching. This is the 2-layer DAG version of our vertex weight existence result, which was the inspiration of our generalization to arbitrary DAGs.
 Lattanzi \etal\cite{DBLP:conf/soda/LattanziLMV20} utilize similar predictions based on proportional weights to give algorithms with predictions for online load balancing.  
 
 %In addition to studying this sort of prediction, our work also considers whether the weights for proportional allocation can be effectively learned and shows robustness to empirical instances defined by nearby distributions, which is not  addressed by both these papers.

% Main topic 2 data driven
%\smallskip
%\noindent \textbf
\subparagraph*{Data-Driven Algorithm Design:} This paper considers the learnability of the predictions through the model of data-driven algorithms.  In classical algorithm design, the main desire is finding an algorithm that performs well in the worst case against all inputs for some measure of performance, e.g. running time or space usage.  Data-driven algorithm design~\cite{DBLP:journals/siamcomp/GuptaR17,BalcanDDKSV19,DBLP:conf/focs/BalcanDV18,DBLP:conf/icml/BalcanDSV18,DBLP:conf/nips/BalcanDW18,ChawlaGTTZ19}, in contrast, wants to find an algorithm that performs well on the instances that the user is typically going to see in practice.  This is usually formalized by fixing a class of algorithms and an unknown distribution over instances, capturing the idea that some (possibly worst case) instances are unlikely to be seen in practice.  The typical question asked is: how many sample instances are needed to guarantee you have found the best algorithm for your application domain?

% Highlight main papers in data-driven algos.
%The seminal paper of Gupta and Roughgarden~\citep{DBLP:journals/siamcomp/GuptaR17} initiates the study of data-driven algorithm design.  They formalize the model and make  connections to computational learning theory, notably the pseudo-dimension of a class of functions.
%Additionally, they apply this framework to a class of parameterized greedy algorithms.
%The theory underlying data-driven algorithm design is further developed in~\cite{BalcanDDKSV19} and ~\cite{DBLP:conf/focs/BalcanDV18}. Further work in data-driven algorithm design explores clustering~\cite{DBLP:conf/nips/BalcanDW18}, integer programming~\cite{DBLP:conf/icml/BalcanDSV18}, and search algorithms~\cite{ChawlaGTTZ19}.

%We borrow this idea to define the learnability of predictions. More precisely, in the paper of Gupta and Roughgarden~\citep{DBLP:journals/siamcomp/GuptaR17}, they apply their framework to a class of parameterized greedy algorithms, where each greedy algorithm is characterized by a parameter. Assuming each instance is sampled from an unknown distribution, they want to learn a parameter with the best expected performance. In this paper, we view predictions as those parameters and say the predictions are learnable if we can learn the best predictions for an empirical instance sampling unknown distribution efficiently in terms of the size of the empirical instance.

% Other online matching and flow problems
%\smallskip
%\noindent \textbf
\subparagraph*{Other Related Work:} Online matching and related allocation problems have been extensively studied in both the adversarial arrival setting~\cite{DBLP:conf/stoc/KarpVV90,DBLP:journals/tcs/KalyanasundaramP00,DBLP:conf/soda/DevanurJK13,DBLP:journals/jacm/MehtaSVV07} and with stochastic arrivals~\cite{DBLP:conf/sigecom/DevenurH09,DBLP:conf/esa/FeldmanHKMS10,DBLP:journals/mor/GuptaM16,DBLP:conf/icalp/MolinaroR12,DBLP:journals/ior/AgrawalWY14}.
% advice model
A related but different setting to ours is the online algorithms with advice setting~\cite{Boyar}.  Here the algorithm has access to an oracle which knows the offline input.  The oracle is allowed to communicate information to the algorithm about the full input, and the goal is to understand how many bits of information are necessary to achieve a certain competitive ratio.  This has also been extended to the case where the advice can be arbitrarily wrong~\cite{DBLP:conf/innovations/AngelopoulosDJKR20}.  This can be seen as similar to our model, however the emphasis isn't on tying the competitive ratio to the amount of error in the advice.

%% file: model.tex
\section{Algorithms with Learnable and Instance-Robust Predictions}
\label{sec:model}

%This section formalizes the model. Intuitively, the model enforces that predictions can be learned for a distribution.  This ensures near-optimal performance for future instances that come from the distribution.  Next, we desire instance robustness.  This enforces that instances are robust to adversarial future instances similar to those the predictions perform well on. 
%\tl{Maybe we should use notation for the performance of the algorithm?  E.g. $\cA(\cI,\hat{y}(\cI))$. For now we tried to avoid that below.}

%\smallskip 
%\noindent \textbf
\noindent \textbf{Learnability via Sample Complexity:}  We consider the following setup inspired by PAC learning and recently considered in data-driven algorithms.  Assume a maximization problem and let $\cD$ be an unknown distribution over problem instances.  Let $\alg(\cI,y)$ be the performance\footnote{In general this can be any performance metric, such as running time or solution value.  Here we focus on the value of some objective function such as the size of a fractional flow.} of an algorithm using parameters $y$ on instance $\cI$.  The ideal prediction for this distribution is then $y^* := \arg\max_{y} \E_{ \cI \sim \cD} [\alg( \cI,y)]$.  Since we assume that $\cD$ is unknown, we wish to learn from samples.  In particular, we wish to use some number $s$ of independent samples to compute a parameter $\hat{y}$ such that $\E_{ \cI \sim \cD}[\alg( \cI,\hat{y})] \geq (1-\epsilon)\E_{ \cI \sim \cD}[\alg( \cI,y^*)]$ with probability $1-\delta$, for any $\eps,\delta \in (0,1)$.  The sample complexity $s$ depends on the problem size as well as $1/\epsilon$ and $1/\delta$.  As is standard in learning theory, we require the sample complexity to be polynomial in these parameters\footnote{For more difficult problems, we can relax the $1-\epsilon$ requirement to be a weaker factor.}.

Inspired by competitive analysis, we also compare to the following stronger benchmark in this paper.  For any instance $ \cI$, let $\opt( \cI)$ be the value of an optimal solution on $ \cI$.  We give learning algorithms producing predicted parameters $\hat{y}$ such that $\E_{ \cI \sim \cD}[\alg( \cI,\hat{y})] \geq (1-\epsilon) \E_{ \cI \sim \cD}[\opt( \cI)]$
%\footnote{\tl{This guarantee is similar to PAC-learning~\cite{DBLP:journals/cacm/Valiant84}, but adapted to our setting of online optimization}} 
and polynomial sample complexity under the assumptions on $\cD$ described earlier.  Note that this guarantee implies the first one.

\medskip
\noindent \textbf{Instance Robustness:}
Let $\cI$ and $\cI'$ be two problem instances, and consider running the algorithm on instance $\cI'$ with the prediction $y^*(\cI)$.  We bound the performance of the algorithm as a function of the difference between these two instances. 
%Measuring the difference in two instances is inherently problem specific. 
In contrast, prior work focuses on differences in the predicted parameters  $y^*$ and  $y'$ for the same instance $\cI$.  
Moreover, it is desirable that the algorithm never performs worse than the best worst-case algorithm. 

%\tl{Thomas: Make slightly more formal.} 
For example, in online flow allocation, we can consider an instance as a vector of types,  i.e. $\cI_i$ is the number of impressions of type $i$.  Then we can take the difference between the instances as $\gamma =\|\cI-\cI'\|_1$. 
%Let $\cD'$ and $\cD$ be two distributions over problem instances. Say that  $\hat{y}(\cD)$ results in an algorithm $A(\hat{y}(\cD),I)$ having good performance if $I$ is drawn from  $\cD$.     Then say that $\hat{y}(\cD))$ is used from a problem instance $I'$ drawn from a similar distribution $\cD'$.  Intuitively, $\cD'$ is an  incorrect distribution similar to the real distribution $\cD$. The expected performance of the algorithm should be bounded as a function of the difference between the two distributions.   The distributions considered are not defined a priori, but are problem dependent.   Moreover, it is desirable if the algorithm never performs worst than the best worst case algorithm.
Say $y^*(\cI)$ can be used to give a  $c$-competitive algorithm on instance $\cI$.  
Let $\alpha$ be the best competitive ratio achievable in the worst-case model.
%A reasonable goal could be 
We desire
an algorithm that is $\max\{f(c,\gamma), \alpha\}$-competitive where 
%$\cI$ and $\cI'$ differ by $\gamma$ in the $\ell_1$ norm  \ravi{when their proportion of types are} viewed as distributions and 
$f$ is a monotonic function depending on $c$ and $\gamma$.We remark that the online model requires  $\cI'$ to arrive in a worst-case order.

\subsection{Putting the Model in Context}

%We now put the model in context.

%\medskip 
%\noindent \textbf
\subparagraph*{Relationship to Prior Predictions Model:}

The first main difference in this model as compared to prior work is  \emph{learnability}.  With the notable exception of \cite{anand2020customizing},   prior work  has introduced predictions without establishing they are learnable.  Without this requirement there is no objective metric to detail if a prediction is reasonable or not.  To see this shortcoming, imagine simply predicting the optimal solution for the problem instance.  This is  often not reasonable because the optimal solution is too complex to learn and use as a prediction.  We introduce  bounded sample complexity for learning predictions in our model to ensure predictions can be provably learned.

Next difference is in how to measure error.  The performance of the algorithm is bounded in terms of the error in the prediction in the prior model.  For example, say the algorithm is given a predicted vector $\hat{y}(\cI)$ for problem instance $\cI$ and the true vector that should have been predicted is $y^*(\cI)$. One can define $\eta_{\hat{y}( \cI)} = \|\hat{y}( \cI) - y^*( \cI) \|_p$  to be the \emph{error} in the parameters for some norm $p\geq 1$.  The goal is to give an algorithm that is $f(\eta_{\hat{y}( \cI)})$-competitive for an online algorithm where $f$ is some non-decreasing function of $\eta_{\hat{y}( \cI)}$: the better the function $f$, the better the algorithm performance.
One could also consider run time or approximation ratio similarly. Notice the bound is worst-case for a given error in the prediction. This we call \emph{parameter robustness}.

It is perhaps more natural to define a difference between two problem instances as in our model rather than the difference between two predicted parameters. Indeed, consider predicting optimal dual linear program values. These values can be different for problem instances that are nearly identical. Therefore, accurate parameters will not be sufficient to handle inconsequential changes in the input.  
 %This \ravi{lack of} scale invariance makes it difficult to compare the performance of two different predictions. 
  %Instance robustness on the other hand allows for better comparison of two predictions. 
Instance robustness allows for more accurate comparison of two predictions on similar problem instances. More practically, instance closeness is easier to monitor than closeness of the proposed predictions to an unknown optimal prediction for the whole instance.

\subparagraph*{Learning Algorithm Parameters:}  
%A requirement of our model is showing how to learn parameters for an algorithm.  Given this, one may ask why not simply learn the input distribution and then use algorithms designed for stochastic inputs.
Learning algorithmic parameters has distinct advantages over learning an input distribution. In many cases it can be easier to learn a decision rule than it is to learn a distribution. For example, consider the unweighted $b$-matching problem in bipartite graphs for large $b$ in the online setting. In this problem there is a bipartite graph $G = (I \cup A,E)$ with capacities $b \in \Z_+^A$.  The objective is to find a collection of edges such that each node in $I$ is matched at most once and each node $a\in A$ is matched at most $b_a$ times.  Nodes on one side of the graph arrive online and must be matched on arrival.  Say the nodes are i.i.d. sampled from an unknown discrete distribution over types. A type is defined by the neighbors of the node.  Let $s$ be the number of types. Then the sample complexity of learning the distribution is $\Omega(s)$. Notice that $s$ could be superlinear in the number of nodes.  
In our results, the sample complexity is independent of the number of types in the support of the distribution.
%However, it is known \cite{XXX} that for this problem, an algorithm can be learned that requires a sublinear number of samples.
The phenomenon that it is sometimes easier to learn good algorithmic parameters rather than the underlying input distribution 
%the algorithm is facing 
has been observed in several prior works.  See \cite{DBLP:journals/siamcomp/GuptaR17,BalcanDDKSV19,DBLP:conf/focs/BalcanDV18,DBLP:conf/icml/BalcanDSV18,DBLP:conf/nips/BalcanDW18,ChawlaGTTZ19} for examples. 
%\textcolor{red}{FILL IN CITES}

%{\color{red} \textbf{Chenyang: Maybe we can explain distributional Robustness in this way:} \color{black} although the distribution is hard to be predicted directly (we need to find some other algorithmic parameters to predict), the distribution provides us a platform to connect different parameters predictions. We can assume that we have a way to predict the distribution of an online instance when comparing different parameter predictions. We compute different algorithmic parameters from the predicted distribution, apply them to a new online instance and bound all their performances by the error of the two instance distributions. Note that some algorithmic parameters may fluctuate sharply even if the distribution of instance changes a little bit. If only checking the parameter robustness, these parameters seem to be good predictions. In this way, we can compare the competitive ratio of different predictions.  But in practice, the predicted error of these parameters will always be very large since they are very sensitive. }
Table~\ref{tab:previous_work} illustrates how our paper relates to prior work, focusing on the two pillars for augmenting algorithms 
%augmented with predictions we
emphasized in our model.

\begin{table}[t]
\captionsetup{justification=centering}
\centering
\begin{tabular}{p{120pt}p{90pt}p{60pt}p{80pt}}
\hline 
 Problem & Parameter  Robustness & Learnability & Instance  Robustness \\
\hline

 Caching & \cite{MLCachingLykouris,DBLP:conf/soda/Rohatgi20,DBLP:conf/icalp/JiangP020,WeiCaching20} & - & -\\
 Completion Time Scheduling  & \cite{PurohitNIPS} & - & \cite{PurohitNIPS}\\
 Ski Rental & \cite{PurohitNIPS,anand2020customizing,DBLP:conf/icml/GollapudiP19} & \cite{anand2020customizing}& \cite{PurohitNIPS,anand2020customizing,DBLP:conf/icml/GollapudiP19} \\
 \hline
 Restricted  Assignment  & \cite{DBLP:conf/soda/LattanziLMV20} &   This Paper  & This Paper \\
   $b$-Matching &  This Paper  & This Paper & This Paper \\
 Flow~Allocation &  This Paper  &  This Paper &  This Paper\\
 \hline
 \end{tabular}
\caption{Relationship to Prior Work} 
\label{tab:previous_work} \vspace{-.6cm}
\end{table}

%\subparagraph*{Relationship to Prior Work:}  Table~\ref{tab:previous_work} illustrates how our paper relates to prior work, focusing on the two pillars for augmenting algorithms 
%augmented with predictions we
%emphasized in our model.
%The first column states if they prove the existence of useful predictions.  The second denotes if the error is bounded in terms of error in the predicted parameter.  The third and fourth denote whether they show instance or parameter robustness. Finally, the last is if the parameter predicted is provably learnable.
%Note that prior work essentially puts the ski rental problem within our model.  
%For the ski rental problem the predicted parameter is essentially a prediction of the entire instance, and thus parameter robustness is the same as instance robustness.
%Extending this line of literature, we show there are predictions that have both properties we introduce as well as the well studied parameter-robustness simultaneously for problems with complex combinatorial structure.

\subparagraph*{Paper Organization:}
%The rest of the paper is organized as follows.
%The main results and a high-level overview of the analysis is given first and the technical analysis appears next. 
For online flow allocation, 
both learnability and instance-robustness rely on showing the existence of node predictions which is described in
Section~\ref{sec:flows_existence}, followed by learnability and robustness in Sections~\ref{sec:flows_learnability} and~\ref{sec:flows_robustness} respectively.
While these sections give technical overviews, full proofs are in the appendix. We show the existence of predictions in 3-layer DAGs in Appendix~\ref{sec:existweight} and give the results for general DAGs in Appendix~\ref{sec:apdxexistweight}.
We first prove the learnability of predictions in 2-layered graphs in Appendix~\ref{sec:learnability} and then generalize this result to general DAGs in Appendix~\ref{sec:apdxlearnability}.
The proofs about the instance- and parameter-robustness are in Appendices~\ref{sec:robustness} and~\ref{app:robustness} respectively. 
For load balancing, the results are in Section~\ref{sec:load_balancing_results} and the corresponding proofs are in Appendix~\ref{sec:apdx_scheduling}.

%The direct approach is to learn the unknown distribution from samples and utilize known ideas from stochastic optimization (where knowledge of the distribution is key). However, if the number of item types (denoted by s) is very large, the sample complexity for learning the distribution could be huge. However, if we only aim to learn the threshold ϴ, the sample complexity becomes much smaller. More precisely, the pseudo-dimension of this threshold is at most log(s). Also, from previous work (Rishi Gupta & Tim Roughgarden, 2017), the sample complexity is O(log(s)) which is significantly smaller. 

%% file: flows_existence.tex
\section{Matchings and Flows: Existence of Weights} \label{sec:flows_existence}

%We  establish that there is a single weight that can be predicted for each node in the DAG to route the flows near optimally. These weights are in the same spirit as weights used in other works on Adwords~\citep{PropMatchAgrawal} and scheduling~\citep{DBLP:conf/soda/LattanziLMV20}. The weights can be used to allocate flow for each impression such that the flow for each impression is deterministic and computed independently of other impressions.
%This flow is routed by splitting the ongoing flow to the sink in a proportional manner based on the weights of a node's descendants and their distances from a dummy source connected to all the source nodes (generalizing their layer number).
%Note in particular that all the flows passing through any node are split in the same proportion irrespective of the originating source.
%Moreover, the flow is near optimal. Because the flow can be computed in one forward pass for each impression in isolation, they naturally can be used online. These weights are also efficiently computable offline given the entire problem instance.

Consider a directed acyclic graph $G=(\{s,t\} \cup V, E)$, where each vertex $v\in V$ has capacity $C_v$ and is on some $s-t$ path. Our goal is to maximize the flow sent from $s$ to $t$ without violating vertex capacities. %(note that there are no edge capacities).
% new text for 3-layer version
Before considering the general version, %of this problem, 
we examine the 3-layered version. Say a graph is $d$-layered if the vertices excluding $s,t$ can be partitioned into $d$ ordered sets where all arcs go from one set to the next.
Then the $3$-layered case is defined as follows. The vertices in $V$ are partitioned into 3 sets $I,A$, and $B$. $s$ is connected to all of $I$ and $t$ is connected from all of $B$, while the remaining edges are only allowed to cross from $I$ to $A$ and from $A$ to $B$.  Let $N_u := \{v \in V \mid (u,v) \in E\}$ be $u$'s out-neighbors.  We have the following result generalizing the prior work of Agrawal~et~al~\cite{PropMatchAgrawal} on 2-layered graphs.

% complicated general case, replace w/ 3 layer version in intro

\begin{theorem} \label{thm:weight_existence_3_layer}
	Let $G = (\{s,t\} \cup V, E)$ be a 3-layered DAG.  For each edge $(u,v) \in E$, let $x_{uv}$ be the proportion of flow through $u$ which is routed along $(u,v)$.  For any $\epsilon \in (0,1)$, there exist weights $\{\alpha_v\}_{v\in V}$ such that setting $x_{uv} = \frac{\alpha_v}{\sum_{v' \in N_u} \alpha_{v'}}$ yields a $(1-\epsilon)$-approximate flow. Moreover, these weights can be obtained in time $O(n^4 \log(n/\epsilon) / \epsilon^2)$.
\end{theorem}

We can generalize this theorem to $d$-layered graphs.  In particular, our algorithm for the $d$-layered case produces weights with additional properties, which we leverage to handle general DAGs.  See Section~\ref{sec:existweight} for precise statements.
% new version of previous paragraph here
Notice that the number of weights is proportional to the number of nodes in the graph and not the number of edges. We believe it is an interesting combinatorial property that such succinct weights on the nodes can encode a good flow on the asymptotically larger number of edges and is of independent interest.

\subsection*{Technical Overview}

Here we give a technical overview.  The full proof is in  Section~\ref{sec:existweight}. 

%\smallskip
%\noindent \textbf
\medskip 
\noindent \textbf{A Simple Algorithm for Layered Graphs:} Prior work \cite{PropMatchAgrawal}  showed that there exists a set of weights giving nearly the same guarantees we show, but only for bipartite graphs.  The existence of such weights can be generalized to  $d$-layer graphs easily as follows.  First find an (optimal) maximum flow $f$.  For each vertex $v$, let $f(v)$ be the flow going through $v$.  Reset the vertex capacity of $v$ to be $f(v)$.  For each pair of adjacent layers find the weights between the two layers independently using the algorithm of \cite{PropMatchAgrawal}, treating nodes $v$ on the left hand side as $f(v)$ individual impressions.
By the previous result, each layer only loses a negligible portion of the total flow which can be compounded to yield a low loss for these set of weights.

The above reduction does not generalize to general DAGs.  One can arrange a DAG into layers, but there are fundamental algorithmic challenges with constructing weights that arise when edges cross layers.
%In particular, the above procedure has no mechanism for handling cross layer edges
One of this paper's algorithmic contributions is showing how to construct such weights for general DAGs.  Moreover, as an intermediate result, we show how to compute the weights directly extending the approach of  \cite{PropMatchAgrawal} for multi-layer graphs without first solving a flow problem optimally as in the above reduction.

%\smallskip
%\noindent \textbf
\medskip 
\noindent \textbf{Finding Weights for Bipartite Graphs:}  We begin by first simplifying the algorithm of \cite{PropMatchAgrawal} for bipartite graphs.  Let $G=(\{s,t\}\cup I\cup A,E )$ be such a graph.  In this case, the fraction of flow $u \in I $ sends  to  $v \in A$ simplifies to $ x_{uv} =  \frac{\alpha_v}{ \sum_{v'\in N_u} \alpha_{v'}  }$ where $\{\alpha\}_{v \in I \cup A}$ are the set of weights.  Initially all of the weights are $1$ for a vertex $a \in A$.  Some of the nodes receive more flow than their capacity in this initial proportional allocation according to the weights.   We say a node for which the current proportional allocation of flow exceeds its capacity by a $1+\eps$ factor is \emph{overallocated}.  In an iteration, the algorithm decreases the weights of these nodes by a $1+\eps$ factor.\footnote{Prior work~\cite{PropMatchAgrawal} performed this operation as well as {\em increasing} the weights of nodes whose allocation was significantly below the capacity.  Our simplification to allow only decreases helps with the generalization to more complex graphs and correcting for error in the weights.}  After this process continues for a  poly-logarithmic number of iterations, we will be able to show the resulting weights result in a near optimal flow.

To prove that the final weights are near optimal, we show that the weights can be directly used to identify a
%near optimal minimum
% TL: "near optimal minimum" seems redundant and confusing
vertex cut whose value matches the proportional flow given by the weights. %(truncated by the capacities).  
In particular, we will partition the nodes in $A$ based on their weight values.  For a parameter $\beta$, we say a node is `above the gap' if its weight is larger than $\beta n /\epsilon$.  A node of weight less than $\beta$ is below the gap.  All others are in the gap.  The parameter $\beta$ is chosen such that the nodes in the gap contribute little to the overall flow and they can essentially be discarded via an averaging argument\footnote{This averaging is what necessitates the poly-logarithmic number of weight update iterations in the algorithm}. Assume this set is empty for simplicity.  Let $\g(A)^+$ and $\g(A)^-$ be the sets of vertices in $A$ above and below the gap, respectively.

We now describe a cut.  Let $I_0\subseteq I$ be the impression nodes adjacent to at least one node in $\g(A)^+$.
%above the gap. 
Then the vertex cut is $I_0  \cup \g(A)^-$. Since all paths must either cross $I_0$ or $\g(A)^-$, this is a valid vertex cut. We now show the cut value is close to the flow achieved by the weights, completing the analysis using the weaker direction of the max-flow min-cut theorem.

First, nodes in $I_0$ are cut.  Due to the way flow is sent based on the weight proportions, for any vertex $I_0$, at least a $(1-\epsilon)$ proportion of its flow will be sent to $\g(A)^+$.  Since the nodes in $\g(A)^+$ did not decrease the weights at least once, at some point they were not over-allocated.  We claim that because of this, they will never be over-allocated hence and, therefore, nodes in $I_0$ send nearly all of their flow to the sink successfully. Next nodes in $\g(A)^-$ are cut. These nodes decreased their weights (almost) every iteration 
%and we can show that this must be 
because they are at or above their allocation.  Thus, for all these nodes we get flow equal to their total capacity.
The fraction of this flow through $\g(A)^-$ coming from paths using $I_0$ is negligible because of the weight proportions so this flow is almost disjoint from that of $I_0$.
Thus, we have found a proportional flow nearly matching the value of the cut identified.

%\smallskip
%\noindent \textbf
\medskip 
\noindent \textbf{General Graphs:}   Now we consider the more general algorithm. To convey intuition, we will only consider directly computing weights for a 3-layered graph $G=(\{s\} \cup I\cup A\cup B\cup \{t\} ,E )$ where edges are between adjacent layers. This  will highlight several of the new ideas.  As before, weights of all nodes are initially one. And as before, a node decreases its weight if it is over-allocated, which we will refer now to as a \emph{self-decrease}.    Now though, whenever a node in $A$ decreases its weight it does so by a $(1+\eps')$ factor and those in $B$ decrease at a $(1+\eps)$ factor where $\eps' \leq \eps$.

A new challenge is that a node $b$ in layer $B$ may be over allocated and it may not be enough for $B$ to reduce its weight. Indeed, $B$ may need some neighbors in $A$ to reduce their allocation.  For instance, if $b \in B$ has neighbors in $A$ for which it is the only neighbor, then reducing $b$'s weight does not change its allocation and the flow needs to be redistributed in the first layer.    In this case, the nodes in $B$ will specify that some nodes in $A$ need to decrease their allocation.   We call this a \emph{forced decrease}.  This set has to be carefully chosen and intuitively only the nodes in $A$ that are the largest weight as compared to $b\in B$ are decreased.   We run this procedure for a polylogarithmic number of steps and again we seek to find a cut matching the achieved flow.

We discuss the need for different  $\eps$ and $\eps'$.  In the bipartite case when a node decreased its weight, that node is guaranteed to receive no more allocation in the next round (it could remain the same though). Intuitively, this is important because in the above proof for bipartite graphs we want that if a node is in $\g(A)^+$, above the gap, if it was ever under-allocated then it never becomes over-allocated in later iterations.  Our update ensures this will be the case since self-decreases will continue henceforth to keep the load of such a node below its capacity.  Consider setting $\eps= \eps'$ for illustration. Because of the interaction between layers, a node $i$ in $B$ could receive more allocation even if it decreases its weight in an iteration. This is because the nodes in $A$ and $B$ could change their weights.  Nodes in $A$ changing their weight can give up to an extra $(1+\eps)$ allocation (via predecessors of $i$ that are not decreased), and the same for $B$ for a total of $(1+\eps)^2$ extra allocation arriving at $i$.   The node decreasing its weight reduces its allocation by a $(1+\eps)$ factor for a total  change of a $(1+\eps)^2 \cdot \frac{1}{(1+\eps)} >1$ factor.  By choosing $\eps$ and $\eps'$ to be different, as well as the characterization of which nodes decrease during a forced decrease, we can show any node will not receive less allocation if its weight does not decrease and will not receive more allocation if it performs a decrease. We call these properties ``Increasing monotonicity'' (Property~\ref{ppt:1}) and ``Decreasing monotonicity"  (Property~\ref{ppt:2}) in Section~\ref{sec:existweight}.

As in the bipartite case, we can find a gap in layers $A$ and $B$, which gives sets above the gap $\g(A)^+$ and $\g(B)^+$ in $A$ and $B$, respectively. Let the sets $\g(A)^-$ and $\g(B)^-$ be the nodes below the gap.  As before nodes in $\g(A)^-$ (resp., $\g(B)^-$) decreased their weight many more iterations than  $\g(A)^+$ (resp. , $\g(B)^+$,).  For simplicity, assume this partitions the nodes of the entire graph, so no nodes are inside either gap.  Let $I_0\subseteq I$ be the nodes adjacent to
%in $I$ that have an edge to a 
at least one node in $\g(A)^+$. Let $A_0$ be nodes in $\g(A)^-$, below the gap, that have an edge to $\g(B)^+$ (the analogue of $I_0$ in the $A$-layer) .
Any flow path that crosses the $A$ layer at $\g(A)^+, A_0$ and $\g(A)^- \setminus A_0$ are blocked by the sets $I_0,A_0$ and $\g(B)^-$ respectively showing that this set forms a vertex cut.
%This separates the source and the sink by their definition.

We show the flow obtained is nearly the value of the vertex cut $I_0 \cup A_0 \cup  \g(B)^-$.
As before, nodes in $I_0$ (resp. $A_0$) send almost all their flow to nodes above the gap in the next layer $\g(A)^+$ (resp. $\g(B)^+$).  Like before $\g(A)^+$ and $\g(B)^+$ do not decrease every round, so we can show they are not allocated more than their capacity (see Lemma~\ref{lem:b_lowerbound} and Lemma~\ref{lem:a_lowerbound}).  The algorithmic key is that, by choosing the forced decrease carefully, we can show each node in $\g(A)^+$ has a neighbor in $\g(B)^+$.  This ensures almost all of $\g(A)^+$'s flow reaches the sink because all of these nodes will send essentially all their the flow to $\g(B)^+$ and these nodes are not at capacity. Thus, $I_0 $ can send all its flow to the sink successfully.
Similarly, $A_0$ sends its flow to $\g(B)^+$ and then to the sink.
Finally, as before, $\g(B)^-$ (as well as  $\g(A)^-$) are sets of nodes near their allocation because they decreased essentially every iteration (see Lemma~\ref{lem:b_upperbound} and Lemma~\ref{lem:a_upperbound}).
Thus, $\g(B)^-$ sends its flow directly to the sink.
%, and a negligible fraction of this flow is double-counted.
Moreover, we ensure that only a negligible fraction of this flow is double counted by the definition of the large weight reduction across the gaps (see Lemma~\ref{lem:gap}); therefore we have  found a flow allocation obtained by the weights whose value nearly matches the value of an identified cut.

This analysis generalizes to layered DAGs where edges do not cross layers.  To extend the existence of these weights to general DAGs we reduce the problem to finding weights with additional structural  properties on a layered DAG. From the input DAG we make additional copies of nodes that have ancestors in earlier layers and link these copies via a path to the original neighbor. We convert any edge that crosses many layers to one in the layered DAG from its original head node to the copy of its tail node in the next layer.  Then we argue that a key functional relation exists between the weights of any original node and its copies in earlier layers. 
This allows us to transfer the weights computed in the auxiliary layered DAG to the original DAG (see Section~\ref{subsec:existgenweight} for more details).
%This allows us to define weights in the original graph based on the weights computed in the auxiliary layered DAG (see Section~\ref{subsec:existgenweight} for more details). 

%% file: flows_learnability.tex
\section{Matchings and Flows: Learnability of Predictions} \label{sec:flows_learnability}

We show that the weights are efficiently learnable.  Assuming that each arriving impression is i.i.d. sampled from an unknown distribution, we want to learn a set of weights from a collection of past instances and examine their expected performance on a new instance from the same distribution\footnote{We can also analyze the performance on similar distributions by applying techniques from our instance robustness result}.
A direct approach might be to learn the unknown distribution from samples and utilize known ideas from stochastic optimization (where knowledge of the distribution is key).
A major issue though is that there can be a large number of possible types (potentially exponential in the number of nodes of the DAG).  A distribution that is sparse over types is not easy to learn with a small number of samples.
	
	%Notice that in the above we showed that if the weights are computed on a distribution they are robust to modest modifications of the distribution.  This shows that we can learn a distribution and then compute the weights.  The performances is as good as  how well the distribution fits the impression that arrive online. A major issue though is that there can be a large number of possible types (potentially exponential in the number of nodes of the DAG).  A distribution that  is sparse over types is  not easy to learn with a small number of samples.

We claim that the weights are efficiently learnable, even if the distribution of types of impressions is not. We show that this task has low sample complexity and admits an efficient learning algorithm.  Consequently, if there is an unknown arbitrary distribution that the impressions are drawn from, then only a small number of instances is required to compute the weights. The number of samples is proportional to size of the DAG without the impressions.  In most problems such as Adwords, the number of arriving impressions is much larger than the fixed (offline) portion of the graph.

Before stating our results, we introduce two necessary assumptions. The first assumption is that each impression is i.i.d. sampled from an unknown distribution $\cD$. Where no ambiguity will result, we also say an instance is sampled from $\cD$ if each impression is an i.i.d sample from $\cD$. The second assumption is related to the expected instance of the distribution $\cD$. The \textbf{expected instance} of a distribution is the instance where the number of each type of impressions is exactly the expected value.\footnote{Note this could be a fractional value.} We assume that in the optimal solution of the expected instance, the load of each node is larger than a constant. Namely, it cannot happen that in the optimal flow, there exist many vertices which obtain very small amount of flow.

\begin{theorem}
\label{thm:DAG-learnability}
	Under the two assumptions above,
	for any $\epsilon,\delta \in (0,1)$,
	there exists a learning algorithm such that, after observing $O(\frac{n^2}{\epsilon^2}\ln(\frac{n\log n}{\delta}))$ instances,
	 returns  weights $\{ \predw \}$, satisfying that with probability at least $1-\delta$,
	$ \E_{I \thicksim \D}[R(\predw,I)] \geq (1-\epsilon) \E_{I \thicksim \D}[R(\alpha^*,I)]$
	where $R(\alpha,I)$ is the value of the fractional flow obtained by applying $\alpha$ to instance $I$ and $\alpha^* = \arg\max\limits_{\alpha } \E_{I \thicksim \D}[ R(\alpha,I) ]$.
\end{theorem}

\noindent\textbf{Technical Overview:}
Here we overview the analysis. The full proof is in  Section~\ref{sec:learnability}.  To show that the weights are learnable we utilize a model similar to that of data-driven algorithm design.  To illustrate our techniques we focus on the 
%bipartite 
case 
%here, i.e. 
when the instance is a bipartite graph $G = (I\cup A,E)$ with capacities $C_a$ for each $a \in A$ (also recall that $|A| = n$).

In this setting there is an unknown distribution $\cD$ over instances of $I$ of length $m$.  The $t$'th entry of $I$ represents the $t$'th impression arriving online.  Our goal is to find a set of weights that performs well for the distribution $\cD$.  In particular let $\alpha^* := \arg\max_{\alpha \in \cS} \E_{I \sim \cD}[R(\alpha,I)]$ be the \emph{best} set of weights for the distribution.  Define $R(\alpha,I)$ to be the value of the matching using weights $\alpha$ on instance $I$.  Here $\cS$ is a set of ``admissible'' weights, and in particular we only consider weights output by a proportional algorithm similar to the algorithm of Agrawal \etal\cite{PropMatchAgrawal}.  We are allowed to sample $s$ independent samples $I_1,I_2,\ldots,I_s$ from $\cD$ and use these samples to compute a set of weights $\hat{\alpha}$.  We say that a learning algorithm $(\epsilon,\delta)$-learns the weights if with probability at least $1-\delta$ over the samples from $\cD$, we compute a set of weights $\hat{\alpha}$ satisfying $\E_{I \sim \cD}[R(\hat{\alpha},I)] \geq (1-\epsilon)\E_{I \sim \cD}[R(\alpha^*,I)].$

This definition is similar to PAC learning~\citep{DBLP:journals/cacm/Valiant84}.  Also note we are aiming for a relative error guarantee rather than an absolute error.
The main quantity of interest is then the \emph{sample complexity}, i.e. how large does $s$ need to be as a function of $n$, $m$, $\epsilon$, and $\delta$ in order to $(\epsilon,\delta)$-learn a set of weights?  Ideally, $s$ only depends polynomially on $n$, $m$, $1/\epsilon$, and $1/\delta$, and smaller is always better.

The standard way to understand the sample complexity for this type of problem is via the \emph{pseudo-dimension}.  Intuitively, pseudo-dimension is the natural extension of VC-dimension to a class of real valued functions.  In our case the class of functions is $\{R(\alpha,\cdot) \mid \alpha \in \cS\}$, i.e. we are interested in the class of function mapping each instance $I$ to the value of the fractional matching given by each fixed set of weights $\alpha$.  If the pseudo-dimension of this class of functions is $d$, then $s \approx \frac{d}{\epsilon^2} \log(1/\delta)$ samples are needed to $(\epsilon,\delta)$ learn the weights, given that we are able to approximately optimize the empirical \emph{average} performance~\cite{anthonyNNL}.

The good news for our setting is that the pseudo-dimension of our class of functions is bounded.  Each node in the set $A$ can only have one of $T$ different weight values for some parameter $T$.  Then since the number of nodes in $A$ is $n$, there can only be at most $T^n$ different ``admissible'' weights.  It is well known that the pseudo-dimension of a finite class of $k$ different functions is $\log_2(k)$.  Thus the pseudo-dimension of our class of functions is $d = n\log_2(T)$.  As long as $T$ isn't growing too fast as a function of $n$ and $m$, we get polynomial sample complexity.

Unfortunately, finding weights to optimize the average performance across the $s$ sampled instances is  complicated.  Note that for a fixed instance $I$, the value of the matching as a function of the weights, $R(\cdot,I)$, is non-linear in the weights since we are using proportional allocation.  Moreover, it is neither convex nor concave in the parameters $\alpha$ so applying a gradient descent approach will not work. Due to this, it is difficult to analyze the learnability via known results on pseudo-dimension.

The main tool we have at our disposal is that for a fixed instance $I$, we can compute weights $\alpha$ such that $R(\alpha,I) \geq (1-\epsilon)\opt(I)$.  This motivates the following natural direct approach.  Take the $s$ sampled instances $I_1,I_2,\ldots,I_s$ and take their union to form a larger ``stacked'' instance $\hat{I}$.  We then run the aforementioned algorithm on $\hat{I}$ to get weights $\hat{\alpha}$.  Intuitively, if $s$ is large enough, then by standard concentration inequalities $\hat{I} \approx s\E[I]$, i.e. the stacked instance approaches $s$ copies of the  ``expected'' instance.  Then to complete the analysis, we need to show that $\E[R(\predw,I)] \approx R(\predw,\E[I])$.  
%Due to the non-linearity of this function, 
In general, it is not true that $\E[R(\predw,I)] \approx R(\predw,\E[I])$.  Using more careful analysis, we show that when the distribution $\cD$ is a product distribution and our two assumptions hold, this is in fact the case.

%To extend these results to the more general setting of $d$-layer graphs and DAG's, we first extend our analysis to laminar instances and then use a series of reductions to apply the result for laminar instances. 

%% file: flows_robustness.tex
\section{Matching and Flows: Robustness} \label{sec:flows_robustness}

%Finally, we establish that the predictions are robust to error. Given a set of weights, we bound the performance of the online algorithm in terms of the prediction error.  In particular, we show the performance degrades gracefully in terms of the error. As mentioned above, we consider distributional robustness and parameter robustness, giving us two ways of defining error.

%In this section, we establish that the predictions are both instance robust and parameter robust. 
%More precisely, the performance of our algorithm not only degrades gracefully in terms of the predicted parameter error (like previous works), but also moves smoothly in terms of the difference of two instances. 
%two ways to measure the prediction error.
%We show this in two ways.
\subparagraph*{Instance Robustness: } To show the instance robustness, we assume that we can describe the instance directly. Say we have a description of the entire instance denoted by a vector $\predI =(\predm_1,\ldots,\predm_i,\ldots)$, where $\predm_i$ is the number of impressions of type $i$. 
%we consider a distribution over types of impressions. The type of an impression is defined by the subset of $V$ to which it has outgoing arcs. 
We show that if a set of weights performs well in instance $\predI$, it can be transferred to a nearby instance $I$ robustly.

%for a set of weights computed using a type distribution $\cD$, they can be easily transferred to another distribution close to $\cD$. 

%This is a natural quantity to predict from past data, but it is not clear how to use this information to compute allocations online.   

%We show that if the weights are computed using predictions of this type distribution, then	the amount of flow that is lost degrades gracefully according to the error in the distribution predicted.

%The first type of prediction we consider is a distribution over types of impressions.  The type of an impression is defined by the subset of $V$ to which it has outgoing arcs.  This is a natural quantity to predict from past data, but it is not clear how to use this information to compute allocations online.   We show that if the weights are computed using predictions of this type distribution, then	the amount of flow that is lost degrades gracefully according to the error in the distribution predicted.
	
%Next we show that if the weights are computed according to a distribution over types of impressions, then the error gracefully degrades when the parameters of this distribution are modified. Nodes of one \emph{type} have outgoing edges to the same set of nodes. The theorem below shows that the predictions are a $1-\eps$ approximation when $\opt$ can route almost all of the impressions.

	\begin{theorem} \label{thm:dist_robustness}

		%Let $\{\predprob_S\}_{S \subseteq V}$ be the impression type distribution used to compute the predicted weights. 
		For any  $\eps> 0$, if a set of weights $\predw$ returns a $(1-\eps)$-approximated solution in instance $\predI$, it yields an online flow allocation on instance $I$ whose value is at least
		$\max\{ (1-\eps)\opt - 2\gamma, \opt/(d+1) \} . $
		Here $\opt$ is the maximum flow value on instance $I$, $d$ is the diameter of this graph excluding vertex $t$, and $\gamma$ is the difference between two instances, defined by $|| \predI-I||_1$.
		
		%prediction error.  If $\{p_S^*\}_{S \subseteq V}$ is the actual type distribution in the offline instance, then $\gamma := \|\predprob - p^*\|_1$.

	\end{theorem}

This theorem can be interpreted as follows. If we sample the instance and compute the weights in it, these weights will work well and break through worst-case bounds when the type proportions are sampled well.  Indeed, the weights will perform well in nearby instances with similar type proportions.  Moreover, the algorithm never performs worse than a $\frac{1}{d+1}$ factor of optimal.   We remark that this is the best competitive ratio a deterministic integral algorithm can achieve because we show a lower bound on such algorithms in Section~\ref{subsec:worst_case_bound}. This example builds a recursive  version of the simple lower bound of $\frac{1}{2}$ on the competitive ratio of deterministic online  algorithms for the matching problem.
%Thus, in the case of a small number of types, we can observe a small sample of initial arrivals and use them to reliably learn the type distribution and use this learned distribution with the above theorem to get a good performance for these learned weights.

The technical proof is deferred to Appendix~\ref{sec:robustness}.  Recall that the key to showing existence of the weights was to construct a cut whose capacity is nearly the same as the value of the flow given by the weights.  To show robustness against nearby instances, we observe how this proof can be extended to nearby cuts.  This allows us to argue about the optimal value of the new instance. Standard calculations then let us connect the value of the predicted weights to this optimal value while losing only $O( \gamma)$ in the value of the flow.  To ensure the algorithm is never worse than a $\frac{1}{d+1}$ factor of the optimal, we guarantee that the algorithm always returns a maximal allocation.  

%Second, 
%we examine the parameter robustness of the weights and 
\subparagraph*{Parameter Robustness: } For the parameter robustness, we show that the performance degrades linearly in the relative error of the weight parameter.  
Thus, our algorithm has the same robustness guarantees shown in other works as well. 
	
%{\color{red} BEN: we probably need to state state some kind of lower bound for multi layer graphs.}
	
	\begin{theorem}
	\label{thm:flow_param-rob}	
		Consider a prediction of  $\predw_v$ for each vertex $v\in V$. Due to scale invariance, we can assume that the minimum predicted vertex weight $\predw_{min}=1$.
		Define the prediction error $\eta := \max_{v\in V} (\frac{\predw_v}{\alpha^*_v} ,\frac{\alpha_v^*}{\predw_v} ),$
		where $\{ \alpha_v^* \}_{v\in V}$ are vertex weights that can achieve an $(1-\epsilon)$-approximate solution and $\alpha^*_{min}=1$ for any fixed $\eps >0$.
		Employing predictions $\predw$, we can obtain a solution with competitive ratio
		$ \max(\frac{1}{d+1},\frac{1-\epsilon}{\eta^{2d}}), $
		where $d$ is the diameter of this graph excluding vertex $t$.
		
	\end{theorem}
	
When the prediction error $\eta$ approaches one, the performance smoothly approaches  optimal.	
While the above theorem involves comparing the propagation of the prediction errors in the performance analysis, we also investigate how inaccurate predictions can be adaptively corrected and improved in the 2-layered Adwords case.
For that case, we give an improved algorithm that can correct error in the weights, so that the loss is only $O(\log \eta)$. This result is in Appendix~\ref{subsec:bi_robust}.
%logarithmic in the error of the predictions.
% TL: I don't like the next sentence as the result is only interesting for a limited range of eta values
%This way, we show that even if the error is large, we can adapt to the error in the weights.
Moreover, we show that the way we adapt is \emph{tight} and the best possible up to constant factors for any algorithm given predicted weights with error $\eta$.

%\subsection{Technical Overview}

%\paragraph{Instance Robustness}
%\smallskip
%\noindent \textbf

%\paragraph{Parameter Robustness}
%\smallskip
%\noindent \textbf
The parameter robustness follows almost directly from the definition of the proportional assignment given by the weights.
In each layer, the over allocation (potentially above the capacity) can be easily bounded by a $\eta^2$ factor, resulting in a loss of at most $\eta^{2d}$ on a $d$ layer graph. We remark that directly using the weights ensures the algorithm is never worse than a $\frac{1}{d+1}$ factor of the optimal solution.
The technical proof, along with the weight-adapting technique for the 2-layered case, is deferred to Appendix~\ref{app:robustness}.

%% file: scheduling_results.tex
\section{Results on Load Balancing} \label{sec:load_balancing_results}

Next we  show results for restricted assignment load balancing problem  in our model.  In particular, we study the instance robustness and learnability of proportional weights for this problem.  The existence of useful weights and parameter robustness for predicting these weights were shown in prior work~\cite{PropMatchAgrawal,DBLP:conf/soda/LattanziLMV20}.  As discussed before, we focus on analyzing fractional assignments.

To describe our results we introduce the following notation.  Let $[m]$ denote the set of machines and $S$ denote a set of jobs.  Each job $j \in S$ has a size $p_j$ and a neighborhood $N(j)$ of feasible machines.  Given a set of positive weights $\{w_i\}_{i \in [m]}$ on the machines, we define a fractional assignment for each job $j$ by setting $x_{ij}(w) = \frac{w_i}{\sum_{i' \in N(j)} w_{i'}}$ for each $i \in N(j)$.  Let $\alg(S,w)$ be the fractional makespan on the jobs in $S$ with weights $w$ and similarly let $\opt(S)$ be the optimal makespan on the jobs in $S$.  Prior work~\cite{PropMatchAgrawal,DBLP:conf/soda/LattanziLMV20} shows that for any set of jobs $S$ and $\epsilon >0$ there exists weights $\alpha$ such that $\alg(S,w) \leq (1+\epsilon)\opt(S)$.

To describe our instance robustness result we consider an instance of the problem as follows.  Consider instances with $n$ jobs.  The type of a job is the subset of machines to which it can be assigned.  Let $S_j$ be the total size of jobs of type $j$ in instance $S$.  We consider relative changes in the instance and define the difference between instances $S$ and $S'$ as $\eta(S,S'):= \max_j \max \{ \frac{S_j}{S'_j}, \frac{S'_j}{S_j}\}$.  Our instance robustness result is given in the following theorem.
%\tl{Chenyang should check the following rephrasing.  Do we need to assume unit size for this case or does it work for any set of sizes? }
%{\color{red} Chenyang: It works for any set of sizes since consider fractional assignment, a job with size $t$ can be viewed as $t$ unit jobs. Thus, rather than say $S_j$ is the number of jobs, we might say the total size of jobs.  }

\begin{theorem} \label{thm:makespan-robustness}
	For any instance $S$ and $\epsilon >0$, let $w$ be weights such that $\alg(S,w) \leq (1+\eps)\opt(S)$.  Then for any instance $S'$ we have $\alg(S',w) \leq (1+\eps)^2\eta(S,S')^2\opt(S')$.  
\end{theorem}

Let $w$ be as in the statement of the theorem and $w'$ be weights such that $\alg(S',w') \leq (1+\eps)\opt(S')$.  Intuitively, we lose the first factor of $\eta(S,S')$ by bounding the performance of $w$ on $S'$ and the second factor by bounding the performance of $w'$ on $S$.

Next we study learnability. We give the first result showing these weights are learnable in any model.  In order to understand the sample complexity of learning the weights we need to consider an appropriate discretization of the space of possible weights.  For integer $R >0 $ and $\epsilon >0$, let $\cW(R) = \{ \alpha \in \R^{m} \mid \alpha_i = (1+\epsilon)^k, i \in [m], k \in \{0,1,\ldots,R\}\}$.  Additionally, let $p_{\max}$ be an upper bound on all jobs sizes.  The following theorem characterizes the learnability of the weights for restricted assignment load balancing.

\begin{theorem} \label{thm:learnability_makespan}
Let $\epsilon,\delta \in (0,1)$ be given and set $R = O(\frac{m^2}{\eps^2}\log(\frac{m}{\eps}))$ and let $\cD = \prod_{j=1}^n \cD_j$ be a product distribution over $n$-job restricted assignment instances such that $\E_{S \sim \cD}[\opt(S)] \geq \Omega(\frac{1}{\eps^2} \log(\frac{m}{\eps}))$.  There exists an algorithm which finds weights $w \in \cW(R)$ such that

\noindent$\E_{S \sim \cD}[\alg(S,w)] \leq (1+\eps)\E_{S \sim \cD}[\opt(S)]$with probability at least $1- \delta$ when given access to $s = \tilde{O}(\frac{m^3}{\eps^2}\log( \frac{1}{\delta}))$ independent samples $S_1,S_2,\ldots,S_s \sim \cD$.
\end{theorem}

Our techniques here are similar to that of the online flow allocation problem in that we use the samples to construct a ``stacked'' instance then compute a set of near optimal weights on this instance.  We then have to show that these weights work well in expectation with high probability.  This step necessitates the two assumptions in the theorem statement. First, we need that the expected optimal makespan is reasonably large so that the expected makespan is close to the maximum of the expected loads of the machines.  Second, we need that the instance is drawn from a product distribution so that the stacked instance converges in some sense to $s$ copies of the ``expected'' instance.  See Section~\ref{sec:apdx_scheduling} for complete arguments.

%% file: existweight.tex
\section{Existence of Useful Weights for Max Flow in 3-layer DAGs}\label{sec:existweight}

%TL: Removed "magic", too informal
In this section, we show the existence of useful weights for 3-layer DAGs. 
We first state the generalized theorem for online flow allocation on DAGs formally and give a reduction from DAGs to layered graphs such that if the weights in layered graphs satisfy some properties, the theorem can be proved. Then we focus on the 3-layered case and give the detailed proof of Theorem~\ref{thm:weight_existence_3_layer}. 

%We will give an algorithm to compute the vertex weights for the $3$-layered version, which can be easily generalized to the case of $d$-layered graphs and produces vertex weights with some properties needed for the case of general directed acyclic graphs.

%will show that there exists a set of vertex weights yielding a nearly maximum flow in DAGs and prove Theorem~\ref{thm:weight_existence}.
Consider a directed acyclic graph $G=(\{s,t\} \cup V, E)$, where each vertex $v\in V$ has capacity $C_v$. Our goal is to maximize the flow sent from $s$ to $t$ without violating any vertex capacity constraint. For any vertex $v$, let $d_v$ be the longest distance from $s$ to $v$. Define
$d_{uv} := d_{v}-d_{u}-1$ for each edge $(u,v)$. We claim the following theorem:

%\iffalse************
\begin{theorem}\label{thm:weight_DAG}
	For any edge $(u,v)\in E$, let $x_{uv}$ be the proportion of flow crossing this edge in all flow received by $u$.
        For any given $\epsilon\in (0,1)$,
		there exists a weight $\alpha_v$ for each vertex $v$ such that we can obtain a $(1-\epsilon)$-approximate solution by setting the proportion of flow out of $u$ to neighbor $v$ to be $ x_{uv} =  \frac{\alpha_v^{1/(2n)^{d_{uv}}}}{ \sum_{v'\in N_u} \alpha_{v'}^{1/(2n)^{d_{uv'}}}  },$
		where $n$ is the number of vertices and $N_u$ is the set of vertices pointed by $u$. Moreover, these weights can be obtained in a time of $O(d^2n^{d+1}\log(n/\epsilon)/\epsilon^2 )$ where $d$ is the diameter of this graph excluding vertex $t$.
%There exists a weight $\alpha_v$ for each vertex $v$ such that we can obtain a near-optimal solution if for any edge $(u,v)$,
%\[ x_{uv} =  \frac{\alpha_v^{\rho^{d_{uv}}}}{ \sum_{v'\in N_u} \alpha_{v'}^{ \rho^{d_{uv'}}}  },\]
%where $\rho$ is a parameter and $N_u$ is the set of vertices pointed by $u$.
\end{theorem}
%\fi%%%%%%%%%%

%To prove Theorem~\ref{thm:weight_DAG},
To prove Theorem~\ref{thm:weight_DAG},
we first reduce our model to a maximum flow problem in a $d$-layered graph where an algorithm must return the weights that obey certain structural properties. We then find such weights in the $d$-layered graph in Appendix~\ref{sec:apdxexistweights_d} within the claimed time and use them to find weights on the original graph. We provide the key ingredients for proving the $d$-layered result by proving the 3-layered case in Appendix~\ref{subsec:exist3weight} below.

%\textbf{Insert Reduction}
\subsection{Reduction from DAGs to Layered DAGs}
\label{subsec:existgenweight}
%\subsection{Reduction}

For any directed acyclic graph $G=(\{s,t\} \cup V, E)$ with vertex capacities, we construct a $d$-layered DAG $\Gr=(\{s,t\} \cup \Vr, \Er)$ where all arcs go from one layer to the next and
$d:=\max_{v\in V} d_v$.
%Our goal is to construct an ($s$-$t$) $d$-layer graph, which means in $\Gr$, there is a $d$-layered graph between $s$ and $t$.

We initialize $d$ empty vertex layers between $s$ and $t$ in $G^r$.
Copy vertex $t$ $d$ times and add them to $d$ different layers. They all have infinite capacities.
Use $\tilde{t}_j$ to represent the copy in the $j$-th layer.
For each vertex $v\in V$, copy $d_v$ times and add them to the previous $d_v$ layers.
Similarly,
use $\vr_j$ to represent the copy in the $j$-th layer.
We call the copy in the $d_v$-th layer the real copy of $v$ and other vertices the virtual copies of $v$.
Set the capacity of the real copy to be $C_v$ and the capacities of virtual copies to be infinite: $C_{\vr_{d_v}}:= C_v$ and $C_{\vr_j} = \infty$ for $j< d_v$.

Now we construct the edge set $\Er$.
For each $v \in V\cup \{t\}$, connect its two copies in every two neighboring layers.
Use $\Er_1$ to represent the set of these edges.
For each edge $(u,v) \in E$, connect the real copy of $u$ to the copy of $v$ in the $(d_u+1)$-th layer.
Use $\Er_2$ to represent the set of these edges.
In other words, we add $(\ur_{d_u}, \vr_{d_u+1})$ to $\Er_2$ for each $(u,v)\in E$.
Let $\Er:=\Er_1 \cup \Er_2$.
Finally, remove all vertices that can not be reached by $s$ in the new graph.
According to our construction, the edges can only occur in two neighboring layers, indicating that $\Gr$ is an ($s$-$t$) $d$-layered graph (see Fig~\ref{fig:reduction} as an illustration).

\begin{figure}[t]%{\textwidth}
	\centering
	\includegraphics[width=0.8\linewidth]{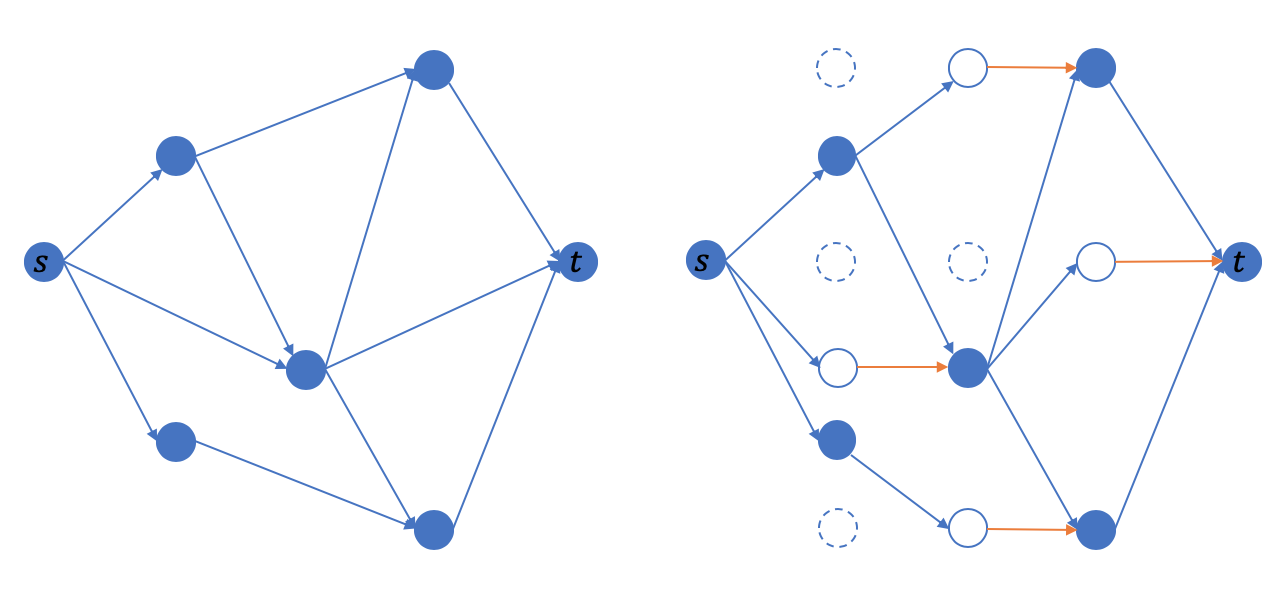}
	\caption{An illustration of the reduction. On the left is the original graph $G$ and on the right is the new $d$-layered graph $\Gr$. In graph $\Gr$, blue vertices are the real copies while the white ones are virtual. The dotted vertices are the vertices removed in the last stage.
		We color the edges in $\Er_1$ orange and the edges in $\Er_2$ blue.  }
	\label{fig:reduction}
\end{figure}

Clearly, any feasible $s$-$t$ flow in $G$ corresponds to a feasible $s$-$t$ flow in the new graph $\Gr$ and vice-versa. Thus, the values of the maximum flow in these two graphs are the same.

In order to obtain the vertex weights in the original DAG, we compute the vertex weights in the reduced $d$-layered graph and transfer these weights back to the original graph.  To be able to transfer the weights back to the original graph, we will need that the weights satisfy certain structural properties. In general, without these properties the weights compute on the $d$-layered graph will not be useful for obtaining weights in the original graph.  \footnote{As mentioned in the Techniques section, the vertex weights for $\Gr$ can be computed by a simple algorithm. However, using that simple algorithm, we may not be able to transfer these vertex weights back to the original graph.  In particular they may not satisfy the properties stated.}

%According to that algorithm, for each vertex $v$ in $G$, the weights of its copies in $\Gr$ are independent of each other, indicating that it is hard to obtain an unique weight for $v$ as described in Theorem~\ref{thm:weight_DAG}.

Now we state two properties and show that if there exists  weights for $\Gr$ with these two properties, then we can prove Theorem~\ref{thm:weight_DAG}.  The key property is the virtual-weight dependence, which ensures that the weights given to copies of a node created in the reduction can be interpreted in the original graph.

\begin{property}[Near optimality]\label{ppt:near_optimal}
	For the $d$-layered graph $\Gr$, a set of vertex weights $\{ \alpha_v \}$ is near optimal if we can obtain a near optimal solution by setting $x_{uv}=\frac{\alpha_v}{ \sum_{v'\in N_u} \alpha_{v'} }$
	for each edge $(u,v)\in E$.
\end{property}

\begin{property}[Virtual-weight dependence]\label{ppt:virtual_dependence}
	For the $d$-layered graph $\Gr$, a set of vertex weights $\{ \alpha_v \}$ has dependent virtual weights if for any two neighboring copies $\vr_j,\vr_{j+1}$ of any vertex $v\in G$, we have $\alpha_{\vr_{j}} = (\alpha_{\vr_{j+1}})^{\rho}$ for some $\rho > 0$ .
	
\end{property}

 %Using these properties, we prove Theorem~\ref{thm:weight_existence}. 

\begin{proof}[\textbf{Proof of Theorem~\ref{thm:weight_DAG}}]

	Assume that we have a set of vertex weights $\{\alpha\}$ of $\Gr$ with the two properties.
	The vertex weights $\{ \beta \}$ in $G$ are constructed by letting $\beta_v = \alpha_{\vr_{d_v}}$ for each $v\in V$. Namely, for each vertex $v$, its weight $\beta_v$ is the weight of $v$'s  copy in $\Gr$. Each edge $(u,v)$ in $G$ corresponds to edge $(\ur_{d_u}, \vr_{d_u+1})$ in $\Gr$.  We show the values of flow crossing the edge in both graphs with their respective weights are the same. This will prove that the weights $\{\beta\}$ return a near optimal solution in $G$.
	
	Consider any edge pair $(u,v)$ and $(\ur_{d_u}, \vr_{d_u+1})$. According to the allocation rules in Theorem~\ref{thm:weight_DAG}, for the original graph $G$, we have $x_{uv} =  \frac{\beta_v^{\rho^{d_{uv}}}}{ \sum_{w\in N_u} \beta_{w}^{\rho^{d_{uw}}}  },$ and according to the allocation rules in the near optimal property, for the reduced graph $\Gr$, we have $x_{u'_{d_u} v'_{d_u+1}} =  \frac{\alpha_{\vr_{d_u+1}} }{ \sum_{w\in N_u} \alpha_{\tilde{w}_{d_u+1}}  }.$ For any two neighboring copies $\vr_j,\vr_{j+1}$ of a vertex $v$, $\alpha_{\vr_{j}} = (\alpha_{\vr_{j+1}})^{\rho}$, we have  $\alpha_{\vr_{d_u+1}} = \alpha_{\vr_{d_v}} ^{\rho^{(d_v-d_u-1)}} =  \alpha_{\vr_{d_v}}^{ \rho^{d_{uv}} } = \beta_v^{ \rho^{d_{uv}} } . $ Thus, using weights $\{\beta_v\}$, for each edge pair, we have $x_{uv} = x_{\ur_{d_u} \vr_{d_u+1}} $, completing the proof of weight existence. 

To finish the proof, we will have to argue the existence of weights with these above two properties in $d$-layered DAGs and that they can be computed in the time claimed in the theorem. We use  $\rho = 1/(2n)$  to achieve these properties. We supply these proofs in Section~\ref{sec:apdxexistweights_d} in the Appendix to complete the proof.
\end{proof}

In the following, we will show how to obtain the weights with the two properties for the special case of 3-layered graphs. 
%The proof for the general $d$-layer case is in Section~\ref{sec:apdxexistweights_d}.

\subsection{Max Flow in 3-layered Graphs}
\label{subsec:exist3weight}

This section gives an algorithm to compute the requisite weights in three-layered graphs and this uses many of the key ideas needed for general DAGs.
If the reader is interested in an even simpler proof, we provide a simplified algorithm to compute vertex weights for two-layered (bipartite) graphs in Appendix~\ref{sec:apdxexistweights_bi}.

%Note that a set of weights is good if it satisfies the two properties.
We first focus on the near optimality property (Theorem~\ref{thm:weight_existence_3_layer}), showing that for any 3-layered graphs, our algorithm returns near optimal weights.
Then we claim that if the 3-layered graph is a reduction graph $\Gr$, the vertex weights satisfy the virtual-weight dependence property. Particularly, in the proof of the virtual-weight dependence property, the parameter $\rho$ is set to be $1/2n$.

%TL: added wrapfigure here
\begin{figure}[t]%{L}{0.55\textwidth}
	\centering
	\captionsetup{justification=centering}
	\includegraphics[width=0.5\linewidth]{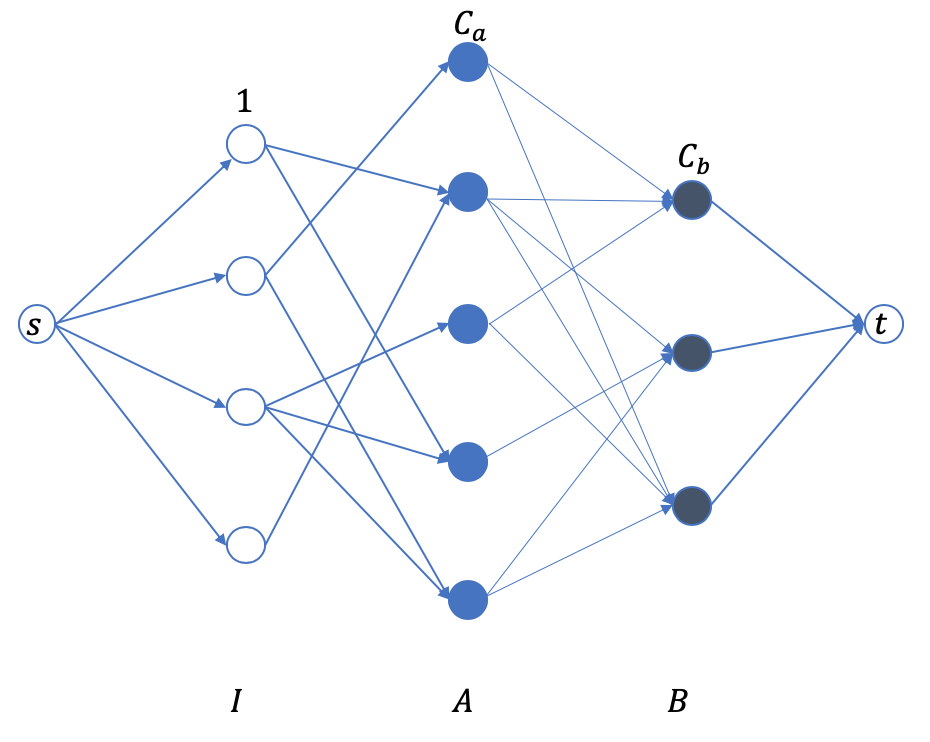}
	\caption{An $(s$-$t)$ 3-layered graph.}
	\label{fig:g3graphstructure}
\end{figure}

Consider an $(s$-$t)$ 3-layered graph $G=(\{s,t\}\cup I\cup A \cup B, E)$ (see Fig~\ref{fig:g3graphstructure}).
Each vertex $i$, $a$, $b$ in the $I$ layer, the $A$ layer and the $B$ layer, respectively, have capacity $1$, $C_a$ and $C_b$.

Use $N_{(i,A)}$, $N_{(a,I)}$, $N_{(a,B)}$ and $N_{(b,A)}$ to represent $i$'s neighborhood in $A$, $a$'s neighborhood in $I$, $a$'s neighborhood in $B$ and $b$'s neighborhood in the $A$ respectively.
For each $(i,a)\in E$, we use $x_{i,a}$ to denote the proportion of the flow sent from $i$ to $a$. Similarly, we define $y_{a,b}$ for each $(a,b)\in E$.
For each vertex $v$, use $Alloc_v$ to represent the total amount of flow sent to it.
The algorithm runs iteratively.
We use superscript $(t)$ to denote the value of variables in the end of iteration $t$.
For example, let $Alloc_v^{(t)}$ be the value of $Alloc_v$ in the end of iteration $t$.
But for simplicity,  we will use $Alloc_v$ directly to represent $Alloc_v^{(T)}$.
Let $\epsilon_{max} = \max(\epsilon_A,\epsilon_B)$ and $\epsilon_{min}=\min(\epsilon_A,\epsilon_B)$.

Our framework is stated in Algo~\ref{alg:3_layer_framework}.
We set weights for both layers and update them iteratively. Each vertex $b\in B$ updates its weight according to its allocation while each vertex $a\in A$ updates the weight, not only because of itself (self-decrease), but also due to its neighborhood in $B$ (forced-decrease).
$\cond$ is a Boolean function of $N_{(a,B)}$ for any $a\in A$. Namely, only when $N_{(a,B)}$ satisfies some conditions (represented by $\cond$), we let $\alpha_a$ do forced-decrease.
By defining different $\epsilon_A$, $\epsilon_B$ and $\cond$, we can obtain different algorithms.

%\begin{wrapfigure}{L}{0.6\textwidth}
%\begin{minipage}{0.6\textwidth}
\begin{algorithm}[t]
	\caption{The framework of the algorithm in $(s-t)$ 3-layered graphs}
	\label{alg:3_layer_framework}
	\KwIn{ $G= (\{s,t\}\cup I\cup A\cup B,\mathbb{E})$, $\{ C_a \}_{a \in A}$, $\{ C_b \}_{b \in B}$, parameter $\epsilon_A, \epsilon_B $ }
	
	Initialize $\alpha_a = 1$ $\forall a \in A$ and $\beta_b = 1$ $\forall b \in B$. \

	\For {iteration $1,2,...,T$ }
	{
		For each $(i,a)\in \mathbb{E}$ and $(a,b) \in \mathbb{E}$, let $x_{i,a} =\frac{\alpha_a}{\sum_{a'\in N_{(i,A)}} \alpha_{a'} } $ and $y_{a,b} =  \frac{\beta_b}{\sum_{b' \in N_{(a,B)} } \beta_{b'} }$. \
		
		For each vertex in the $A$ layer and the $B$ layer, let $Alloc_a = \sum_{i\in N_{(a,I)}} x_{i,a} $ and $Alloc_b = \sum_{a \in N_{(b,A)} } \min(Alloc_a, C_a) y_{a,b} $ respectively. \
		
		\For {each $b \in \mathbb{B}$ }
		{
			
			\If {$Alloc_b > (1+\epsilon_B)(1+\epsilon_A) C_b $}
			{
				$\beta_b \leftarrow \beta_b / (1+\epsilon_B) $. \

			}						
		}
		
		\For{each $a \in \mathbb{A}$}
		{

			\If{ $Alloc_a > (1+\epsilon_A)C_a$ }
			{

				$\alpha_a \leftarrow \alpha_a / (1+\epsilon_A) $  /* Self-decrease */

			}
			\ElseIf{ $\cond(N_{(a,B)}) =\true$}
			{

				$\alpha_a \leftarrow \alpha_a / (1+\epsilon_A) $ /* Forced-decrease */

			}
			
		}

	}
	\KwOut{ $\{\alpha_a\}_{a\in A}$ and $\{\beta_{b}\}_{b\in B}$ }
\end{algorithm}
%\end{minipage}\vspace{.3cm}
%\end{wrapfigure}

Since each weight decreases at most once per iteration, the potential minimum weight $\alpha_{min}$ in the $A$ layer and $\beta_{min}$ in the $B$ layer are $\frac{1}{(1+\epsilon_A)^T}$ and $\frac{1}{(1+\epsilon_B)^T}$ respectively after $T$ iterations. Then we can partition vertices in $A$ and $B$ into several classes (we also call them levels in the following):
\[ B(k) := \{ b\in B | \beta_b = (1+\epsilon_B)^k \beta_{min} \}, \]
\[ A(k) := \{ a\in A | \alpha_a = (1+\epsilon_A)^k \alpha_{min} \}. \]
For any vertex $v$, use $\lev(v)$ to denote the level that it belongs to. Similarly, define $\lev^{(t)}(v)$ be the value of $\lev(v)$ in iteration $t$. Clearly, if $\lev(v)=T$, we know that the weight of vertex $v$ has not decreased so far, while $\lev(v)=0$ indicates that its weight decreased in every iteration.

We now introduce four properties such that for any algorithm under this framework, if it satisfies these four properties, it will return a $(1-O(\epsilon))$-approximated solution when $T = \poly(n,\epsilon)$ ($n$ is the number of vertices).

\begin{property}[Increasing monotonicity]\label{ppt:1}
	For any vertex $v$ in $A \cup B$, in an iteration $t$, if its weight does not decrease (i.e. $\alpha_a^{(t)} = \alpha_a^{(t-1)}$ ), then we have $Alloc_v^{(t)} \geq Alloc_v^{(t-1)}$.% will also not decrease.
\end{property}

\begin{property}[Decreasing monotonicity]\label{ppt:2}
	For any vertex $v$ in $A \cup B$, in an iteration $t$, if its weight decreases (i.e. $\alpha_a^{(t)} < \alpha_a^{(t-1)}$), then we have $Alloc_v^{(t)} \leq Alloc_v^{(t-1)}$.% will not increase.
\end{property}

\begin{property}[Layer dominance]\label{ppt:3}
	For any vertex $a$ in the $A$ layer, in any iteration $t$, there exists at least one vertex $b\in N_{(a,B)}$ such that $\lev^{(t)}(b) \geq \lev^{(t)}(a)$.
\end{property}

\begin{property}[Forced decrease exemption]\label{ppt:4}
	
	For any vertex $a$ in the $A$ layer, in any iteration $t$,
	if there exists one vertex $b\in N_{(a,B)}$ with $\lev^{(t)}(b)=T$ or satisfying $\lev^{(t)}(b) - \lev^{(t)}(a) \geq \log(n/\epsilon_{max})/\epsilon_{min}$, then $\cond(N_{(a,B)})=\false$.
\end{property}

These properties will be used to show the following theorem:
\begin{theorem}\label{thm_g3}
	For any algorithm under our framework with the four properties, if $T=O(\frac{n\log(n/\epsilon_{max})}{\epsilon_{max}\epsilon_{min}} )$, the algorithm will return a $(1-O(\epsilon_{max}))$-approximated solution.
\end{theorem}

To prove Theorem~\ref{thm_g3}, we consider a new graph $G'$, a smaller graph by removing some vertices and edges.
We first construct an $s$-$t$ cut in $G'$. Then prove that the value of our solution in $G$ is at least $(1-O(\epsilon_{max}))$ times the value of this cut, thus at least $(1-O(\epsilon_{max}))$ times the maximum flow in $G'$. Finally, we show that
the optimal value in $G'$ is close to that in $G$, completing the proof.

We now define some vertex sets in order to construct $G'$. Given any integer $1+\frac{\log(n/\epsilon_{max})}{\epsilon_{min} }\leq \ell \leq T-1-\frac{\log(n/\epsilon_{max})}{\epsilon_{min} }$, we can define a gap in the $A$ layer: $\g(A):= \bigcup_{k=\ell}^{\ell'} A(k)$, where $\ell'=\ell+\frac{\log(n/\epsilon_{max})}{\epsilon_{min} }$. As shown in Fig~\ref{fig:3layergap}, all vertices in the $A$ layer are partitioned into three parts: $\g(A)$, $\g(A)^-$ and $\g(A)^+$, where $\g(A)^-:=\bigcup_{k=0}^{\ell-1} A(k)$ and $\g(A)^+:=\bigcup_{k=\ell'+1}^{T} A(k)$. Similarly, we can define $\g(B)$, $\g(B)^-$ and $\g(B)^+$ using the same $\ell$. The length of these two gaps are both $\frac{\log(n/\epsilon_{max})}{\epsilon_{min} }$ in order to make sure that regardless of the layer, any vertex above the gap has weight at least $\frac{n}{\epsilon_{max}}$ times than the weight of any vertex below the gap.

\begin{figure}[t]
		\centering
		\captionsetup[subfigure]{justification=centering}
	%\subfigure[]{
		\begin{subfigure}[t]{0.45\linewidth}
			\centering
			\includegraphics[height=7cm]{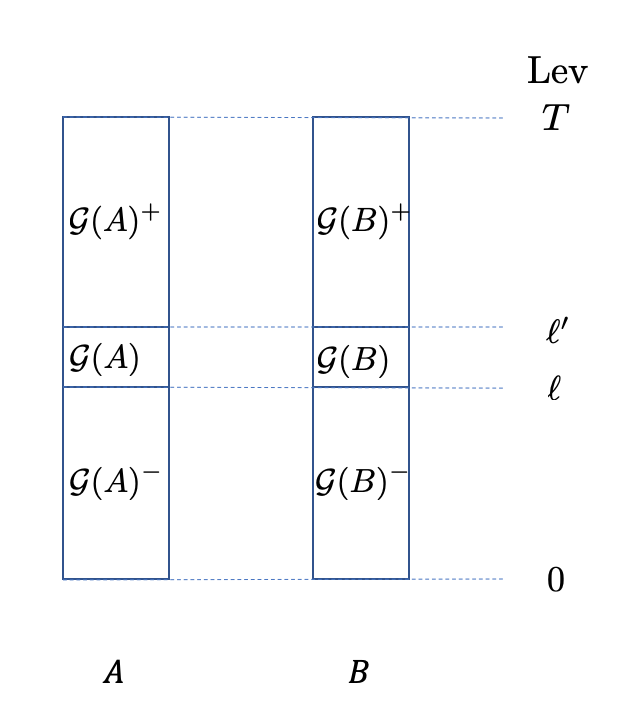}
			\caption{}
			\label{fig:3layergap}
		\end{subfigure}
	%}%
	%\subfigure[]{
		\begin{subfigure}[t]{0.5\linewidth}
			\centering
			\includegraphics[height=6.25cm]{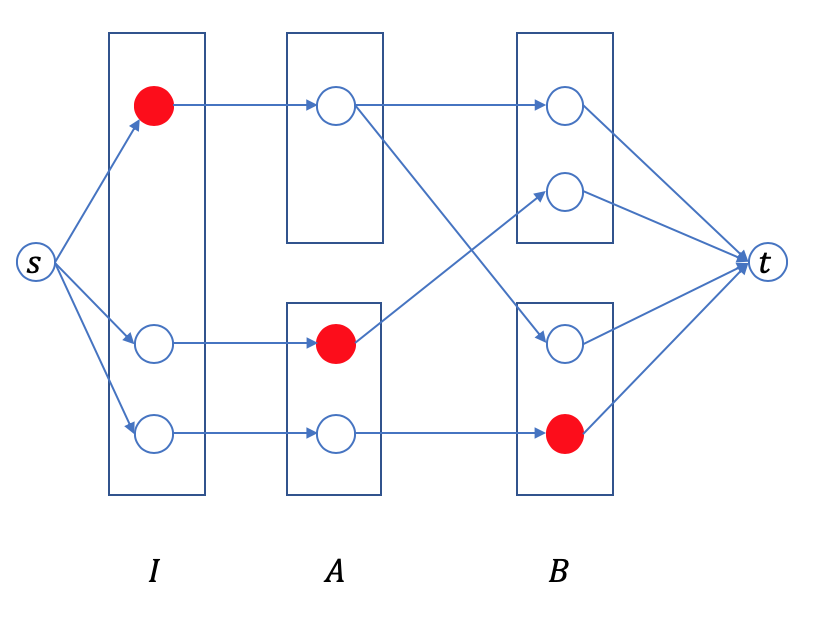}
			\caption{ }
			\label{fig:3layercut}
		\end{subfigure}
	%}%
		\caption{Fig~(a) is an illustration of the partition of the $A$ layer and the $B$ layer. Removing $\g(A)$ and $\g(B)$ can get the graph $G'$. Fig~(b) is an illustration of the vertex cut in graph $G'$. According to our rules, all red vertices are added to $\C$.}
		\label{fig:3layergap_and_cut}
\end{figure}

The graph $G'$ is created by removing all vertices in $\g(A)$ and $\g(B)$ and all edges adjacent to them. We give two rules to construct an $s$-$t$ vertex cut $\C$:

\begin{enumerate}%[itemsep=0mm]
\item For any $s$-$t$ path crossing both $\g(A)^-$ and $\g(B)^-$, add the vertex $b$ in the $B$ layer to $\C$.
\item Let $\g^+ = \g(A)^+ \cup \g(B)^+$. For any $s$-$t$ path crossing at least one vertex in $\g^+$, find the first vertex $v$ in $\g^+$ and add the vertex before $v$ to $\C$.
\end{enumerate}

See Fig~\ref{fig:3layercut} as an illustration.
Observe that $\C$ is a feasible $s$-$t$ vertex cut, meaning that all $s$-$t$ paths in $G'$ will be blocked if removing $\C$.
Use $\g(A^-,B^+)$ to represent the vertices in $\g(A)^-$ which are adjacent to at least one vertex in $\g(B)^+$. Then according to the two rules, we can compute the value of this cut:

\begin{equation}
C(\C) =  C(\g(B)^-) + C(\g(A^-,B^+)) + |N_{(\g(A)^+,I)}|  . \nonumber
\end{equation}

The value of our solution can be computed easily using the flow into $t$:
\begin{equation}
\val = \sum_{b\in \g(B)^-} \min(Alloc_b,C_b) + \sum_{b\in \g(B)} \min(Alloc_b,C_b)+ \sum_{b\in \g(B)^+} \min(Alloc_b,C_b)\nonumber
\end{equation}

In the following, we first show that $\sum_{b\in \g(B)^-} \min(Alloc_b,C_b)$ and $C(\g(B)^-)$ are close, and then establish the relationship between $\sum_{b\in \g(B)} \min(Alloc_b,C_b)+ \sum_{b\in \g(B)^+} \min(Alloc_b,C_b)$ and $ C(\g(A^-,B^+))+|N_{(\g(A)^+,I)}|.$ To prove the first statement, we need the following lemma:

\begin{lemma}\label{lem:b_lowerbound}
	If the increasing monotonicity property holds, after iteration $T$, $\forall b \in \bigcup_{k=0}^{T-1} B(k)$, we have $Alloc_b \geq C_b$.
\end{lemma}
\begin{proof}
	For any vertex $b$ in $\bigcup_{k=0}^{T-1} B(k)$, $\beta_b$ decreased at least once. Consider the last iteration $t$ that it decreased. In the beginning of that iteration, we have $ Alloc_b^{(t-1)} > (1+\epsilon_B)(1+\epsilon_A)C_b. $
	Since in one iteration, for any $a\in N_{(b,A)}$, $Alloc_a$ decreases at most $(1+\epsilon_A)$ and $y_{a,b}$ decreases at most $(1+\epsilon_B)$, we have
	$Alloc_b^{(t)} \geq \frac{Alloc_b^{(t-1)} }{(1+\epsilon_A)(1+\epsilon_B)} > C_b. $ After iteration $t$, $\beta_b$ did not decrease. Due to Property~\ref{ppt:1}, $Alloc_b$ also did not decrease, thus finally,
	$Alloc_b \geq Alloc_b^{(t)} \geq C_b.$
\end{proof}

According to the lemma above, the first statement can be proved easily: $\sum_{b\in \g(B)^-} \min(Alloc_b,C_b) = \sum_{b\in \g(B)^-} C_b = C(\g(B)^-).  $ To prove the second statement, we first show that for any vertex $b\in \g(B) \cup \g(B)^+$, $\min(Alloc_b,C_b)$ is close to $Alloc_b$, and then analyze $Alloc_b$ to complete this proof.

\begin{lemma}\label{lem:b_upperbound}
	If the decreasing monotonicity property holds, after iteration $T$, $\forall b \in \bigcup_{k=1}^{T} B_k$, we have $Alloc_b \leq (1+5\epsilon_{max})C_b$.
\end{lemma}

\begin{proof}
	The basic idea of this proof is similar to the proof of Lemma~\ref{lem:b_lowerbound}. Clearly, for any vertex $b\in \bigcup_{k=1}^{T} B(k)$, $\beta_{b}$ did not decrease in every iteration. Consider the last iteration $t$ that it did not decrease. In the beginning of that iteration, we have
$Alloc_b^{(t-1)} \leq (1+\epsilon_B)(1+\epsilon_A)C_b. $ Since in one iteration, for any $a\in N_{(b,A)}$, $Alloc_a$ increases at most $(1+\epsilon_A)$ and $y_{a,b}$ increases at most $(1+\epsilon_B)$, we have
	$   Alloc_b^{(t)} \leq Alloc_b^{(t-1)}(1+\epsilon_A)(1+\epsilon_B) \leq (1+5\epsilon_{max}) C_b. $ After iteration $t$, $\beta_{b}$ decreased in every iteration. Due to Property~\ref{ppt:2}, $Alloc_b$ did not increase, thus finally, $Alloc_b \leq Alloc_b^{(t)} \leq (1+5\epsilon_{max})C_b. $
\end{proof}

According to the lemma above, for any vertex $b \in \g(B) \cup \g(B)^+$ we obtain a relationship between $\min(Alloc_b,C_b)$ and $Alloc_b$:
\begin{equation}
\min(Alloc_b,C_b) \geq \frac{Alloc_b}{1+5\epsilon_{max}} \geq (1-O(\epsilon_{max})) Alloc_b \nonumber
\end{equation}

Now we  analyze $Alloc_b$. For any vertex $b$ in $\g(B)^+ \cup \g(B)$, according to our assignment, we have:
\begin{equation}
\begin{aligned}
Alloc_b =& \sum_{a \in N_b } \min(Alloc_a,C_a)y_{a,b}\nonumber
\end{aligned}
\end{equation}

For a vertex $a\in A$, define $N_{(a,B)}^+ = N_{(a,B)} \cap (\g(B)^+ \cup \g(B)) $.
Summing $Alloc_b$ over all $b$ in $\g(B)^+ \cup \g(B)$, we have

\begin{equation}
\begin{aligned}
\sum_{b\in \g(B)^+ \cup \g(B)}Alloc_b = & \sum_{a\in A} \min(Alloc_a,C_a) \sum_{b\in N_{(a,B)}^+} y_{a,b} \nonumber
\end{aligned}
\end{equation}

As mentioned above, the length of the gap is $\frac{\log(n/\epsilon_{max})}{\epsilon_{min} }$, meaning that any vertex above the gap has weight $\frac{n}{\epsilon_{max}}$ times the weight of any vertex below the gap. If $N_{(a,B)} \cap \g(B)^+  \neq \emptyset$, the total proportion assigned to $\g(B)^-$ is at most $\epsilon_{max}$. Namely,
\[ \sum_{b\in N_{(a,B)}^+} y_{a,b} \geq 1-\epsilon_{max}. \]
Thus, due to Property~\ref{ppt:3}, we have

\begin{equation}
\begin{aligned}\label{eq:1}
\sum_{b\in \g(B)^+ \cup \g(B)}Alloc_b \geq & \sum_{a\in \g(A)^+   } \min(Alloc_a,C_a) (1-\epsilon_{max}) \\
& + \sum_{a\in \g(A^-,B^+)   } \min(Alloc_a,C_a) (1-\epsilon_{max}) 
\end{aligned}
\end{equation}

Applying the same technique as in the proof of Lemma~\ref{lem:b_lowerbound}, we have the following lemma:
\begin{lemma}\label{lem:a_lowerbound}
	If Property~\ref{ppt:1} holds, $\forall a \in \bigcup_{k=0}^{T-1} A(k)$, if in the last iteration its weight did self-decrease, then we have $Alloc_a \geq C_a$.
\end{lemma}

According to Property~\ref{ppt:4}, for any vertex $a \in \g(A^-,B^+)$, $\alpha_a$ did self-decrease in the last iteration $t$ that it decreased, because there already existed a vertex $b\in N_a$ with $\lev^{(t)}(b)-\lev^{(t)}(a)$ large enough due to the definition of $\g(A^-,B^+)$. Thus, we have

\begin{equation}
\begin{aligned}
\sum_{a\in \g(A^-,B^+)   } \min(Alloc_a,C_a) (1-\epsilon_{max})  &\geq \sum_{a\in \g(A^-,B^+)   } C_a (1-\epsilon_{max}) \nonumber\\
&= (1-\epsilon_{max}) C(\g(A^-,B^+)).
\end{aligned}
\end{equation}

\begin{lemma}\label{lem:a_upperbound}
	If Property~\ref{ppt:2} holds, $\forall a \in \bigcup_{k=1}^{T} A(k)$, we have $Alloc_a \leq (1+3\epsilon_A)C_a$.
\end{lemma}

This lemma can also be proved by the same technique as in the proof of Lemma~\ref{lem:b_upperbound}. Thus, for any vertex $a\in \g(A)^+ \cup \g(A)$, we have
\[ \min(Alloc_a,C_a) \geq (1-O(\epsilon_{max})) Alloc_a. \]

Now we start to analyze $Alloc_a$. Due to the gap in the $A$ layer, we know at most $O(\epsilon_{max})$ proportion of flow from $N_{(\g(A)^+,I)}$ is assigned to $\g(A)^-$. Namely,

\begin{equation}
\begin{aligned}\label{eq:alloc}
\sum_{a\in \g(A)^+} Alloc_a + \sum_{a\in \g(A)} Alloc_a &\geq (1-O(\epsilon_{max})) |N_{(\g(A)^+,I)}| 
\end{aligned}
\end{equation}

Combining all related inequalities, we can obtain the relationship between the value of our solution and the size of cut $\C$:

\begin{equation}
\begin{aligned}
\val &\geq (1-O(\epsilon_{max})) C(\C) - (1-O(\epsilon_{max}))\sum_{a\in \g(A)} Alloc_a \nonumber\\
& \geq (1-O(\epsilon_{max})) \opt(G')- (1-O(\epsilon_{max}))\sum_{a\in \g(A)} Alloc_a
\end{aligned}
\end{equation}

In the following, we will show that by selecting appropriate $\ell$, both $\opt(G)-\opt(G')$ and $\sum_{a\in \g(A)} Alloc_a$ can be made very small compared to $\val$. To prove this, we need to divide $A$ into two sets $P$ and $Q$ first:
\[ P = \{ a\in A| N_{(a,B)} \cap B(T) \neq \emptyset \}, \]
\[ Q = \{ a\in A | N_{(a,B)} \cap B(T) = \emptyset \}. \]
Let $\g(P) = P \cap \g(A)$ and $\g(Q) = Q \cap \g(A)$.

Clearly, if we add $\g(A)$ and $\g(B)$ back to the graph, the maximum flow will increase at most $C(\g(P)) + C(N_{(\g(Q),B)} \cup \g(B))$. Namely,
\[ \opt(G)-\opt(G') \leq C(\g(P)) + C(N_{(\g(Q),B)}\cup \g(B)). \]

Checking $\sum_{a\in \g(A)} Alloc_a$, we can obtain the following inequality due to Lemma~\ref{lem:a_upperbound}:
\begin{equation}
\begin{aligned}
\sum_{a\in \g(A)} Alloc_a &= \sum_{a\in \g(P)} Alloc_a + \sum_{a\in \g(Q)} Alloc_a \\ \nonumber
&\leq (1+O(\epsilon_{max})) C(\g(P)) + (1+O(\epsilon_{max})) \sum_{a\in \g(Q)} \min(Alloc_a,C_a)
\end{aligned}
\end{equation}

Recall that $\sum_{a\in \g(Q)} \min(Alloc_a,C_a)$ is the total amount of flow sent from $\g(Q)$ to $N_{(\g(Q),B)}$.
To complete the proof, we try to bound this value by $C( N_{(\g(Q),B)} )$.

According to Lemma~\ref{lem:b_upperbound}, the flow sent from $\g(Q)$ to $N_{(\g(Q),B)} \setminus B(0)$ can be bounded by $C(N_{(\g(Q),B)}\setminus B_0)$, because for any vertex $b$ in the $B$ layer except $B(0)$, the total amount of flow that it received is at most $(1+\epsilon_{max})C_b$.

The range of $\ell$ is $[1+\frac{\log(n/\epsilon_{max})}{\epsilon_{min} }, T-1-\frac{\log(n/\epsilon_{max})}{\epsilon_{min} }]$. According to Property~\ref{ppt:3}, the proportion of flow sent from $\g(Q)$ to $B(0)$ is very small, at most $O(\epsilon_{max})$. Thus, the total amount of flow sent from $\g(Q)$ to $N_{(\g(Q),B)}$ can be bounded:
\[  \sum_{a\in \g(Q)} \min(Alloc_a,C_a) \leq (1+O(\epsilon_{max})) C( N_{(\g(Q),B)} ).\]

Then we have
\begin{equation}
\sum_{a\in \g(A)} Alloc_a  \leq (1+O(\epsilon_{max})) ( C(\g(P)) + C( N_{(\g(Q),B)} )). \nonumber
\end{equation}

Now, if we prove that $C(\g(P)) + C(N_{(\g(Q),B)}\cup \g(B))$ is at most $O(\epsilon_{max})\val$, the whole proof is completed. We do this with a simple averaging argument.

\begin{lemma}\label{lem:gap}
	When $T=O(\frac{n\log(n/\epsilon_{max})}{\epsilon_{max}\epsilon_{min}})$, there exist an appropriate $\ell$ such that
	\[  C(\g(P)) + C( N_{(\g(Q),B)} \cup \g(B)) \leq O(\epsilon_{max})\val \]
\end{lemma}

\begin{proof}
	Summing $C(\g(P)) + C( N_{(\g(Q),B)} \cup \g(B)) $ over all potential $\ell$ (notice $\g(P)$ and $\g(Q)$ are defined by $\g(A)$ that depends on $\ell$), we have
	\begin{equation}
	\begin{aligned}
	&\sum_{\ell=1+\frac{\log(n/\epsilon_{max})}{\epsilon_{min}}}^{ T-1-\frac{\log(n/\epsilon_{max})}{\epsilon_{min} } }C(\g(P)) + C( N_{(\g(Q),B)} \cup \g(B)) \\
	&\leq  \frac{\log(n/\epsilon_{max})}{\epsilon_{min}} \sum_{k=0}^{T-1} C(P(k)) + \frac{n\log(n/\epsilon_{max})}{\epsilon_{min}} \sum_{k=0}^{T-1} C(B(k))\nonumber
	\end{aligned}
	\end{equation}
	This inequality holds because for each $k \in [0,T-1]$, $C(P(k))$ occurs at most $\frac{\log(n/\epsilon_{max})}{\epsilon_{min}} $ times and $C(Q(k))$ occurs at most $\frac{n\log(n/\epsilon_{max})}{\epsilon_{min}}$ times.
	
	Due to Property~\ref{ppt:4}, every vertex in $P$ only did self-decreases.
	Then according to Lemma~\ref{lem:a_lowerbound} and Lemma~\ref{lem:b_lowerbound}, we have
	\begin{equation}
	\sum_{k=0}^{T-1} C(P(k)) + C(B(k)) \leq 2\val\nonumber
	\end{equation}
	Combing the two inequalities above, we have
	\begin{equation}
	\begin{aligned}
	&\sum_{\ell=1+\frac{\log(n/\epsilon_{max})}{\epsilon_{min}}}^{ T-1-\frac{\log(n/\epsilon_{max})}{\epsilon_{min} } }C(\g(P)) + C( N_{(\g(Q),B)} \cup \g(B)) \\\nonumber
	&\leq  \frac{2n\log(n/\epsilon_{max})}{\epsilon_{min}} \val
	\end{aligned}
	\end{equation}
	Taking the average over all potential $\ell$,
	\begin{equation}
	\begin{aligned}
	&\frac{1}{ T-2- \frac{2\log(n/\epsilon_{max})}{\epsilon_{min}} }\sum_{\ell=1+\frac{\log(n/\epsilon_{max})}{\epsilon_{min}}}^{ T-1-\frac{\log(n/\epsilon_{max})}{\epsilon_{min} } }C(\g(P)) + C(N_{(\g(Q),B)}\cup \g(B)) \\
	&\leq  \frac{2n\log(n/\epsilon_{max}) /\epsilon_{min} }{  T-2- 2\log(n/\epsilon_{max})/\epsilon_{min}} \val\\
	&\leq O(\epsilon_{max}) \val,	\nonumber
	\end{aligned}
	\end{equation}
	when $T = \Omega(\frac{n\log(n/\epsilon_{max})}{\epsilon_{max}\epsilon_{min}})$.
\end{proof}

\begin{proof}[\textbf{Proof of Theorem~\ref{thm_g3}}]
	
	Combining all related inequalities, we have the following inequality:
	\begin{equation}
	\begin{aligned}
	\val \geq & (1-O(\epsilon_{max}))(\opt(G') - \sum_{a\in \g(A)} Alloc_a) \\
	\geq & (1-O(\epsilon_{max})) (\opt(G) - (\opt(G) -\opt(G') + \sum_{a\in \g(A)} Alloc_a  )  ) \\\nonumber
	\geq & (1-O(\epsilon_{max})) \opt(G) - O(\epsilon_{max})\val
	\end{aligned}
	\end{equation}
	Thus, we get a $(1-O(\epsilon_{max}))$-approximate solution.
	
\end{proof}

We give our $\cond$ function in Algo~\ref{alg:cond}.
Our final algorithm is designed by letting $\epsilon_A = \epsilon_B/n$ and using this $\cond$ function.
In the following, we will show that our algorithm has the four properties mentioned above.
\begin{algorithm}[t]
	\caption{$\cond(N_{(a,B)})$}
	\label{alg:cond}

	Let $\beta_{max}^{(a)} = \max_{b'\in N_{(a,B)}} \beta_{b'}$, $N_{(a,B)}^* := \{ b\in N_{(a,B)} | \beta_b = \beta_{max}^{(a)}  \} $. \
	
	\If{ $\forall b \in N_{(a,B)}^*$, $\beta_b$ decreases in this iteration and $\lev(b) - \lev(a) < \log(n/\epsilon_{max}) / \epsilon_{min}$  }
	{
		
		return $\true$
	}
	\Else
	{
		return False
	}
\end{algorithm}

\begin{lemma}\label{lem:ppt1}
	Increasing monotonicity (Property~\ref{ppt:1}) holds if $\epsilon_A = \epsilon_B/(2n)$ and we use $\cond$ in Algo~\ref{alg:cond}.
\end{lemma}
\begin{proof}
	For any vertex $a\in A$, if $\alpha_a$ does not decrease, all $x_{i,a}$ will not decrease. Since $Alloc_a = \sum_{i:a\in N_i} x_{i,a} $, $Alloc_a$ will not decrease. But the situation is different in the $B$ layer.
	
	Recall the assignment rule of the $B$ layer. For any $b \in B$, let $Alloc_b = \sum_{a \in N_{(b,A)} } \min(Alloc_a,C_a) y_{a,b}$.
	If $\beta_b$ doesn't decrease, all $y_{a,b}$ will not decrease. However, some $a\in N_{(b,A)}$, $\min(Alloc_a,C_a)$ may decrease.
	
	In one iteration, $\min(Alloc_a,C_a)$ decreases at most $(1+\epsilon_A)$. If the claim below is proved, $y_{a,b}$ will increase at least $(1+\epsilon_B/n)$. Since $\epsilon_A = \epsilon_B/(2n)$, $Alloc_b$ will not decrease, proving that Property~\ref{ppt:1} holds.

	\begin{claim}\label{claim:1}
		For any vertex $a\in N_{(b,A)}$ in any iteration $t$, if $\min(Alloc_a,C_a)$ decreases, $y^{(t)}_{a,b} \geq (1+\epsilon_B/n)y^{(t-1)}_{a,b}$
	\end{claim}
	
	\begin{proof}[Proof of Claim~\ref{claim:1}]
		For simplicity, in this proof, we use $y_{a,b}$ and $y'_{a,b}$ to denote $y^{(t-1)}_{a,b}$ and $y^{(t)}_{a,b}$ respectively.
		If $\alpha_a$ does a self-decrease, $\min(Alloc_a,C_a)$ does not decrease. So the only reason that $\min(Alloc_a,C_a)$ decreases is that $\alpha$ does a forced-decrease, indicating that for all $b' \in N_{(a,B)}^*$, $\beta_{b'}$ decreases in this iteration using $\cond$ in Algo~\ref{alg:cond}. We show that their decrease is sufficient to guarantee that the proportional allocation $y_{a,b}$ increases by enough to offset the decrease of $Alloc_a$.
		
		\begin{equation}
		\begin{aligned}
		y'_{a,b} &= \frac{\beta'_b}{\sum_{b'\in N_{(a,B)}} \beta'_{b'} }\\\nonumber
		& =  \frac{\beta'_b}{\sum_{b'\in N^*_{(a,B)} } \beta'_{b'} + \sum_{b'\in N_{(a,B)} \setminus N_{(a,B)}^*} \beta'_{b'} }\\
		& \geq \frac{\beta_b}{\sum_{b'\in N^*_{(a,B)}} \beta_{b'}/(1+\epsilon_B) + \sum_{b'\in N_{(a,B)} \setminus N_{(a,B)}^*} \beta_{b'} }\\
		& = \frac{(1+\epsilon_B)\beta_b}{\sum_{b'\in N^*_{(a,B)} } \beta_{b'} +(1+\epsilon_B) \sum_{b'\in N_{(a,B)} \setminus N_{(a,B)}^*} \beta_{b'} } \\
		& = \frac{(1+\epsilon_B) \sum_{b'\in N_{(a,B)}} \beta_{b'} }{\sum_{b'\in N^*_{(a,B)}} \beta_{b'} +(1+\epsilon_B) \sum_{b'\in N_{(a,B)}\setminus N_{(a,B)}^* }\beta_{b'} } \cdot \frac{\beta_{b}}{ \sum_{b'\in N_{(a,B)}} \beta_{b'} } \\
		& = (1+ \frac{\epsilon_B\sum_{b'\in N^*_{(a,B)}} \beta_{b'} }{\sum_{b'\in N^*_{(a,B)}} \beta_{b'} +(1+\epsilon_B) \sum_{b'\in N_{(a,B)} \setminus N_{(a,B)}^*} \beta_{b'} } ) y_{a,b} \\
		& \geq (1+\frac{\epsilon_B}{n} )y_{a,b}
		\end{aligned}
		\end{equation}
	\end{proof}
\end{proof}

\begin{lemma}\label{lem:ppt2}
	Decreasing monotonicity (Property~\ref{ppt:2}) holds if $\epsilon_A = \epsilon_B/(2n)$ and we use $\cond$ in Algo~\ref{alg:cond}.
\end{lemma}
\begin{proof}
	This property can be proved similarly. For any vertex $a\in A$, if $\alpha_a$ decreases, all $x_{i,a}$ will not increase, so $Alloc_a$ will not increase.
	
	For any vertex $b\in B$, when $\beta_{b}$ decreases, all $y_{a,b}$ will not increase, but for some $a \in N_{(b,A)}$, $\min(Alloc_a,C_a)$ may increase. We give a similar claim to prove this property:
	\begin{claim}\label{claim:2}
		For any vertex $a\in N_{(b,A)}$ in any iteration $t$, if $\min(Alloc_a,C_a)$ increases, then $y^{(t)}_{a,b} \leq \frac{1}{(1+\epsilon_B/n)}y^{(t-1)}_{a,b}$
	\end{claim}
	\begin{proof}
		For simplicity, in this proof, we use $y_{a,b}$ and $y'_{a,b}$ to denote $y^{(t-1)}_{a,b}$ and $y^{(t)}_{a,b}$ respectively.
		The increase of $\min(Alloc_a,C_a)$ indicates that $\alpha_a$ does not change in this iteration.
		If $\alpha_a$ does not decrease, we know that either $\lev(N_{(a,B)}^*) - \lev(a) \geq \log(n/\epsilon_{max})/\epsilon_{min}$, or $\exists b^*\in N^*_{(a,B)}$, such that $\beta_{b^*}$ does not decrease.
		
		If $\lev(N_{(a,B)}^*) - \lev(a) \geq \log(n/\epsilon_{max})/\epsilon_{min}$, the last iteration that $\alpha_a$ decreased is due to itself. If the weights of $a$ and $N_{(a,B)}^*$ decreased together, $\lev(N_{(a,B)}^*) - \lev(a)$ would still be less than $\log(n/\epsilon_{max})/\epsilon_{min}$. Thus, due to Lemma~\ref{lem:a_lowerbound}, $Alloc_a \geq C_a$, indicating that $\min(Alloc_a,C_a)$ does not increase even if $Alloc_a$ increases.
		
		For the second case, we employ the similar technique in the proof of Claim~\ref{claim:1}:
		\begin{equation}
		\begin{aligned}
		y'_{a,b} &= \frac{\beta'_{b}}{\sum_{b' \in N_{(a,B)} } \beta'_{b'} } \\\nonumber
		& =  \frac{\beta'_{b}}{\beta'_{b^*} + \sum_{b' \in N_{(a,B)}, b' \neq b^* } \beta'_{b'} } \\
		& \leq \frac{\beta_{b}/(1+\epsilon_B)}{\beta_{b^*} + \sum_{b' \in N_{(a,B)}, b' \neq b^* } \beta_{b'}/(1+\epsilon_B) } \\
		& = \frac{\beta_{b}}{(1+\epsilon_B)\beta_{b^*} + \sum_{b' \in N_{(a,B)}, b' \neq b^* } \beta_{b'}}  \\
		& = \frac{\sum_{b'\in N_{(a,B)}} \beta_{b'} }{(1+\epsilon_B)\beta_{b^*} + \sum_{b' \in N_{(a,B)}, b' \neq b^* } \beta_{b'}}  \cdot \frac{\beta_{b}}{ \sum_{b'\in N_{(a,B)}} \beta_{b'} } \\
		& = \frac{y_{a,b}}{ 1+ \frac{\epsilon_B \beta_{b^*} }{ \sum_{b'\in N_{(a,B)}} \beta_{b'} } }\\
		& \leq \frac{y_{a,b}}{1+\epsilon_B/n}
		\end{aligned}
		\end{equation}
		
	\end{proof}
\end{proof}

\begin{lemma}
	Layer dominance (Property~\ref{ppt:3}) holds if we use $\cond$ in Algo~\ref{alg:cond}.
\end{lemma}
\begin{proof}
	Assume that $\forall b\in N^*_{(a,B)}$, $\lev(b) < \lev(a)$. Consider the first iteration that this situation occurs. Clearly, in the beginning of that iteration, $\lev(N^*_{(a,B)}) = \lev(a)$ and in that iteration, $\alpha_a$ did not decrease while $\forall b\in N_{(a,B)}^*$, $\beta_{b}$ decreased, contradicting our algorithm. Thus, there must exist at least one $b\in N_{(a,B)}$ such that $\lev(b) \geq \lev(a)$.
\end{proof}

\begin{lemma}
	Forced decrease exemption (Property~\ref{ppt:4}) holds if we use $\cond$ in Algo~\ref{alg:cond}.
\end{lemma}
This lemma can be proved directly from the description of our algorithm.
Since our final algorithm satisfies the four properties, according to Theorem~\ref{thm_g3}, when $T$ is large enough, it will return a near-optimal solution. We complete the last piece of the proof for Theorem~\ref{thm:weight_existence_3_layer} by giving the following lemma for the running time:

\begin{lemma}~\label{lem:3_running_time}
	If $\epsilon_A = \epsilon/(2n)$ and $\epsilon_B = \epsilon$, the running time of our algorithm is $O(n^4\log(n/\epsilon)/\epsilon^2)$.
\end{lemma}
\begin{proof}
	According to Theorem~\ref{thm_g3}, the number of iterations is $O(\frac{n\log(n/\epsilon_{max})}{ \epsilon_{max} \epsilon_{min} })$. Since $\epsilon_A=\epsilon/(2n)$ and $\epsilon_B=\epsilon$, the number of iterations is $O(\frac{n^2\log(n/\epsilon)}{ \epsilon^2})$. In each iteration, we need to compute $x_{uv}$ for each edge in $G$ and update the weight of each vertex. Thus, the running time of an iteration is $O(n^2)$, completing this proof.
\end{proof}

Finally, we prove the virtual-weight dependence of these weights.
\begin{theorem}\label{thm:3_layer_good_weights}
	Under the framework of Algo~\ref{alg:3_layer_framework}, if we let $\epsilon_A = \epsilon_B/(2n)$ and use $\cond$ in Algo~\ref{alg:cond}, given an $(s$-$t)$ 3-layered reduction graph $\Gr$, for any two neighboring copies $\vr_j,\vr_{j+1}$ of any vertex $v$, we have
	\[\alpha_{\vr_{j+1}} = (\alpha_{\vr_j})^{2n} .\]
\end{theorem}

\begin{proof}
	Clearly, the weights of the $I$ layer can be arbitrary. So we only need to consider the case that $\vr_j,\vr_{j+1}$ are in the $A$ layer and the $B$ layer respectively. For simplicity, use $a$ and $b$ to denote these two vertices.
	
	Since $a$ is a virtual copy, it has infinite capacity. According to our framework, it can only do forced-decrease.
	
	Recall the construction of $\Er_1$ and $\Er_2$ in the reduction graph.
	Vertex $b$ is the only neighbor of $a$ in the $B$ layer. Thus, due to Algo~\ref{alg:cond}, any time, if $\beta_b$ decreases, $\alpha_a$ also decreases.
	Assuming that $\beta_{b}$ decreases $k$ times, we have
	\begin{equation}
	\begin{aligned}
	\frac{\log \beta_b}{\log \alpha_a} = & \log (1+\epsilon_B)^{(-k)} / \log (1+\epsilon_A)^{(-k)}\\
	=&\log (1+\epsilon_B)/\log (1+\epsilon_A)\\
	= & \log (1+2n \epsilon_A) / \log (1+\epsilon_A) \\\nonumber
	\approx & \log (1+\epsilon_A)^{2n} / \log (1+\epsilon_A) \\
	= & 2n,
	\end{aligned}		
	\end{equation}
	completing this proof.
\end{proof}

Using Theorem~\ref{thm_g3} and Theorem~\ref{thm:3_layer_good_weights}, if the reduction graph $\Gr$ is $(s$-$t)$ 3-layered DAG, our algorithm returns a set of vertex weights with near optimality and virtual-weight dependence, indicating a set of good vertex weights for the original directed acyclic graph. Moreover, these weights can be obtained in time $O(n^4\log(n/\epsilon)/\epsilon^2)$ by Lemma~\ref{lem:3_running_time} completing the analysis for the 3-layered DAGs.

The basic algorithmic framework, as well as the properties and the cut constructing rules can be generalized to more general $d$-layered graphs smoothly.
We give the algorithm and the proof for general $d$-layered graphs in Appendix~\ref{sec:apdxexistweights_d}.

%% file: learnability.tex
\section{Learnability of Predictions for Online Flow Allocation in 2-layered graphs}\label{sec:learnability}

In this section, we consider the learnability of the vertex weights.
%under some mild assumptions.
We first introduce our formal definition of learnability.
%Given a DAG $G$, the instance set $\Pi$ can be viewed as a set of different impression sets $I$. Intuitively, for any unknown distribution $\cD$ over $\Pi$, we wish that we can learn the best fit vertex weights of $G$.
Given a DAG $G$, let $\Pi$ be a class of instances of the online flow allocation problem on the fixed graph $G$.  That is, each $I \in \Pi$ represents a different sequence of impressions.  There is also an unknown distribution $\cD$ over instance $\Pi$, and we can access independent samples from this distribution.  Our goal is to find the best set of weights for this distribution.

More precisely,
we follow the definition proposed by Gupta and Roughgarden \cite{DBLP:journals/siamcomp/GuptaR17} for application specific algorithm selection.

\begin{definition}
	A learning algorithm $L$ $(\epsilon,\delta)$-learns the optimal algorithm in an algorithm set $\cA$ if for any distribution $\cD$ over the instance set $\Pi$, given $s$ instances $I_1,I_2,...,I_s \sim \cD$, with probability at least $1-\delta$, $L$ outputs an algorithm $\learnA \in \cA$ such that
	\[  | \E_{I \sim \cD}[\obj(\learnA,I)] - \E_{I \sim \cD}[ \obj(A^*,I) ] |\leq \epsilon,\]
	where $\obj(A,I)$ is the objective value obtained by algorithm $A$ in instance $I$ and $A^*$ is the algorithm in $\cA$ with optimal $\E_{I \sim \cD}[ \obj(A,I)] $.
	
\end{definition}

Consider the simplest strategy that uses the predicted weights directly to send flow for a given instance.
This strategy can be viewed as an algorithm set $\cA$, where each set of predicted weights $\{\alpha\}$ corresponds to an algorithm $A(\alpha) \in \cA$.  Thus in our setting we are interested in the case when $\obj(A(\alpha),I)= R(\alpha,I)$.  Recall that $R(\alpha,I)$ is the amount of flow we route to the sink using weights $\alpha$ on impression set $I$.  We say these vertex weights are PAC-learnable if there exists a learning algorithm $L$ that $(\epsilon,\delta)$-learns the optimal algorithm in $\cA$. 

The case of general distributions is challenging for our problem.  In particular, we cannot use some standard approaches for general distributions since for fixed $I$, $R(\alpha,I)$ is neither a convex nor concave function of the weights $\alpha$.  Thus we restrict the class of distributions considered

Our results hold for the case when the distribution $\cD$ is a product distribution, i.e. $\cD = \cD_1 \times \cD_2 \times \ldots \times \cD_m$ and each $\cD_i$ is an independent distribution over impressions.  Such distributions have been studied in the context of self improving algorithms~\citep{SelfImproving}.
%Namely, with high probability, we can find the near-optimal vertex weights for a distribution $\cD$.  We study the case when the distribution $\cD$ is a product distribution.
To simplify the presentation below, we focus on the i.i.d. case.  Below  we let $\cD$ be a fixed unknown distributions over impressions and we are interested in instances $I$ sampled according to $\cD \times \cD \times \ldots \times \cD = \cD^m$.  We note that the proofs easily generalize to the case of more general product distributions, as we mainly require independence across impressions.

\begin{theorem}\label{thm:learnability}
	Assume that for any instance $I\in \Pi$, each impression is i.i.d. sampled from an unknown distribution $\cD$.
	Namely, an instance $I$ is sampled from the distribution $\D^{m}$.
	%{\color{red}: TL: we should unpack what the mild assumptions are here.}
	Under some mild assumptions, for any $\epsilon,\delta \in (0,1)$, there exists a learning algorithm such that, after observing $O(\frac{n^2}{\epsilon^2}\ln(\frac{n\log n}{\delta}))$ instances,
	it will return a set of weights $\{ \predw \}$, satisfying that with probability at least $1-\delta$,
	\[ \E_{I \sim \D^m}[R(\predw,I)] \geq (1-\epsilon) \E_{I \sim \D^m}[R(\alpha^*,I)] \]
where $R(\alpha,I)$ is the value of the fractional flow obtained by applying $\alpha$ to instance $I$ and \\
$\alpha^* = \arg\max\limits_{\alpha } \E_{I \thicksim \D^m}[ R(\alpha,I) ]$.
\end{theorem}

To prove this theorem for DAG's and $d$-layered graphs we need to make some assumptions about how well the optimal solution saturates the internal vertices of the graph.  Intuitively, we need to disallow vertices that the optimal solution only sends a negligible amount of flow through.  See Appendix~\ref{sec:apdxlearnability} for a formal discussion.

Use $d$ to denote the diameter of the graph $G$.
In this section, we focus on the proof of the learnability for the case that $d=2$. %The proof of the more general case is given in Appendix~\ref{sec:apdxlearnability}.
When $d=2$, this problem can be seen as the maximum cardinality bipartite matching problem. The graph excluding the sink $t$ consists of two layer $I$ and $A$, where $A$ is given initially and $I$ is sampled from $\cD^m$. We claim the following theorem.

\begin{theorem}\label{thm:bi_learn}
	Assume that for any instance $I\in \Pi$, each impression is i.i.d. sampled from an unknown distribution $\cD$.
	Given any $\epsilon,\delta \in (0,1)$, if any vertex $a\in A$ has a capacity $C_a \geq \poly(1/\epsilon)$, there exists a learning algorithm such that, after observing $O(\frac{n^2}{\epsilon^2}\ln(\frac{n\log n}{\delta}))$ instances,
	it will return a set of weights $\{ \predw \}$, satisfying that with probability at least $1-\delta$,
	
	\begin{equation}\label{eq:bi_learn}
	\E_{I \sim \D^m}[R(\predw,I)] \geq (1-\epsilon) \E_{I \sim \D^m}[R(\alpha^*,I)].
	\end{equation}
	
\end{theorem}

We first introduce our learning algorithm in Algorithm~\ref{alg:learning}. The algorithm is very simple: we construct a new instance by averaging over all impressions from $s$ randomly sampled instances, and compute its weights as $\{\predw\}$.

\begin{algorithm}[t]
	\caption{Learning Algorithm}
	\label{alg:learning}
	\KwIn{$\epsilon\in (0,1)$, $s$ sampled impression set $I_1,I_2,...,I_s$}
	
	Construct a new impression set $\predI$, where the number of impressions with each type is its (rounded) mean value in the $s$ samples.\

	Compute the $(1-\epsilon)$-approximate weights $\{ \predw \}$ for this new instance.\

	\KwOut{$\{ \predw \}$}
\end{algorithm}

We start by defining some notations. For an instance $I$ and a set of weights $\{ \alpha \}$, let $Alloc_a(\alpha, I)$ be the number of the impressions assigned to $a$ and let $R_{a}(\alpha, I):=\min (Alloc_a(\alpha, I),C_a)$ to represent advertiser $a$'s real contribution to the objective value.

Since $\{ \predw \}$ is $(1-\epsilon)$-approximate for instance $\predI$, we have

\begin{equation}\label{eq:bi_near_optimal}
R(\predw,\predI) \geq (1-\epsilon)\opt(\predI) \geq (1-\epsilon) R(\alpha^*,\predI),
\end{equation}
where $\opt(\predI)$ is the optimal value of instance $\predI$.

Thus, if we prove the following two inequalities:
\begin{equation}\label{eq:bi_learn_1}
	R(\alpha^*,\predI) \geq (1-O(\epsilon) ) \E_{I\sim \cD^m} [R(\alpha^*,I)],
\end{equation}
and
\begin{equation}\label{eq:bi_learn_2}
	\E_{I\sim \cD^m} [R(\predw,I)] \geq (1-O(\epsilon)) R(\predw,\predI),
\end{equation}
Theorem~\ref{thm:bi_learn} can be obtained directly.

We consider these two inequalities one by one. By the definition of $R(\alpha,I)$, we have
\begin{equation}
	R(\alpha,I) = \sum_{a\in A} \min(Alloc_a(\alpha,I), C_a ).
\end{equation}

Due to the concavity of the $\min$ function and Jensen's inequality, for any weights $\{\alpha\}$
\begin{equation}\label{eq:bi_learn_1_1}
R(\alpha,\E[I]) = \sum_{a\in A} \min(\E[Alloc_a(\alpha,I)], C_a ) \geq \sum_{a\in A}
\E[\min(Alloc_a(\alpha,I), C_a )] = \E[R(\alpha,I)].
\end{equation}

Now if we can prove $R(\alpha,\E[I]) $ and $R(\alpha,\predI) $ are close, Eq~\eqref{eq:bi_learn_1} can be proved.

\begin{lemma}\label{lem:bi_sample_complexity}
	Given any $\epsilon >0$, $\delta \in (0,1]$ and vertex weights $\{ \alpha \}$, if the number of instances $s$ is no less than $ O(\frac{n}{\epsilon^2}\ln(\frac{n}{\delta}))$,
	with probability at least $1-\delta$, ,
	\begin{equation}\label{eq:bi_sample_comlexity}
	|  R(\alpha,\E[I])  - R(\alpha,\predI)|  \leq O(\epsilon) R(\alpha,\E[I])
%	|\min(Alloc_a(\alpha,\predI), C_a  ) -\min(\E[Alloc_a(\alpha,I)],C_a) | \leq  O(\epsilon) \min(\E[Alloc_a(\alpha,I)],C_a).
	\end{equation}
\end{lemma}

\begin{proof}
	Consider a vertex $a\in A$.
	Due to the property of the min function, we only need to show that $Alloc_a(\alpha,\predI)$ is close to $\E[Alloc_a(\alpha,I)]$.
	Since each impression is i.i.d. sampled, $sAlloc_a(\alpha,\predI)$ can be viewed as the sum of $sm$ i.i.d. random variables $x_{i,a}(\alpha) \in [0,1]$.
	For simplicity, let $\mu = \E[Alloc_a(\alpha,I)]$.

	Employing Chernoff's inequality, we have
	
	\begin{equation}
	\begin{aligned}
	&\Pr[ |Alloc_a(\alpha,\predI) - \mu| \geq \epsilon \mu + \epsilon/n   ] \\
	& = \Pr[ |sAlloc_a(\alpha,\predI) - s\mu| \geq (\epsilon+ \epsilon/(n\mu) )s\mu  ]\\
	& \leq 2\exp(- (\epsilon+ \epsilon/(n\mu) )^2 s \mu /4 ) \\
	& = 2\exp(- \epsilon^2( 1+2/(n\mu) + 1/(n\mu)^2 )s\mu/4 )\\
	& = 2\exp(-\epsilon^2(\mu + 1/(n^2\mu) + 2/n ) /4  )s \\
	& \leq 2\exp(-s\epsilon^2/n)
	\end{aligned}
	\end{equation}
	
	If for each $a\in A$, we have
	\begin{equation}\label{eq:bi_learn_1_2}
		|Alloc_a(\alpha,\predI) - \E[Alloc_a(\alpha,I)]| \leq \epsilon \E[Alloc_a(\alpha,I)]+ \epsilon/n,
	\end{equation}
	then according to the property of $\min$ function, we will get
		\begin{equation}
	|\min (Alloc_a(\alpha,\predI),C_a) - \min(\E[Alloc_a(\alpha,I)],C_a)| \leq  \epsilon\min(\E[Alloc_a(\alpha,I)], C_a )+ \epsilon/n.
	\end{equation}
	Summing the above over the $n$ impressions, we have
	\begin{equation}
		\begin{aligned}
		|  R(\alpha,\E[I])  - R(\alpha,\predI)| \leq& \sum_{a\in A}  |\min (Alloc_a(\alpha,\predI),C_a) - \min(\E[Alloc_a(\alpha,I)],C_a)|\\
		\leq & \sum_{a\in A} \epsilon \min(\E[Alloc_a(\alpha,I)], C_a )+ \epsilon/n \\
		= &\epsilon R(\alpha,\E[I]) + \epsilon \\
		\end{aligned}
	\end{equation}
	
	It is reasonable to assume that the number of satisfied impressions is at least $1$. So if for each $a\in A$, Eq~\eqref{eq:bi_learn_1_2} holds, we have
	\begin{equation}
	\begin{aligned}
	|  R(\alpha,\E[I])  - R(\alpha,\predI)| \leq& \epsilon R(\alpha,\E[I]) + \epsilon \\
	\leq & O(\epsilon) R(\alpha,\E[I])
	\end{aligned}
	\end{equation}
	
	Thus, we can bound the probability of the event that $|  R(\alpha,\E[I])  - R(\alpha,\predI)|  \geq O(\epsilon) R(\alpha,\E[I]) $:
	\begin{equation}
		\Pr[ |  R(\alpha,\E[I])  - R(\alpha,\predI)|  \geq O(\epsilon) R(\alpha,\E[I])  ] \leq \Pr[\exists a\in A, |Alloc_a(\alpha,\predI) - \mu| \geq \epsilon \mu + \epsilon/n ] .
	\end{equation}

	Due to the union bound,
	\begin{equation}
	\begin{aligned}
	\Pr[\exists a\in A, |Alloc_a(\alpha,\predI) - \mu| \geq \epsilon \mu +\epsilon/n ] \leq 2n\exp(-s\epsilon^2/n)
	\end{aligned}
	\end{equation}
	
	Letting $\delta$ be $2n\exp(-s\epsilon^2/n)$ , we obtain $s = O(\frac{n}{\epsilon^2}\ln(\frac{n}{\delta}) )$, completing this proof.
	
\end{proof}

%Note that we need to prove the lemma above holds for any potential set of vertex weights $\{ \alpha \}$. 
To prove Theorem~\ref{thm:bi_learn}, we need to show that the above lemma holds, not just for a single fixed set of weights, but any set of weights that could be output by the learning algorithm.  This can be accomplished by setting $\delta$ in the above lemma appropriately and then applying a union bound over all sets of weights.  
According to the weight computing algorithm, the number of potential weight sets is $O((\log n)^n)$. Thus to make the union bound argument go through, we need to let the number of samples be $O(\frac{n^2}{\epsilon^2}\ln(\frac{n\log n}{\delta}))$ in Theorem~\ref{thm:bi_learn}.
%{\color{red} RR: I didn't quite get this argument for increasing $s$ to $n^2$. Can you add one more sentence?}

Combining Lemma~\ref{lem:bi_sample_complexity} and Eq~\eqref{eq:bi_learn_1_1}, the first inequality Eq~\eqref{eq:bi_learn_1} can be proved.

By Lemma~\ref{lem:bi_sample_complexity}, the second inequality in Equation~\ref{eq:bi_learn_2} holds if we prove the following lemma.

\begin{lemma}\label{lem:bi_learn_2}
	For any $\epsilon \in (0,1)$ and any $a\in A$, we have
	\begin{equation}\label{eq:bi_learn_2_1}
	\E[R_a(\predw, I)] \geq (1-O(\epsilon)) \min(\E[Alloc_a(\predw,I) ], C_a)
	\end{equation}
	if $C_a = \Omega(\frac{1}{\epsilon^2}(\ln\frac{1}{\epsilon}))$.
\end{lemma}
\begin{proof}
	We divide the proof into three cases based on the value of $\E[Alloc_a(\predw,I) ]$:
	
	(1) when $\E[Alloc_a(\predw,I) ] \geq (1-\epsilon) C_a$
	
	(2) when $1\leq \E[Alloc_a(\predw,I) ]< (1-\epsilon) C_a$
	
	(3) when $\E[Alloc_a(\predw,I) ]< 1 $
	
	\textbf{Case 1}
	According to the definition of the expectation, we can do the following expansion:
	\begin{equation}
	\begin{aligned}
	&\min(\E[Alloc_a(\predw,I) ], C_a) -  \E[R_a(\predw, I)]  \\
	=&\min(\E[Alloc_a(\predw,I) ], C_a) - \sum_{I\in \Pi} \Pr[I] R_a(\predw, I) \\
	=&\sum_{I\in \Pi} \Pr[I] ( \min(\E[Alloc_a(\predw,I) ], C_a) - R_a(\predw, I) ) \\
	\end{aligned}	
	\end{equation}
	
	Intuitively, we want to show that for most instance, $\min(\E[Alloc_a(\predw,I) ], C_a) - R_a(\predw, I) $ is small. We first construct a set of good instances with small $\min(\E[Alloc_a(\predw,I) ], C_a) - R_a(\predw, I) $, and then prove the probability that instance $I$ does not belong to this set is very small. More specifically,
	we define a set of good instances $\cG = \{ I \mid Alloc_a(\predw,I)\geq (1-2\epsilon) C_a \}$. Clearly, for each instance $I \in \cG$,
	
	\begin{equation}
	\min(\E[Alloc_a(\predw,I) ], C_a) - R_a(\predw, I) \leq 2\epsilon C_a \leq \frac{2\epsilon}{1-\epsilon}\min(\E[Alloc_a(\predw,I) ], C_a)
	\end{equation}
	
	The remaining part is to show that $\Pr[ I \notin \cG ]$ is very small. Employing Chernoff's inequality, we have
	\begin{equation}
	\begin{aligned}
	&\Pr[I \notin \cG] \\
	= & \Pr[Alloc_a(\predw,I) < (1-2\epsilon)C_a ] \\
	\leq & \Pr[Alloc_a(\predw,I)< \frac{1-2\epsilon}{1-\epsilon}\E[ Alloc_a(\predw,I) ] ]\\
	\leq & \exp(-\frac{1}{2}(\frac{\epsilon}{1-\epsilon})^2 \E[ Alloc_a(\predw,I) ] ) \\
	\leq & \exp(-\frac{\epsilon^2}{2(1-\epsilon)} C_a )\\	
	\end{aligned}	
	\end{equation}
	When $C_a \geq \frac{2(1-\epsilon)}{\epsilon^2} \ln(\frac{1}{\epsilon})$, $\Pr[I \notin \cG]$ is at most $\epsilon$.
	
	Combining the two inequalities above, we complete the proof of this case:
	\begin{equation}
	\begin{aligned}
	&\min(\E[ Alloc_a(\predw,I) ], C_a) -  \E[R_a(\predw, I)]    \\
	=&\sum_{I\in \Pi} \Pr[I] ( \min(\E[ Alloc_a(\predw,I) ], C_a) -R_a(\predw, I) ) \\
	=&\sum_{I \in \cG} \Pr[I] (\min(\E[ Alloc_a(\predw,I) ], C_a) -R_a(\predw, I)  ) \\
	&+ \sum_{I \notin \cG} \Pr[I] ( \min(\E[ Alloc_a(\predw,I) ], C_a) -R_a(\predw, I)  )  \\
	\leq &  \frac{2\epsilon}{1-\epsilon}\min(\E[ Alloc_a(\predw,I) ], C_a) + \Pr[I \notin \cG]\min(\E[ Alloc_a(\predw,I) ], C_a) \\
	\leq & O(\epsilon)  \min(\E[ Alloc_a(\predw,I) ], C_a) .
	\end{aligned}	
	\end{equation}

	\textbf{Case 2}
	In this case, since $\E[ Alloc_a(\predw,I) ]$ is less than $(1-\epsilon)C_a$, we have
	
	\begin{equation}
	\begin{aligned}
	&\min(\E[ Alloc_a(\predw,I) ], C_a) = \E[ Alloc_a(\predw,I) ].
	\end{aligned}
	\end{equation}
	
	Thus,
	\begin{equation}
	\begin{aligned}
	&\min(\E[ Alloc_a(\predw,I) ], C_a) - \E[R_a(\predw, I)] \\
	=&\E[ Alloc_a(\predw,I) ] - \E[R_a(\predw, I)]\\
	= & \int_{0}^{\infty}\Pr[Alloc_a(\predw,I)= t]  ( Alloc_a(\predw,I) -R_a(\predw, I) ) dt.
	\end{aligned}
	\end{equation}
	
	Note that $R_a(\predw, I) = \min (Alloc_a(\predw,I),C_a)$. Clearly, if $Alloc_a(\predw,I) \leq C_a$, $Alloc_a(\alpha,I) -R_a(\predw, I) = 0 $. Otherwise, $Alloc_a(\predw,I) -R_a(\predw, I) =  Alloc_a(\predw,I) - C_a$. Thus,
	\begin{equation}
	\begin{aligned}
	&\min(\E[ Alloc_a(\predw,I) ], C_a) - \E[R_a(\predw, I)]\\
	= & \int_{C_a}^{\infty}\Pr[Alloc_a(\predw,I)= t]  ( t-C_a) dt \\
	= & \int_{C_a}^{\infty}\Pr[Alloc_a(\predw,I)\geq t]  dt.
	\end{aligned}
	\end{equation}
	
	The remaining part is to show that for any $t>C_a$, $\Pr[Alloc_a(\predw,I) \geq t]$ is very small.
	For any $t > C_a$, we can find a $\gamma\geq \frac{\epsilon}{1-\epsilon}$ such that $t = (1+ \gamma)\E[ Alloc_a(\predw,I) ]$. Employing Chernoff's inequality again, we can obtain
	
	\begin{equation}\label{eq:bi_learn_2_2}
	\begin{aligned}
	&Pr[Alloc_a(\predw,I)\geq t] \\
	= & Pr[Alloc_a(\predw,I) \geq (1+ \gamma) \E[Alloc_a(\predw,I)] ] \\
	\leq & \exp(-\E[ Alloc_a(\predw,I) ] \gamma^2/4 ) \\
	= & \exp(-\frac{\gamma^2}{4(1+\gamma)} t ) \\
	\leq & \exp(-\frac{\epsilon^2}{4(1-\epsilon)} t ).
	\end{aligned}
	\end{equation}
	For simplicity, we let $\rho = \frac{\epsilon^2}{4(1-\epsilon)}$, thus
	
	\begin{equation}\label{eq:bi_learn_2_3}
	\begin{aligned}
	&\min(\E[ Alloc_a(\predw,I) ], C_a) - \E[R_a(\predw, I)]\\
	= & \int_{C_a}^{\infty}\Pr[Alloc_a(\predw,I)\geq t]   dt \\
	\leq &  \int_{C_a}^{\infty} \exp( -\rho t)  dt \\
	= & \frac{1}{\rho} \exp(-\rho C_a).
	\end{aligned}
	\end{equation}
	
	When $C_a \geq \frac{1}{\rho} \ln(\frac{1}{\rho\epsilon} )$, we can complete the proof of this case:	
	\begin{equation}
	\min(\E[ Alloc_a(\predw,I) ], C_a) - \E[R_a(\predw, I)]\leq \epsilon \leq \epsilon \min(\E[ Alloc_a(\predw,I) ], C_a) .
	\end{equation}
	
	\textbf{Case 3}
	In this case, $\E[ Alloc_a(\predw,I) ]$ is very small, less than $1$, so we still have
	\begin{equation}
	\min(\E[ Alloc_a(\predw,I) ], C_a) = \E[ Alloc_a(\predw,I) ].
	\end{equation}
	Intuitively, when $\E[ Alloc_a(\predw,I) ]$ is very small, the contribution of each impression $i$ to $\E[ Alloc_a(\predw,I) ]$ and $\E[R_a(\predw, I)]$ should be very close.
	For each impression $i$, use random variable $x_{i,a}(\predw) \in [0,1]$ to represent the space occupied by impression $i$ in advertiser $a$. Clearly, $Alloc_a(\predw,I)= \sum_{i \in I} x_{i,a}(\predw)$.
	
	To analyze each impression's contribution to $R_a(\predw, I)$, we introduce a new random variable $y_{i,a}(\predw) = x_{i,a}(\predw) \cdot  \bfo_{ \{  Alloc_a^{-i}(\predw,I)\leq C_a -1\}}$, where $Alloc_a^{-i}(\predw,I) = \sum \limits_{i' \in I, i' \neq i} x_{i',a}(\predw)$, representing the load of vertex $a$ without $i$,
	and $\bfo_{ \{  Alloc_a^{-i}(\predw,I)\leq C_a -1\}}$ is an indicator of the event that the load of vertex $a$ without $i$ is less than $C_a - 1$. Clearly, we have
	\begin{equation}
	R_a(\predw, I) = \min(Alloc_a(\predw,I) ,C_a) \geq \sum_{i\in I} y_{i,a}(\predw).
	\end{equation}
	This equation holds because if $Alloc_a^{-i}(\predw,I)\leq C_a -1$, impression $i$'s contribution to $R_a(\predw, I)$ is exactly $x_{i,a}(\predw) $ since $Alloc_a(\predw,I) $ must be less than $C_a$, or otherwise, $y_{i,a}(\predw)=0$. Taking the expectation of both sides, we have
	\begin{equation}
	\E[R_a(\predw, I)]\geq \sum_{i\in I} \E(y_{i,a}(\predw)) = \sum_{i\in I} \E[ x_{i,a}(\predw) \cdot  \bfo_{ \{  Alloc_a^{-i}(\predw,I)\leq C_a -1\}} ].
	\end{equation}
	Since each impression is sampled independently, random variable $x_{i,a}(\predw)$ and $\bfo_{ \{  Alloc_a^{-i}(\predw,I)\leq C_a -1\}}$ are independent. Thus,
	\begin{equation}
	\begin{aligned}
	\E[R_a(\predw, I)]&\geq\sum_{i\in I} \E[ x_{i,a}(\predw) \cdot  \bfo_{ \{  Alloc_a^{-i}(\predw,I)\leq C_a -1\}} ] \\
	&= \sum_{i\in I} \E[ x_{i,a}(\predw)] \cdot \E[  \bfo_{ \{  Alloc_a^{-i}(\predw,I)\leq C_a -1\}} ] \\
	& = \sum_{i\in I} \E[ x_{i,a}(\predw)] \cdot \Pr[  Alloc_a^{-i}(\predw,I)\leq C_a -1 ]
	\end{aligned}	
	\end{equation}
	Using Markov inequality to estimate $ \Pr[  Alloc_a^{-i}(\predw,I)> C_a -1 ]$, we can obtain
	\begin{equation}
	\begin{aligned}
	\Pr[  Alloc_a^{-i}(\predw,I)> C_a -1 ] & \leq \frac{E(Alloc_a^{-i}(\predw,I) ) }{ C_a-1} \\
	& \leq \frac{1}{C_a-1}
	\end{aligned}
	\end{equation}
	When $C_a \geq \frac{1}{\epsilon} + 1$, this probability is at most $\epsilon$. Combining the above two inequalities, we can complete the proof of this case:
	\begin{equation}
	\begin{aligned}
	\E[R_a(\predw, I)]
	& = \sum_{i\in I} \E[ x_{i,a}(\predw)] \cdot \Pr[  Alloc_a^{-i}(\predw,I)\leq C_a -1 ] \\
	& \geq \sum_{i\in I} \E[ x_{i,a}(\predw)] \cdot (1-\epsilon) \\
	& = (1-\epsilon) \E [ Alloc_a(\predw,I) ].
	\end{aligned}	
	\end{equation}	
\end{proof}

%% file: robustness.tex
\section{Instance Robustness of Predictions for Online Flow Allocation} \label{sec:robustness}

In this section we consider the online flow allocation problem.  Recall that this problem is defined on a DAG $G = (V \cup {t}, E)$.  Each vertex $v \in V$ has a capacity $C_v$ and the distinguished vertex $t$ is a sink in the DAG, i.e. $t$ has no outgoing arcs and every node can reach the sink.

The structure of the graph $G$ is known offline.  Source vertices arrive online and reveal their outgoing arcs connecting to the rest of $G$.  Recall that source vertices have no incoming arcs.  We refer to these online vertices $I$ as impressions.  When $i \in I$ arrives, it is connected to a subset of vertices $N_i \subseteq V$.  At the time of arrival, the algorithm must irrevocably decide a (fractional) flow of value at most 1 from $i$ to $t$ while obeying the vertex capacities (taking into account flow from other impressions).  The goal is to allocate flow online to maximize the total flow that reaches $t$.

We study this online problem in the presence of predictions in two different models.  In the first model, we assume that we can predict the instance directly and bound the performance by the instance predicted error. In the second, we are given access to predictions of weights for a proportional allocation scheme, which were shown to exist in Section~\ref{sec:existweight}.  
Moreover, we state a worst-case bound for integral version of this online problem to show that our algorithm is really competitive.
%In the second model, we assume that the number of impression types is small and we are given access to predictions of the proportions of each type in the offline instance.  

We focus on the instance robustness here. The results for the parameter robustness and the worst-case bound are presented in Appendix~\ref{app:robustness}.

%We outline the results here, but primarily focus on the first model.  See Appendix~\ref{app:robustness} for the remaining results.

%\subsection{Instance Robustness}\label{subsec:instance_robust}

Consider an instance of the online flow allocation problem, consisting of an offline DAG $G = (V \cup{t},E)$ and a set of impressions arriving online.  Let $i$ be an impression type and $m_i$ be the number of impressions in this type. An instance $I$ can be denoted by a vector $<m_1,...,m_i,...>$.

%and let $S\subseteq V$.  We say that $i$ has type $S$ if $N_i = S$.
%For such an instance, define $\{p^*_S\}_{S \subseteq V}$ to be the impression type distribution.  Here $p_S^*$ represents the proportion of impressions with type $S$ in this instance.  That is, if $m$ is the number of impressions then $p_S^* = \frac{1}{m} \sum_{i \in I} \bfo_{\{ N_i = S\}}$.

%Intuitively, under some random settings, the impression type distributions of different online instance are close to each other.
Here we consider online proportional algorithms with predictions of the entire instance $\predI$. 
We consider the error in these predictions to be given by the $\ell_1$ norm:
\[\gamma := || \predI-I ||, \]
where $I$ is the real instance.
Learning such an instance from past data under this notion of error is well understood~\citep{Han2015MinimaxEO,pmlr-v40-Kamath15}.  In order to convert these predictions into our algorithm, we use the predicted instance to compute the weights for a proportional allocation scheme.

Recall the results for the instance robustness.

\begin{theorem}[Theorem~\ref{thm:dist_robustness} Restated] \label{thm:dist_robustness_restated}
	
	%Let $\{\predprob_S\}_{S \subseteq V}$ be the impression type distribution used to compute the predicted weights. 
	For any constant $\eps> 0$, if a set of weights $\predw$ returns a $(1-\eps)$-approximated solution in instance $\predI$, it yields an online flow allocation in instance $I$ whose value is at least
	$\max\{ (1-\eps)\opt - 2\gamma, \opt/(d+1) \} . $
	Here $\opt$ is the maximum flow value, $d$ is the diameter of this graph excluding vertex $t$, and $\gamma$ is the difference between two instances, defined by $|| \predI-I||_1$.
	
	%prediction error.  If $\{p_S^*\}_{S \subseteq V}$ is the actual type distribution in the offline instance, then $\gamma := \|\predprob - p^*\|_1$.
	
\end{theorem}

Given a predicted instance $\predI$, we can compute a set of optimal weights $\{\predw_v\}$ for it.

To prove Theorem~\ref{thm:dist_robustness_restated}, we first show that if use these weights $\{\predw_v\}$ directly on instance $I$, our competitive ratio can be bounded by $2\gamma$. Then claim that we can always design a new algorithm that is never worse than the performance of the previous policy and a $(1/(d+1))$ factor of the optimal.  

\begin{theorem}\label{thm:type_model}
	For any constant $\eps> 0$, if a set of weights $\predw$ returns a $(1-\eps)$-approximated solution in instance $\predI$, routing flow proportionally according to $\{\predw_v\}$ in instance $I$ returns a feasible flow whose value is at least
	\[(1-\epsilon)\opt-2\gamma,\] where $\opt$ is the the maximum flow value of instance $I$.
\end{theorem}

\begin{proof}
	%Use $\predI$ and $I^*$ to represent the impression set with distribution $\{ \predprob_S \}$ and $\{ p^*_S \}$ respectively.
	%The weights $\{ \predw_v \}$ are the optimal weights of the instance $(\predI,G)$ computed for the given precision $\epsilon$.
	Add a source $s$ into the instance $(\predI,G)$, which is adjacent to each $i\in \predI$.
	The new graph is denoted by $\predG$.
	According to the reduction in Section~\ref{sec:existweight}, w.l.o.g., we can assume $\predG$ is an $(s$-$t)$ $d$-layered graph.
	Use $\val(\predw,\predG)$ to represent the value of flow if routed according to $\{\predw_v\}$ on graph $\predG$.
	Similarly, we can define $G$ and $\val(\predw,G)$.
	Note that for any impression set, we can shrink the impression with the same type into one vertex. Thus, $\predG$ and $G$ can be viewed as a same graph with different induced capacity functions on the impression layer $I$.

	Our goal is to prove that
	\begin{equation}\label{eq:type_model}
	\val(\predw,G) \geq (1-\epsilon)\opt-2\gamma.
	\end{equation}

	%{\color{red} TL: For the following, we should refer to a specific theorem or lemma.}
	Theorem~\ref{thm:weight_DAG} implies that in graph $\predG$ there exists an $s$-$t$ vertex cut $\C$ such that
	\begin{equation}\label{eq:type_eq_main}
	\val(\predw,\predG) \geq (1-\epsilon) C(\C,\predG),
	\end{equation}
	where $C(\C,\predG)$ is the value of cut $\C$ based on the capacity function of $\predG$.
	
	The basic framework of proving Theorem~\ref{thm:type_model} is first showing that the performances of $\{\predw\}$ are close in $\predG$ and $G$:
	\begin{equation}\label{eq:type_eq_1}
	\val(\predw,G) \geq \val(\predw,\predG) - \gamma,
	\end{equation}
	and then proving the values of cut $\C$ in $\predG$ and $G$ are also close:
	\begin{equation}\label{eq:type_eq_2}
	C(\C,\predG)\geq C(\C,G) - \gamma.
	\end{equation}
	
	Since $ C(\C,G)$ is the upper bound of $\opt$, Eq~\eqref{eq:type_model} can be proved directly by combining Eq~\eqref{eq:type_eq_main}, Eq~\eqref{eq:type_eq_1} and Eq~\eqref{eq:type_eq_2}.
	
	Create a new graph $G'$ based on $G$ and $\predG$, where the only difference is the capacity function of impression layer $I$. For each impression type $i$, its capacity in $G'$ is the minimum value of its capacity in $G$ and $\predG$, i.e. define
	$ C(i, G' ) := \min (C(i,G),C(i,\predG)  )$.
	Thus we have
	\[ |C(i,G')-C(i,\predG) | \leq |m_i-\predm_i|, \]
	indicating that with the same weights, if the capacity of impression type $i$ increases from $C(i,G')$ to $C(i,\predG)$, the objective value increases by at most $|m_i-\predm_i|$.  So we have
	\[ |\val(\predw,G')-\val(\predw,\predG)| \leq m\sum_i |m_i-\predm_i| = \gamma  \]
	Since the capacity of each vertex in $G$ is no less than that in $G'$, we have
	$ \val(\predw,G) \geq \val(\predw,G')$.
	Chaining inequalities, we have
	$\val(\predw,G) \geq \val(\predw,\predG) - \gamma$.
	
	Eq~\eqref{eq:type_eq_2} can also be proved in the same way by analyzing the capacity of the cut in $G'$ and comparing this to the capacities in the graphs $\predG$ and $G$, respectively.  Doing so yields the following chain of inequalities:
	\[ C(\C,\predG) \geq C(\C,G') \geq C(\C,G^*) - \gamma.\]
	This now completes the proof as argued above.
\end{proof} 

The proof of the other $1/(d+1)$ bound is given in Appendix~\ref{app:robustness}. More precisely, 

Now we give the proof of the other $1/(d+1)$ bound. We claim the following theorem:

\begin{theorem}\label{thm:maximal_bound}
	Given any set of weights, for the proportional algorithm $\cA$, there exists an improved algorithm $\cA'$ such that the competitive ratio of $\cA'$ is at least $1/(d+1)$ and always better than the ratio of $\cA$. 
\end{theorem}

\begin{algorithm}[t]
	\caption{Online algorithm with weight predictions}
	\label{alg:online_d_layer}
	\KwIn{$G= (I \cup V,E)$ where $I$ arrives online, $\{ C_a \}_{a \in A}$, parameter $\epsilon \in (0,1)$, prediction $\{\predw\}$}

	\While{an impression $i$ comes }
	{	
		%Initially, set $x_{i,a} = x'_{i,a}=0$ for all $a\in N_i$.
		
		Use $\gamma_i$ to represent the proportion of the feasible flow sent from $i$ to $t$. Initially, $\gamma_i=0$.
		
		\While{ $\gamma_i < 1$ and $i$ can reach $t$ in $G$ }
		{
			Send $(1-\gamma_i)$ flow from $i$ to $t$ according to the predicted weights.
			
			Update $\gamma_i$.
			
			Remove all blocked vertices in $G$.
			
		}
	}
\end{algorithm}

The description of algorithm $\cA'$ is given in Algo.~\ref{alg:online_d_layer}. 
Before stating the framework of this proof, we give several definitions.
We say a vertex $v\in V$ is blocked if its capacity is full or all vertices in its neighborhood in the next layer are blocked. Intuitively, if a vertex $v$ is blocked, sending any flow to it cannot increase our objective value. In the previous algorithm, we route each flow according to $\{\predw\}$ directly, indicating we still keep sending arriving flows to some vertices after they are blocked. These flows sent to blocked vertices have no contribution to our objective value. 
%The new algorithm $\algoA'$ deals with such flows in a better way.

%We improve the robustness of our algorithm by deal with 

To prove Theorem~\ref{thm:maximal_bound}, we first show that this algorithm always performs better than the previous algorithm, then claim this algorithm always returns a maximal flow and prove that any maximal flow is $\frac{1}{d+1}$-approximated. The definition of a maximal flow will be stated later.

\begin{lemma}\label{lem:improved_online}
	In any instance, Algo~\ref{alg:online_d_layer} performs better than routing the flow according to the predicted weights directly.
\end{lemma}

\begin{proof}
	According to the proportional allocating rules, we have a following claim easily:
	\begin{claim}
		Consider any two neighboring layers $A_j,A_{j+1}$. Under the proportional allocating setting, if for each $a\in A_j$, $\min (Alloc_a,C_a)$ increases, then for each $a\in A_{j+1}$, $\min (Alloc_a,C_a)$ also increases.
	\end{claim}
	
	Recall that for vertices $a_j$ and $a_{j+1}\in N(a_j,A_{j+1})$, the contribution of $a_j$ to $a_{j+1}$ is given by $\min (Alloc_{a_j} ,C_{a_j} ) x_{a_j,a_{j+1}} $, where $x_{a_j,a_{j+1}}$ is fixed if the weights is fixed. Since $\min (Alloc_{a_j} ,C_{a_j} )$ increases, its contribution to $a_{j+1}$ also increases. Thus, for each $a\in A_{j+1}$, $\min (Alloc_a,C_a)$ also increases.
	
	The objective value of a solution is the sum of $\min(Alloc_a,C_a)$ over $a\in A_d$. 
	Use $Alloc_a'$ and $Alloc_a$ to represent the value of $Alloc_a$ in Algo~\ref{alg:online_d_layer} and the previous simple algorithm respectively.
	By the claim above, if we prove that for each $a$ in the first layer $A_1$, $\min(Alloc_a',C_a) \geq \min (Alloc_a,C_a)$, this lemma can be proved.
	
	Consider each vertex $a$ in the first layer. If this vertex is blocked in Algo~\ref{alg:online_d_layer}, clearly we have 
	\[ \min(Alloc_a',C_a) \geq C_a \geq \min (Alloc_a,C_a). \]
	If this vertex is not blocked, the new algorithm will not send less flow to it than the old algorithm, thus, we also have
	\[ \min(Alloc_a',C_a) \geq Alloc_a' \geq Alloc_a \geq \min (Alloc_a,C_a), \]
	completing this proof.
\end{proof}

For each impression $i$, use $P_i$ to represent the set of paths from $i$ to the sink $t$, and let $P_i(v) \subseteq P_i$ be the set of paths crossing vertex $v$.
We say a solution is a maximal flow if for any impression $i$ that are not assigned totally, no path in $P_i$ has a free capacity. This definition can be viewed as a generalization of the maximal matching.
According to the statement of Algo~\ref{alg:online_d_layer}, it always returns a maximal flow.
Now we show that any maximal flow is $\frac{1}{d+1}$-approximated.

\begin{lemma}\label{lem:maximal_flow}
	For any maximal flow in $G$, it is $\frac{1}{d+1}$-approximated. 
\end{lemma}

\begin{proof}
	Consider a linear program and dual program of this model.
	%For an impression $i$
	
	\begin{equation}\label{LP}
	\begin{aligned}
	&\max & \sum_{i\in I} \sum_{p\in P_i}& x_{i,p} &\\
	&s.t. &\sum_{p\in P_i} x_{i,p} &\leq 1 \quad &\forall i \in I \\
	&&\sum_{i\in I}\sum_{p\in P_i(v)} x_{i,p} &\leq C_v  & \forall v\in V \\
	&&x_{i,p} &\geq 0 &\forall i\in I,p\in P_i
	\end{aligned}
	\end{equation}
	For each pair $(i,p)$, use variable $x_{i,p}$ to denote the proportion of impression $i$ assigned to path $p$. Using dual variable $\gamma_i$ to represent the constraint that impression $i$ has only one unit and dual variable $y_v$ to represent the capacity constraint of vertex $v$, we can have the following dual program:
	\begin{equation}\label{DP}
	\begin{aligned}
	&\min  &\sum_{v\in V}y_v C_v & +  \sum_{i\in I} \gamma_i &\\
	&s.t. &\sum_{v\in p} y_v + \gamma_i&\geq 1 \quad &\forall i\in I, p\in P_i  \\
	& &y_v,\gamma_i & \geq 0 &\forall i \in I, v\in V
	\end{aligned}
	\end{equation}
	
	Given a maximal solution $\{x_{i,p}\}$ of LP~\eqref{LP}, we now construct a feasible solution of its dual program. Let $\gamma_i$ be the assigned proportion of impression $i$. Namely, \[\gamma_i =\sum_{p\in P_i} x_{i,p}.\] Clearly, we have
	
	\begin{equation}\label{eq:maximal_1}
	\sum_{i\in I} \sum_{p\in P_i} x_{i,p} = \sum_{i\in I} \gamma_i. 
	\end{equation}

	Let $y_v$ be the proportion of vertex $v$'s capacity occupied by this solution. Namely,
	\[y_v = (\sum_{i\in I}\sum_{p\in P_i(v)} x_{i,p}) /C_v .\]
	
	Note that the diameter of $G$ is $d$. Any path $p$ crosses at most $d$ vertices in $V$. Thus, if summing over all $y_vC_v$, each $x_{i,p}$ is counted at most $d$ times. Thus, we have
	
	\begin{equation}\label{eq:maximal_2}
	d\sum_{p\in P_i} x_{i,p} \geq \sum_{v\in V}y_v C_v. 
	\end{equation}
	
	Combining Eq~\eqref{eq:maximal_1} and Eq~\eqref{eq:maximal_2}, we have
	\begin{equation}
	(d+1)\sum_{p\in P_i} x_{i,p} \geq \sum_{v\in V}y_v C_v + \sum_{i\in I} \gamma_i.
	\end{equation}
	
	If $\{y_v,\gamma_i  \}$ is feasible, then the maximal flow is $\frac{1}{d+1}$-approximated. For any pair $(i,p)$, if $\gamma_i=1$, its constraint is satisfied. Otherwise, according to our algorithm, path $p$ has no free capacity, indicating that there exists a vertex $v$ in this path with $y_v=1$. So we still have $ \sum_{v\in p} y_v + \gamma_i\geq 1.$ Thus, $\{y_v,\gamma_i  \}$ is feasible and this maximal flow is $\frac{1}{d+1}$-approximated.

\end{proof}

Combing Lemma~\ref{lem:improved_online} and Lemma~\ref{lem:maximal_flow}, Theorem~\ref{thm:maximal_bound} can be proved, also completing the proof of Theorem~\ref{thm:dist_robustness_restated}.

%% file: apdxexistweight.tex
\section{Existence of Useful Weights for Max Flow in general DAGs}\label{sec:apdxexistweight}

%\subsection{Proof of Lemma~\ref{lem:reduction}}\label{sec:adpxreduction}
In this section, we first present our simplification of the algorithm in~\citep{PropMatchAgrawal} for the bipartite (or 2-layer DAG) case using only weight decreases rather than increases and decreases in their algorithm. In the following subsection~\ref{sec:apdxexistweights_d}, we extend our proof of the three-layer graph case in Appendix~\ref{subsec:exist3weight} to the general case of $d$ layers and use this to complete the proof of Theorem~\ref{thm:weight_DAG} in Appendix~\ref{subsec:existgenweight}.

\subsection{Maximum Cardinality Bipartite Matching}\label{sec:apdxexistweights_bi}

Let $G = (I,A,E)$ be a bipartite graph.  $I$ is the set of ``impressions'' and $A$ is the set of ``advertisers''.  Each advertiser $a$ has a capacity $C_a$.  We want to find a maximum cardinality matching where each impression can be matched at most once and each advertiser can be matched at most $C_a$ times.  We focus on fractional solutions. Let $N_i$ and $N_a$ represent the neighborhoods of $i$ and $a$, respectively.  We want to find $(1-\epsilon)$-approximate fractional solutions to the following LP.

\begin{equation} \label{eqn:matching_lp_1}
\begin{array}{ccc}
\text{maximize} & \displaystyle \sum_{a \in A}  y_a &  \\
& \displaystyle \sum_{a \in N_i} x_{ia} = 1 & \forall i \in I \\
& \displaystyle y_a \leq \sum_{i \in N_a} x_{ia} & \forall a \in A \\
& \displaystyle y_a \leq C_a & \forall a \in A \\
& x_{ia} \geq 0 & \forall (i,a) \in E
\end{array}
\end{equation}

Let $\{\alpha_a\}_{a \in A}$ be a set of weights for the advertisers.  We consider fractional solutions parameterized by these weights in the following way, for all $(i,a) \in E$ we set $x_{ia}$ as follows:

\begin{equation} \label{eqn:prop_alloc}
x_{ia} := \frac{\alpha_a}{\sum_{a' \in N_i} \alpha_{a'}}
\end{equation}

Initially, each $\alpha_a$ is set to 1, and then updated over time according to Algorithm~\ref{alg:bi_prop_alg}.  Intuitively, we compute the resulting allocation given by the weights, then decrease the weight for each advertiser that has been significantly over-allocated by $1+\epsilon$.  This repeats for some number of rounds $T$.  At the end we scale down the allocation so it is always feasible.  We want to show the following guarantee for this algorithm.

\begin{algorithm}[t]
	\caption{Proportional Allocation Algorithm}
	\label{alg:bi_prop_alg}
	%\begin{algorithmic}
	\KwIn{$G = (I,A,E),C,T,\epsilon$}
	
	$\forall a \in A, \ \alpha_a \gets 1$ \
	
	\For{$t = 1,2,\ldots,T$}
	{
		Let $x_{ia} := \frac{\alpha_a}{\sum_{a' \in N_i} \alpha_{a'}}$ \
		
		$\forall a \in A,$ compute $Alloc_a  := \sum_{i \in N_a} x_{ia}$ \
		
		$\forall a \in A,$ if $Alloc_a \geq (1+\epsilon)C_a$ then $\alpha_a \gets \alpha_a /(1+\epsilon)$ \
	}
	%\EndFor
	
	%\EndFor
	\KwOut{ Weights $\{ \alpha_a \}_{a\in A}$, Allocation $\{x_{ia}\}_{(i,a) \in E}$}
\end{algorithm}

\begin{theorem} \label{thm:bi_prop_thm}
	Let $n = |A|$ be the number of advertisers.  For any $\epsilon > 0$, if Algorithm~\ref{alg:bi_prop_alg} is run for $T = O(\frac{1}{\epsilon^2}\log(\frac{n}{\epsilon}))$ iterations, then the allocation it returns is a $(1- O(\epsilon))$-approximation to LP~\eqref{eqn:matching_lp_1}.
\end{theorem}

We start by defining some notation.
We use superscript $(t)$ to denote the value of variables in the end of iteration $t$. For example, $x_{ia}^{(t)}$ represents the value of $x_{ia}$ in the end of iteration $t$.
Use $Alloc_a := \sum_{i \in N_a} x_{ia}$ to denote the amount allocated to advertiser $a$ using the weights.
Define $\alpha_{min} = \frac{1}{(1+\epsilon)^T}$.
Since any weight can decrease at most $(1+\epsilon)$ per iteration, all weights are at least $\alpha_{min}$.
Let $A(k) = \{ a \in A \mid \alpha_a= (1+\epsilon)^k \alpha_{min}\}$ for $0 \leq k \leq T$ be the advertisers with weight at ``level'' $k$.

The maximum cardinality bipartite matching can be seen as a special maximum flow problem. Consider the natural flow network where there is a sink $s$ and source $t$.  $s$ is connected to each $i \in I$ with capacity 1 and each $a \in A$ is connected to $t$ with capacity $C_a$.  $I$ is connected to $A$ according to $G$ with infinite capacity on each edge (or capacity 1 if you don't like infinite capacity).  Let this network be $G'$.  Note that any feasible flow in $G'$ corresponds to a feasible solution to our problem and vice-versa. %Thus, any $s$-$t$ cut in $G'$ is an upper bound of our problem.

To prove Theorem~\ref{thm:bi_prop_thm}, we consider a subgraph $G''$, a smaller graph by removing some vertices and edges of $G'$.
We first construct an $s$-$t$ cut in $G''$. Then prove that the value of our solution is at least $(1-O(\epsilon))$ times the value of this cut, thus at least $(1-O(\epsilon))$ times the maximum flow in $G''$. Finally, we show that
the optimal value in $G''$ is close to that in $G'$, completing the proof.

\begin{figure}[t]
		\centering
	%\subfigure[]{
	    \captionsetup[subfigure]{justification=centering}
		\begin{subfigure}[t]{0.45\linewidth}
			\centering
			\includegraphics[height=7.2cm]{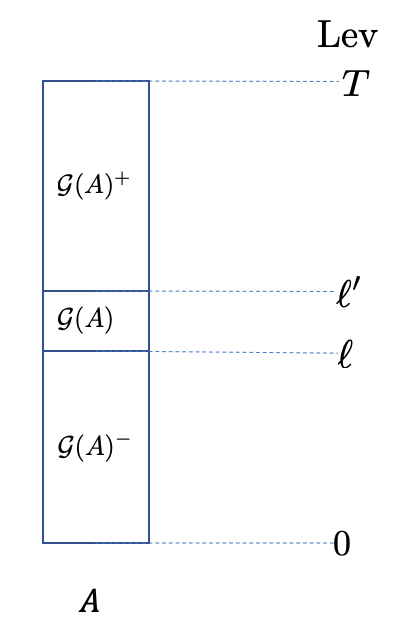}
			\caption{}
			\label{fig:bigap}
		\end{subfigure}
	%}%
	%\subfigure[]{
		\begin{subfigure}[t]{0.5\linewidth}
			\centering
			\includegraphics[height=6.7cm]{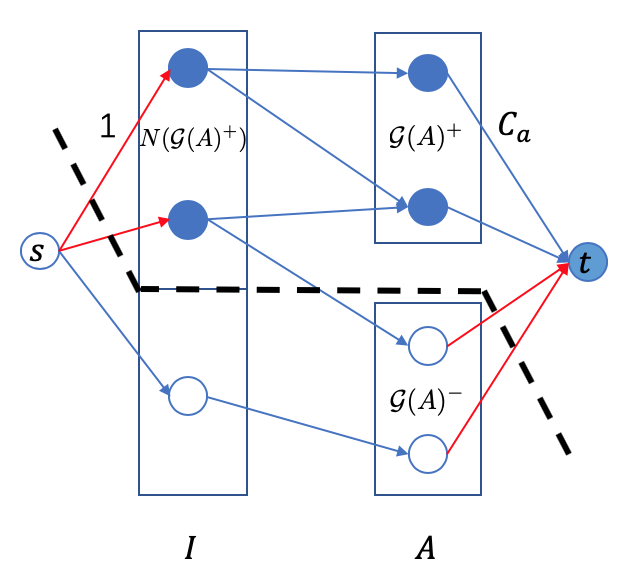}
			\caption{}
			\label{fig:bicut}
		\end{subfigure}
	%}%
		\caption{Fig~(a) is an illustration of the partition of the $A$ layer. Removing $\g(A)$ can get the graph $G''$. Fig~(b) is an illustration of the $s$-$t$ cut in the graph. We color the vertices staying with $t$ blue and the edges in the cut red.}
		\label{fig:bi_gap_and_cut}
\end{figure}

We now define some vertex sets in order to construct $G''$. Given any integer $1 \leq \ell \leq T-1-\frac{\log(n/\epsilon)}{\epsilon}$, we define a gap in the $A$ layer: $\g(A):= \bigcup_{k=\ell}^{\ell'} A(k)$, where $\ell'=\ell+\frac{\log(n/\epsilon)}{\epsilon }$. As shown in Fig~\ref{fig:bigap}, all vertices in the $A$ layer are partitioned into three parts: $\g(A)$, $\g(A)^-$ and $\g(A)^+$, where $\g(A)^-:=\bigcup_{k=0}^{\ell-1} A(k)$ and $\g(A)^+:=\bigcup_{k=\ell'+1}^{T} A(k)$.
%Similarly, we can define $\g(\B)$, $\g(\B)^-$ and $\g(\B)^+$.

The subgraph $G''$ is created by removing $\g(A)$.
%Our strategy is to compare the value of our allocation to the capacity of some cut in $G''$.
Now, we show how to construct an eligible $s$-$t$ cut.
As shown in Fig~\ref{fig:bicut}, we group the advertisers with higher weights $\g(A)^+$ and their neighboring impressions $N(\g(A)^+)$ with $t$.
The remaining vertices are grouped with $s$.
This results in cutting the $(a,t)$ edges for the lower weight advertisers and the $(s,i)$ edges for the neighbors of the higher weight advertisers.
%We want to prove that the value of our allocation is nearly the capacity of this cut.

We define some more notation.  Let $N(S) := \bigcup_{a \in S} N_a$ and $C(S) = \sum_{a \in S} C_a$ be the collective neighborhood and capacity of $S \subseteq A$, respectively.  Let $\opt(G)$ be the maximum flow in graph $G$.
%the optimal value of LP~\eqref{eqn:matching_lp}.
%Using $\opt(G'')$ to denote the maximum flow in $G''$,
According to the construction of the cut above, we have
\begin{equation}\label{eq:bi_upper_bound}
\opt(G'') \leq | N(\g(A)^+) | + C( \g(A)^- ).
\end{equation}

Our first lemma shows that ``low'' weight advertisers have at least their capacity allocated to them and ``high'' weight advertisers are not too over-allocated.

\begin{lemma} \label{lem:bi_lemma_one}
	After iteration $T$, for each $a \in \bigcup_{k=0}^{T-1}A(k)$, we have $Alloc_a \geq C_a$.  Similarly, for $a \in \bigcup_{k=1}^{T}A(k)$ we have $Alloc_a \leq (1+\epsilon)^2C_a$
\end{lemma}
\begin{proof}
	For the first part, let $a \in \bigcup_{k=0}^{T-1}A(k)$ and note that such an advertiser had its weight decreased at least once.  Consider the last iteration $t$ where this occurred. In this iteration, prior to decreasing $\alpha_a$, we had \[ Alloc^{(t-1)}_a \geq (1+\epsilon)C_a\] by definition of the algorithm.  After decreasing $\alpha_a$, $Alloc_a$ decreased at most $(1+\epsilon)$. So we had
	\[ Alloc_a^{(t)} \geq \frac{Alloc_a^{(t-1)}}{1+\epsilon} \geq  C_a. \]
	$Alloc_a$ did not decrease in subsequent iterations since this was the last round $\alpha_a$ decreased.
	Thus, after the final iteration, $Alloc_a$ is still at least $C_a$.
	%and other advertisers can only decrease their weights.
	
	The second part is similar. Let $a \in \bigcup_{k=1}^{T}A(k)$ and that such an advertiser had at least one iteration where it didn't decrease its weight.  Consider the last such iteration $t$.
	Clearly, we had
	\[ Alloc^{(t-1)}_a \leq (1+\epsilon)C_a. \]
	in the beginning of this iteration.
	In this iteration, $Alloc_a$ increased at most $(1+\epsilon)$. So we have
	\[ Alloc^{(t)}_a \leq (1+\epsilon)Alloc_a^{(t-1)} \leq (1+\epsilon)^2C_a. \]
	Since advertiser $a$ decreased its weight in all subsequent iterations, this inequality is maintained, completing the proof.
\end{proof}

%Next we define the cut in $G'$ we want to use to bound the value of our allocation.
%We define some more notation.  Let $N(S) := \bigcup_{a \in S} N_a$ and $C(S) = \sum_{a \in S} C_a$ be the collective neighborhood and capacity of $S \subseteq A$, respectively.  Let $\opt$ be the optimal value of LP~\eqref{eqn:matching_lp}.

%The algorithm will not gain all the value of the cut defined above, but almost all of its value.  To show this we introduce a ``gap'' in the two groups of advertisers $A_1, A_2$ defined above.  We will show that impressions $i$ with edges in both groups have a $1- \epsilon$ fraction of their allocation in only one of these groups.  We then show that the amount lost in the gap is a small fraction of our total value.  Let $\ell,\ell'$ be two indices with $\ell = \ell' + \log_{1+\epsilon}(n / \epsilon)$ and $\ell' \geq 1$.  Define $A_1' = \bigcup_{k=1}^{\ell'} L_k$ and $A_2 = \bigcup_{k=\ell+1}^{0}L_k$ is the same as before.  Let $\val(x) := \sum_{ia \in E} x_{ia}$ be the value of allocation $x$.

Let $\val$ be the value of our solution.
Now we prove that $\val$ is at least $(1-O(\epsilon))$ times the value of that cut.

\begin{lemma} \label{lem:bi_gap_argument}
	For any $\ell$, we have $\val \geq C(\g(A)^-) + (1-O(\epsilon) )|N(\g(A)^+)|$.
\end{lemma}
\begin{proof}
	By Lemma~\ref{lem:bi_lemma_one}, after iteration $T$, we have that each $a \in \g(A)^-$ has \[Alloc_a \geq C_a.\]
	Thus, we get exactly $C( \g(A)^- )$ for these advertisers.
	Next note that some impressions $i \in N(\g(A)^+)$ may have an edge to both $\g(A)^-$ and $\g(A)^+$, we show that we can neglect such edges from $I$ to $\g(A)^-$ to avoid double counting.  Suppose that $i$ has edges $(i,a_1)$ and $(i,a_2)$ where $a_1 \in \g(A)^-$ and $a_2 \in \g(A)^+$.  By definition of the gap, we have
	\[\alpha_{a_2} \geq (1+\epsilon)^{\log(n / \epsilon) /\epsilon} \alpha_{a_1} \geq \frac{n}{\epsilon} \alpha_{a_1}.\]
	It follows that for any impression $i$, \[x_{ia_1} \leq \frac{\epsilon}{n} x_{ia_2}\]
	and that
	\[  \sum_{a \in N_i \cap \g(A)^-} x_{ia} \leq \epsilon. \]
	Counting the value we allocate to $\g(A)^+$, we have
	\[ \sum_{a\in \g(A)^+} Alloc_a = \sum_{i \in N(\g(A)^+)} \sum_{a \in N_i \cap \g(A)^+} x_{ia} \geq (1-\epsilon) |N( \g(A)^+ )|. \]
	By Lemma~\ref{lem:bi_lemma_one}, the advertisers in $\g(A)^+$ are not too over-allocated.
	Thus finally, we have
	\[ \val \geq \sum_{a \in \g(A)^- } \min(Alloc_a,C_a) + \sum_{a \in \g(A)^+ } \min(Alloc_a,C_a) \geq C(\g(A)^-) + (1-O(\epsilon))|N( \g(A)^+ )|. \]
\end{proof}

By Lemma~\ref{lem:bi_gap_argument} and Eq~\eqref{eq:bi_upper_bound}, we can bounded our solution by the maximum flow in $G''$:
\begin{equation}\label{eq:bi_bound}
\val \geq C(\g(A)^-) + (1-O(\epsilon) )|N(\g(A)^+)| \geq (1-O(\epsilon)) \opt(G'').
\end{equation}

Now we show that $\opt(G'')$ is close to $\opt(G')$ by selecting an appropriate $\ell$.
More specifically, since $G''$ is obtained by removing $\g(A)$ from $G'$, we have
\[ \opt(G')-\opt(G'') \leq C(\g(A)). \]
%We are not yet done as there is a gap between $c(A_1)$ and $c(A_1')$, which may be large depending on the choice of $\ell,\ell'$.

Next we show by averaging that for all $\ell$, the difference is small when $T$ is large enough, indicating that there must exist one $\ell$ such that $C(\g(A))$ is small.

\begin{lemma} \label{lem:bi_averaging_arg}
	%Let $A_3 = \bigcup_{k=-R+1}^0 L_k$.
	If $T \geq  \frac{2}{ \epsilon^2 }\log(n/\epsilon)$, then there exist an $\ell$ such that
	\[ C(\g(A)) \leq \epsilon \val. \]
	%the associated $A_1,A_1'$ have $C(A_1) - C(A_1') \leq \epsilon C(A_3)$.
\end{lemma}
\begin{proof}
	
	Summing the capacity in the gap over all values of $\ell$, we have
	\[
	\sum_{\ell=1}^{T-1-\frac{1}{\epsilon}\log(n/\epsilon)} \sum_{k=\ell}^{\ell'} \sum_{a \in A(k)} C_a \leq \frac{\log(n / \epsilon) }{\epsilon}\sum_{k=1}^{T-1} \sum_{a \in A(k)} C_a.
	\]
	
	Due to Lemma~\ref{lem:bi_lemma_one}, we have
	\[ \sum_{k=1}^{T-1} \sum_{a \in A(k)} C_a \leq \val \]
	
	Thus, by averaging, there exists an $\ell \in [1,T-1-\frac{1}{\epsilon}\log(n/\epsilon)]$ such that the capacity in the gap is at most
	\[
	\frac{\log(n / \epsilon) / \epsilon }{T-1-\log(n/\epsilon)/\epsilon}\sum_{k=1}^{T-1} \sum_{a \in L_k} C_a \leq \epsilon \val
	\]
	provided $T \geq \frac{2}{ \epsilon^2 }\log(n/\epsilon)$.
\end{proof}

The above lemma will allow us to show that the amount we lose in the gap is a small fraction of our total value.  Showing this and combining the above lemmas will allow us to complete the proof of Theorem~\ref{thm:bi_prop_thm}.

\begin{proof}[Proof of \textnormal{\bf Theorem~\ref{thm:bi_prop_thm}}]
	When $T=O(\frac{1}{ \epsilon^2 }\log(n/\epsilon))$, we have
	\begin{equation}
	\begin{aligned}
	\val &\geq (1-O(\epsilon)) \opt(G'') \\
	\val&\geq (1-O(\epsilon)) (\opt(G') - \epsilon \val) \\
	\val & \geq (1-O(\epsilon)) \opt(G')
	\end{aligned}
	\end{equation}
	Since $\opt(G')$ is equal to the optimal value of LP~\eqref{eqn:matching_lp_1}, our algorithm returns a $(1-O(\epsilon))$-approximated matching.

\end{proof}

\subsection{Max Flow in \texorpdfstring{$d$}{d}-Layered Graphs}\label{sec:apdxexistweights_d}

In this subsection, we generalize our method to the maximum flow problem in the $(s$-$t)$ $d$-layered graph $G(V\cup\{s,t\},E)$.
For simplicity, the graph used in the following is $(d+1)$-layered as shown in Fig~\ref{fig:dlayergraph}. The vertices in $V$ are partitioned into $d+1$ layers, denoted by $I,A_1,A_2,$ ... and $A_d$. Let $A = \bigcup_{j=1}^{d}A_j$.
Each vertex $v$ has capacity $C_v$ and w.l.o.g. we can assume that the capacity of each vertex $i\in I$ is $1$.

Recall the two properties needed: near optimality (Property~\ref{ppt:near_optimal}) and virtual-weight dependence (Property~\ref{ppt:virtual_dependence}). We also first focus on the near optimality property (Theorem~\ref{thm_dlayer}). Then in the end, we give the proof of virtual-weight dependence (Theorem~\ref{thm:d_layer_good_weights}).

\begin{figure}[t]
	\centering
	\captionsetup{justification=centering}
	\includegraphics[width=0.6\linewidth]{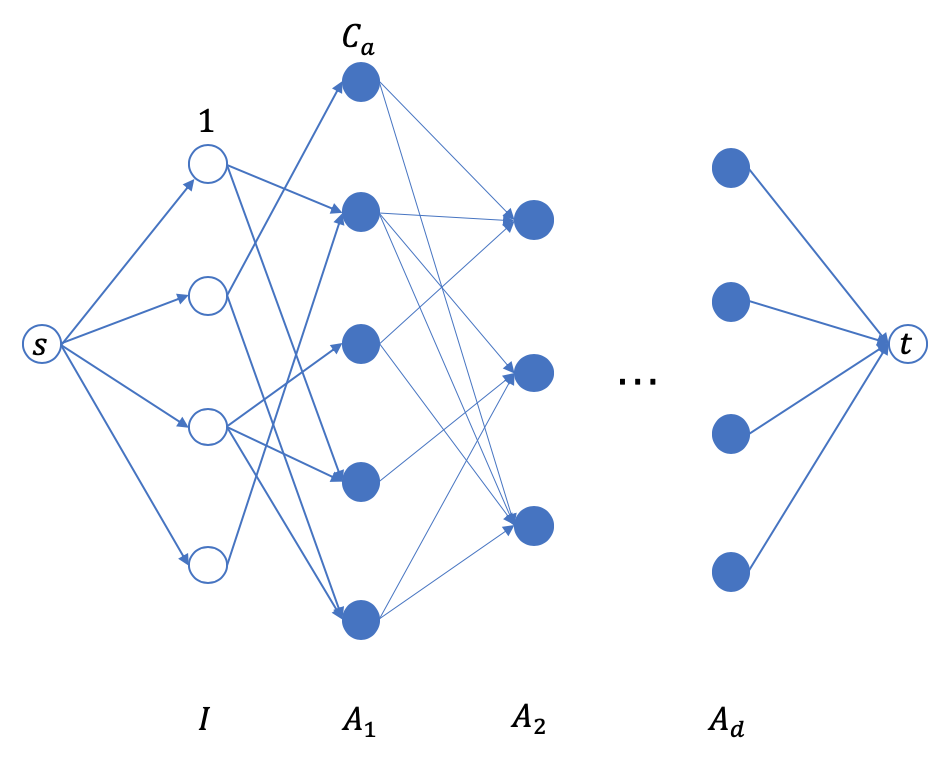}
	\caption{An illustration of the $(s$-$t)$$(d+1)$-layered graph.}
	\label{fig:dlayergraph}
\end{figure}

We define some new notation first.
For each vertex $i\in I$, use $N_{(i,A_1)}$ to represent its neighborhood in $A_1$.
For each vertex $a_j\in A_j$, use $N_{(a_j,A_{j+1})}$ and $N_{(a_j,A_{j-1})}$ to represent its neighborhood in its next layer and previous layer respectively (Note that $I$ is viewed as $A_0$).
For each edge $(a_{j-1},a_j)\in E$, define $x_{(a_{j-1},a_j)}$ be the proportion of flow sent from $a_{j-1}$ to $a_j$.
Define $Alloc_{a_j}$ and $Alloc_{a_j}^{(t)}$ as in the three-layered case.

\begin{algorithm}[t]
	\caption{The framework of the algorithm in $(s$-$t)$$(d+1)$-layered graphs}
	\label{alg:d_layer_framework}
	\KwIn{ $G= (\{s,t\}\cup I\cup A_1\cup ...\cup A_d,E)$, $\{ C_a \}_{a \in \mathbb{A}}$, and parameter $\epsilon_1, ..., \epsilon_d $ }
	
	Initialize $\alpha_a = 1$ $\forall a \in A$. \	

	\For {iteration $1,2,...,T$ }
	{

		For each $(a_{j-1},a_{j})\in E$, let $x_{(a_{j-1},a_{j})} = \frac{\alpha_{a_j}}{\sum_{a_j' \in N_{(a_{j-1}, A_j) }}\alpha_{a_j'}  }$
		
		For each $i\in I$, let $Alloc_i = C_i = 1$. And for each vertex $a_j\in A_j$, let $Alloc_{a_j} = \sum \limits_{a_{j-1} \in N_{(a_j,A_{j-1})}} \min(Alloc_{a_{j-1}}, C_{a_{j-1}}) x_{(a_{j-1},a_j)} $.

		\For{layer $j=d,d-1,...,1$}
		{
			\For{each vertex $a_j\in A_j$}			
			{
				\If{$Alloc_{a_j} > \prod_{j'=1}^{j} (1+\epsilon_{j'}) C_{a_j} $}
				{
					$\alpha_{a_j} \leftarrow \alpha_{a_j} / (1+\epsilon_j)$
				}
				\ElseIf{$j <1$}
				{
					\If{$\cond(N_{(a_j,A_{j+1})}) = \true$}
					{
						$\alpha_{a_j} \leftarrow \alpha_{a_j} / (1+\epsilon_j)$
					}
				}
				
			}
		}

	}
	\KwOut{ $\{\alpha_a\}_{a\in A}$ }
\end{algorithm}

Our generalized framework is stated in Algo~\ref{alg:d_layer_framework}. In each iteration, when we compute $x$ and $Alloc$, we sweep from $A_1$ to $A_d$, but when we update the weights, we reverse sweep from $A_d$ back to $A_1$.
Each vertex $a$ in the last layer only updates its weights due to itself while the vertices in other layers update their weights not only due to themselves, but also due to their neighborhood in the next layer.
We set different values of $\epsilon$ for different layers. For all vertices in $A_j$, we partition them into several layers according to  $\epsilon_j$:
\[ A_j(k) = \{ a_j \in A_j | \alpha_{a_j} = (1+\epsilon_j)^k \cdot \frac{1}{(1+\epsilon_j)^T} \}, \]
and use $\lev(a_j)$ to denote the level of $a_j$.

Now we introduce the generalized version of the four properties:

\begin{property}[Increasing monotonicity]\label{ppt:2_1}
	For any vertex $a\in A$, in iteration $t$, if its weight does not decrease, $Alloc_a^{(t)} \geq Alloc_a^{(t-1)}$.
\end{property}

\begin{property}[Decreasing monotonicity]\label{ppt:2_2}
	For any vertex $a\in A$, in iteration $t$, if its weight decreases, $Alloc_a^{(t)} \leq Alloc_a^{(t-1)}$.
\end{property}

\begin{property}[Layer dominance]\label{ppt:2_3}
	For any vertex $a_j \in A_j$ with $1\leq j \leq d-1$, in any iteration, there exists at least one vertex $a_{j+1} \in N_{(a_j,A_{j+1})}$ such that $\lev(a_{j+1}) \geq \lev(a_j)$.
\end{property}

\begin{property}[Forced decrease exemption]\label{ppt:2_4}
	Let $\epsilon_{max} = \max \limits_j \epsilon_j$, $\epsilon_{min} = \min \limits_j \epsilon_j$ and $n=\max \limits_j |A_j|$. For any vertex $a_j \in A_j $ with $1\leq j \leq d-1$, in any iteration, if there exists one vertex $a_{j+1} \in N_{(a_j,A_{j+1})}$ with $\lev(a_{j+1}) =T$ or satisfying $\lev(a_{j+1})-\lev(a_j) \geq \log(n/\epsilon_{max})/\epsilon_{min}$, then $\cond(N_{(a_j,A_{j+1})}) = \false$.
\end{property}

Similar to the three-layered case, we claim the following theorem:
\begin{theorem}\label{thm_dlayer}
	For any algorithm under our framework with the four properties, if $T=O(\frac{n\log(n/\epsilon_{max})}{\epsilon_{max}\epsilon_{min}} )$, it will return a $\prod_{j=1}^{d}(1-O(\epsilon_j d))$-approximated solution.
\end{theorem}

Thus, given an appropriate $\{\epsilon\}$, our algorithm will return a near-optimal solution.
The basic idea of proving Theorem~\ref{thm_dlayer} is similar to that of Theorem~\ref{thm_g3}. We first define a gap in each layer and remove them to construct a new graph $G'$. We prove that when $T$ is large enough, the optimal value of the flow in $G'$ is close to that of $G$.
Then we construct an $s$-$t$ vertex cut in $G$ and show that the value of our flow in $G'$ is close to the value of this cut, completing the proof.

Given any integer $1+\frac{\log(n/\epsilon_{max})}{\epsilon_{min} }\leq l \leq T-1-\frac{\log(n/\epsilon_{max})}{\epsilon_{min} }$, for each layer $A_j$, we can define a gap: $\g(A_j):= \bigcup_{k=l}^{l'} A_j(k)$, where $l'=l+\frac{\log(n/\epsilon_{max})}{\epsilon_{min} }$. As shown in Fig~\ref{fig:dlayergap}, all vertices in $A_j$ are partitioned into three parts: $\g(A_j)$, $\g(A_j)^-$ and $\g(A_j)^+$, where $\g(A_j)^-:=\bigcup_{k=0}^{l-1} A_j(k)$ and $\g(A_j)^+:=\bigcup_{k=l'+1}^{T} A_j(k)$.

\begin{figure}[t]
	\centering
	\captionsetup{justification=centering}
	\includegraphics[width=0.6\linewidth]{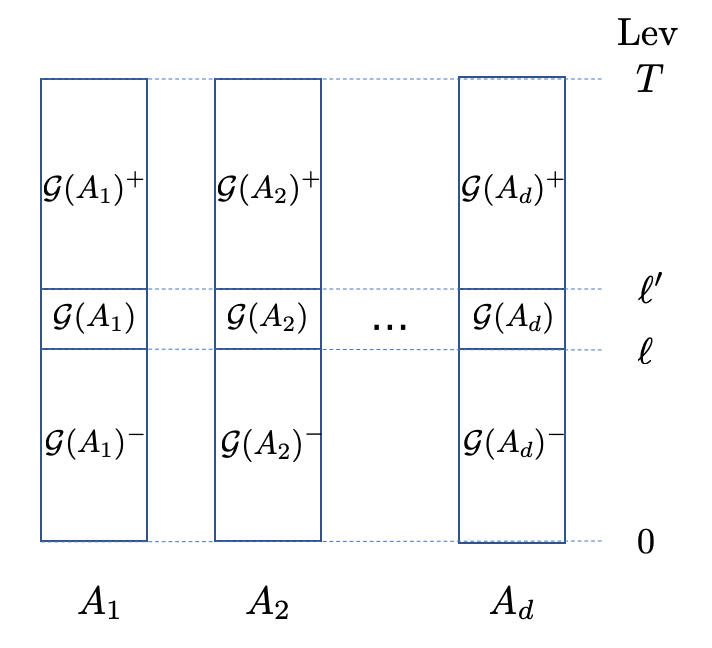}
	\caption{An illustration of the partition of $A_1$,...,$A_d$.}
	\label{fig:dlayergap}
\end{figure}

Define $\g := \bigcup_{j=1}^{d} \g(A_j)$, $\g^+:=\bigcup_{j=1}^d \g(A_j)^+$ and $\g^-:=\bigcup_{j=1}^d \g(A_j)^-$. The graph $G'$ is created by removing all vertices in $\g$ and all edges adjacent to them. The two rules to construct an $s$-$t$ vertex cut $\C$ are the same:

(1) For any $s$-$t$ path always crossing $\g^-$, add the vertex $a_d \in A_d$ to $\C$.

(2) For any $s$-$t$ path crossing at least one vertex in $\g^+$, find the first vertex $a_j$ in $\g^+$ and add the vertex $a_{j-1}$ before it to $\C$. (if $j=1$, add vertex $i$ to $\C$.)

Observe that $\C$ is a feasible $s$-$t$ vertex cut, meaning that all $s$-$t$ paths in $G'$ will be blocked by removing $\C$.
Use $\g(A_j^-,A_{j+1}^+)$ to represent the vertices in $\g(A_j)^-$ which are adjacent to at least one vertex in $\g(A_{j+1})^+$. Then according to the two rules, we can compute the value of this cut:

\begin{equation}
C(\C) =  C(\g(A_d)^-) + C(\g(A_{d-1}^-,A_d^+)) +... + C(\g(A_1^-,A_2^+)) +  |N_{(\g(A_1)^+,I)}|  .
\end{equation}

Also, we can compute the value of our solution:
\begin{equation}\label{eq:val}
\val = \sum_{a_d\in \g(A_d)^-} \min(Alloc_{a_d},C_{a_d}) + \sum_{a_d\in \g(A_d)} \min(Alloc_{a_d},C_{a_d})+ \sum_{a_d\in \g(A_d)^+} \min(Alloc_{a_d},C_{a_d}).
\end{equation}

Using the same technique as in the proof of Lemma~\ref{lem:a_lowerbound} and Lemma~\ref{lem:a_upperbound}, we have the following two lemmas:
\begin{lemma}\label{lem:lower_bound}
	If the increasing monotonicity holds, in any layer $A_j$, $\forall a_j \in \bigcup_{k=0}^{T-1}A_j(k)$, if the last iteration that its weight decreased is due to a self-decrease, we have $Alloc_{a_j} \geq C_{a_j}$.
\end{lemma}

\begin{lemma}\label{lem:upper_bound}
	If the decreasing monotonicity property holds, in any layer $A_j$, $\forall a_j \in \bigcup_{k=1}^{T}A_j(k)$, we have $Alloc_{a_j} \leq C_{a_j} \cdot \prod_{j'=1}^{j}(1+3\epsilon_{j'}) $.
\end{lemma}

Thus, we have
\begin{equation}\label{eq:rel_first}
\sum_{a_d\in \g(A_d)^-} \min(Alloc_{a_d},C_{a_d}) \geq C(\g(A_d)^-).
\end{equation}

Now, we need to establish the relationship between

\begin{equation}\label{eq:val_remain}
\sum_{a_d\in \g(A_d)} \min(Alloc_{a_d},C_{a_d})+ \sum_{a_d\in \g(A_d)^+} \min(Alloc_{a_d},C_{a_d})
\end{equation}

and

\begin{equation}\label{eq:cut_remain}
C(\g(A_{d-1}^-,A_d^+)) +... + C(\g(A_1^-,A_2^+)) +  |N_{(\g(A_1)^+,I)}|.
\end{equation}

To obtain the relationship, similar to Eq~\eqref{eq:1} and Eq~\eqref{eq:alloc}, we claim the following lemma:
\begin{lemma}\label{lem:alloc}
	If the layer dominance property holds,
	for any layer $1<j\leq d$, we have
	\begin{equation}
	\begin{aligned}\label{eq:sum_1}
	\sum_{a_j\in \g(A_j)^+} Alloc_{a_j} + \sum_{a_j\in \g(A_j)} Alloc_{a_j}\geq& \sum_{a_{j-1}\in \g(A_{j-1})^+ } \min(Alloc_{a_{j-1}},C_{a_{j-1}}) (1-\epsilon_{max}) \\
	&+ \sum_{a_{j-1}\in \g(A_{j-1}^-,A_j^+)} \min(Alloc_{a_{j-1}},C_{a_{j-1}}) (1-\epsilon_{max}),
	\end{aligned}
	\end{equation}
	and for the first layer $A_1$, we have
	\begin{equation}\label{eq:sum_2}
	\sum_{a_1\in \g(A_1)^+} Alloc_{a_1} + \sum_{a_1\in \g(A_1)} Alloc_{a_1} \geq (1-\epsilon_{max}) |N_{(\g(A_1)^+,I)}|.
	\end{equation}
\end{lemma}
\begin{proof}
	Since the length of the gap is $\log(n/\epsilon_{max})/\epsilon_{min}$, for any vertex $a_{j-1} \in A_{j-1}$, if it is adjacent to at least one vertex in $\g(A_{j})^+$, the proportion of flow that it sends to $\g(A_{j})^-$ is at most $\epsilon_{max}$. Thus, the second statement can be proved directly.
	
	Due to Property~\ref{ppt:2_3} and the definition of $\g( A_{j-1}^-,A_j^+ )$, any vertex in $\g(A_{j-1})^+$ and $\g( A_{j-1}^-,A_j^+ )$ is adjacent to at least one vertex in $\g(A_j)^+$, proving the first statement.
\end{proof}

Due to Lemma~\ref{lem:upper_bound}, for any vertex $a_j \in \g(A_j)^+ \cup \g(A_j)$, we have
\begin{equation}\label{eq:sum_3}
\min(Alloc_{a_j},C_{a_j}) \geq Alloc_{a_j} \cdot \prod_{j'=1}^j (1-3\epsilon_{j'}).
\end{equation}

Due to Lemma~\ref{lem:lower_bound}, for any vertex $a_j \in \g(A_j)^-$, we have
\begin{equation}\label{eq:sum_4}
\min(Alloc_{a_j},C_{a_j}) \geq C_{a_j}.
\end{equation}

For simplicity, let $1-\delta_j=  (1-\epsilon_{max}) \cdot\prod_{j'=1}^j (1-3\epsilon_{j'}) $.  Combing Eq~\eqref{eq:sum_1}, Eq~\eqref{eq:sum_2}, Eq~\eqref{eq:sum_3} and Eq~\eqref{eq:sum_4}, we can obtain that
\begin{equation}
\begin{aligned}
&\sum_{a_j\in \g(A_j)^+} \min(Alloc_{a_j},C_{a_j}) + \sum_{a_j\in \g(A_j)} \min(Alloc_{a_j},C_{a_j})\\
\geq& (\sum_{a_j\in \g(A_j)^+} Alloc_{a_j} + \sum_{a_j\in \g(A_j)} Alloc_{a_j}) \cdot \prod_{j'=1}^j (1-3\epsilon_{j'})  \\
\geq & (\sum_{a_{j-1}\in \g(A_{j-1})^+ } \min(Alloc_{a_{j-1}},C_{a_{j-1}}) (1-\epsilon_{max}) )\cdot \prod_{j'=1}^j (1-3\epsilon_{j'}) \\
&+ (\sum_{a_{j-1}\in \g(A_{j-1}^-,A_j^+)} \min(Alloc_{a_{j-1}},C_{a_{j-1}}) (1-\epsilon_{max}))  \cdot\prod_{j'=1}^j (1-3\epsilon_{j'}) \\
\geq &\sum_{a_{j-1}\in \g(A_{j-1})^+ } \min(Alloc_{a_{j-1}},C_{a_{j-1}}) (1-\delta_j)
+ C(\g(A_{j-1}^-,A_j^+))(1-\delta_j) \\
\end{aligned}
\end{equation}
and
\begin{equation}
\begin{aligned}
\sum_{a_1\in \g(A_1)^+} \min(Alloc_{a_1},C_{a_1}) &+ \sum_{a_1\in \g(A_1)} \min(Alloc_{a_1},C_{a_1}) \geq (1-\delta_1) |N_{(\g(A_1)^+,I)}|.
\end{aligned}	
\end{equation}

Summing the above two inequalities over $j=1,2,...,d$, we have
\begin{equation}
\begin{aligned}
&\sum_{j=1}^{d} \sum_{a_j\in \g(A_j)^+} \min(Alloc_{a_j},C_{a_j}) +  \sum_{j=1}^{d} \sum_{a_j\in \g(A_j)}\min(Alloc_{a_j},C_{a_j}) \\
\geq & \sum_{j=1}^{d-1} \sum_{a_j\in \g(A_j)^+} \min(Alloc_{a_j},C_{a_j}) (1-\delta_{j+1}) \\ &+ \sum_{j=1}^{d-1} C(\g(A_{j}^-,A_{j+1}^+ ) )(1-\delta_{j+1}) \\
&+ (1-\delta_1)|N_{(\g(A_1)^+,I)}| \\
\geq &(1-\delta_{d}) \sum_{j=1}^{d-1} \sum_{a_j\in \g(A_j)^+} \min(Alloc_{a_j},C_{a_j})  \\ &+ (1-\delta_{d})\sum_{j=1}^{d-1} C(\g(A_{j}^-,A_{j+1}^+ ) ) \\
&+ (1-\delta_d)|N_{(\g(A_1)^+,I)}| \\
\end{aligned}
\end{equation}

Note that the last two terms in the inequality above is $(1-\delta_{d})$ times the value of Term~\eqref{eq:cut_remain}.
Now we can establish a relationship between Term~\eqref{eq:val_remain} and Term~\eqref{eq:cut_remain}:
\begin{equation}
\begin{aligned}\label{eq:rel_last_2}
&\sum_{a_d\in \g(A_d)} \min(Alloc_{a_d},C_{a_d})+ \sum_{a_d\in \g(A_d)^+} \min(Alloc_{a_d},C_{a_d}) \\
\geq& (1-\delta_{d})\sum_{j=1}^{d-1} C(\g(A_{j}^-,A_{j+1}^+ ) ) \\
&+ (1-\delta_d)|N_{(\g(A_1)^+,I)}| \\
&- \delta_{d} \sum_{j=1}^{d-1} \sum_{a_j\in \g(A_j)^+} \min(Alloc_{a_j},C_{a_j}) \\
&-  \sum_{j=1}^{d-1} \sum_{a_j\in \g(A_j)}\min(Alloc_{a_j},C_{a_j})
\end{aligned}
\end{equation}

Clearly, $\sum_{j=1}^{d-1} \sum_{a_j\in \g(A_j)^+} \min(Alloc_{a_j},C_{a_j}) $ is at most $d\val$. Thus, combining Eq~\eqref{eq:rel_first} and Eq~\eqref{eq:rel_last_2}, we have
\begin{equation}
\begin{aligned}\label{eq:thm_p1}
\val =& \sum_{a_d\in \g(A_d)^-} \min(Alloc_{a_d},C_{a_d}) \\
&+ \sum_{a_d\in \g(A_d)} \min(Alloc_{a_d},C_{a_d})+ \sum_{a_d\in \g(A_d)^+} \min(Alloc_{a_d},C_{a_d})  \\
\geq &\quad C(\g(A_d)^-) \\
& +(1-\delta_{d})\sum_{j=1}^{d-1} C(\g(A_{j}^-,A_{j+1}^+ ) ) \\
&+ (1-\delta_d)|N_{(\g(A_1)^+,I)}| \\
&- \delta_{d} \sum_{j=1}^{d-1} \sum_{a_j\in \g(A_j)^+} \min(Alloc_{a_j},C_{a_j}) \\
&-  \sum_{j=1}^{d-1} \sum_{a_j\in \g(A_j)}\min(Alloc_{a_j},C_{a_j}) \\
\geq &\quad (1-\delta_{d})C(\C) - d\delta_d \val \\
&- \sum_{j=1}^{d-1} \sum_{a_j\in \g(A_j)}\min(Alloc_{a_j},C_{a_j}) \\
\geq & \quad \frac{1-\delta_{d}}{1+d\delta_{d}}\opt(G') - \frac{1}{1+d\delta_{d}}\sum_{j=1}^{d-1} \sum_{a_j\in \g(A_j)}\min(Alloc_{a_j},C_{a_j})
\end{aligned}	
\end{equation}

Finally, using the same technique (dividing each layer into $P$ part and $Q$ part and summing all terms related to the gaps over all potential $l$), we can prove the following lemma easily:
\begin{lemma}
	If the four properties hold, when $T=O(\frac{n\log(n/\epsilon_{max})}{\epsilon_{max}\epsilon_{min}} )$, there exists an approximate $l$ such that
	\begin{equation}\label{eq:thm_p2}
	\opt(G)-\opt(G') + \sum_{a_j\in \g(A_j)}\min(Alloc_{a_j},C_{a_j}) \leq d\epsilon_{max} \val.
	\end{equation}
\end{lemma}

\begin{proof}[\textbf{Proof of Theorem~\ref{thm_dlayer}}]
	Combing Eq~\eqref{eq:thm_p1} and Eq~\eqref{eq:thm_p2}, we have
	\begin{equation}
	\begin{aligned}
	\val \geq &\quad (1-O(d\delta_{d})) \opt(G) \\
	&- (1-O(d\delta_{d})) (\opt(G)-\opt(G') + \sum_{a_j\in \g(A_j)}\min(Alloc_{a_j},C_{a_j}) ) \\
	\geq& \quad (1-O(d\delta_{d})) \opt(G) \\
	\geq & \quad \opt(G) \cdot \prod_{j=1}^{d}(1-O(d\epsilon_j))
	\end{aligned}
	\end{equation}
\end{proof}

Now we introduce our final algorithm by giving $\cond$ and setting an  $\epsilon_j$ for all $2 \leq j \leq d$. The $\cond$ function is stated in Algo~\ref{alg:d_cond}, nearly the same as Algo~\ref{alg:cond}. Clearly, Property~\ref{ppt:2_3} and Property~\ref{ppt:2_4} can be proved by the same technique as the proofs of Property~\ref{ppt:3} and Property~\ref{ppt:4}:

%Given any $\epsilon \in (0,1)$, to make sure that our solution

\begin{algorithm}[t]
	\caption{$\cond(N_{(a_j,A_{j+1})})$}
	\label{alg:d_cond}
	%\KwIn{ $N_{(a,\B)}$ }
	
	%Construct a new four-layer graph by splitting each vertex in $\mathbb{A}$ into two nodes (as shown in Fig.~\ref{fig:generalgraph}).
	
	Let $\alpha_{max}^{(a_j)} = \max \limits_{a_{j+1}\in N_{(a_j,A_{j+1})}} \alpha_{a_{j+1}}$ and $N_{(a_j,A_{j+1})}^* := \{ a_{j+1}\in N_{(a_j,A_{j+1})}| \alpha_{a_{j+1}} = \alpha_{max}^{(a_j)}  \} $. \
	
	\quad
	
	\If{ $\forall a_{j+1}\in N_{(a_j,A_{j+1})}^*$, $\alpha_{a_{j+1}}$ decreases in this iteration and $\lev(a_{j+1}) - \lev(a_j) < \log(n/\epsilon_{max}) / \epsilon_{min}$  }
	{
		
		return $\true$
	}
	\Else
	{
		return False
	}
\end{algorithm}

\begin{lemma}
	The layer dominance (Property~\ref{ppt:2_3}) and the forced decrease exemption (Property~\ref{ppt:2_4}) hold if we use $\cond$ in Algo~\ref{alg:d_cond}.
\end{lemma}

Now we give an approximate $\{ \epsilon \}$ to obtain Property~\ref{ppt:2_1} and Property~\ref{ppt:2_2}:
\begin{lemma}
	If for any $1\leq j < d$, we have $\epsilon_j = \frac{\epsilon_{j+1}}{2n}$, then the increasing monotonicity (Property~\ref{ppt:1}) and the decreasing monotonicity (Property~\ref{ppt:2}) hold when using $\cond$ in Algo~\ref{alg:d_cond}.
\end{lemma}

\begin{proof}
	Let us prove Property~\ref{ppt:2_1} holds first. %
	Clearly, for any vertex $a_1\in A_1$, if its weight does not decrease in one iteration, $Alloc_{a_1}$ will not decrease.
	To prove this property holds for the remaining layer, we generalize Claim~\ref{claim:1} to the layers in our model:
	\begin{claim}\label{claim:2_1}
		For any vertex $a_j \in A_j$ $(2\leq j \leq d)$, consider any vertex $a_{j-1} \in N_{(a_j,A_{j-1})}$, if $\alpha_{a_j}$ does not decrease but $\min(Alloc_{a_{j-1}},C_{a_{j-1}} )$ decreases in iteration $t$, we have
		\[ x_{(a_{j-1}, a_j)}^{(t)} \geq (1+\frac{\epsilon_j}{n}) x_{(a_{j-1}, a_j)}^{(t-1)}  \]
	\end{claim}
	The proof of this claim is the same as that of Claim~\ref{claim:1}.
	Since in one iteration, for any vertex $a_{j-1}\in A_{j-1}$, $\min(Alloc_{a_{j-1}},C_{a_{j-1}} )$ increases at most $\prod_{j'=1}^{j-1} (1+\epsilon_{j'})$, if we can obtain that for any $2\leq j \leq d$,
	
	\begin{equation}\label{eq:induction}
	\prod_{j'=1}^{j-1} (1+\epsilon_{j'}) \leq  1+\frac{\epsilon_j}{n},
	\end{equation}
	then Property~\ref{ppt:2_1} holds.
	
	The above inequality can be proved inductively.
	Clearly, when $j=2$, if let $\epsilon_1 = \frac{\epsilon_2}{2n}$, we have
	\[ 1+\epsilon_1 \leq 1+\frac{\epsilon_2}{n}. \]
	Assuming that for any $j \leq  k$, Eq~\eqref{eq:induction} holds, now we prove that when $j=k+1$, Eq~\eqref{eq:induction} still holds.
	\begin{equation}
	\begin{aligned}
	\prod_{j'=1}^{k} (1+\epsilon_{j'}) = (1+\epsilon_k) \prod_{j'=1}^{k-1} (1+\epsilon_{j'})
	\end{aligned}
	\end{equation}
	According to our assumption, we have
	\begin{equation}
	\begin{aligned}
	\prod_{j'=1}^{k} (1+\epsilon_{j'}) \leq& (1+\epsilon_k) (1+\frac{\epsilon_k}{n}) \\
	\leq & 1+\epsilon_k + \epsilon_k(\frac{1}{n} + \frac{\epsilon_k}{n} ) \\
	\leq & 1+ 2\epsilon_k \\
	= & 1+\frac{\epsilon_{k+1}}{n}
	\end{aligned}
	\end{equation}
	Thus, Eq~\eqref{eq:induction} holds for any $2\leq j \leq d$, indicating that for any vertex $a_j \in A_j$, if $\alpha_{a_j}$ does not decrease in one iteration, $Alloc_{a_j}$ will also not decrease.
	
	Property~\ref{ppt:2_2} can also be proved similarly by the following claim:
	\begin{claim}\label{claim:2_2}
		For any vertex $a_j \in A_j$ $(2\leq j \leq d)$, consider any vertex $a_{j-1} \in N_{(a_j,A_{j-1})}$, if $\alpha_{a_j}$ decreases but $\min(Alloc_{a_{j-1}},C_{a_{j-1}} )$ increases in iteration $t$, we have
		\[ x_{(a_{j-1}, a_j)}^{(t)} \leq x_{(a_{j-1}, a_j)}^{(t-1)} / (1+\frac{\epsilon_j}{n})  \]
	\end{claim}
\end{proof}

According to the two lemmas above and Theorem~\ref{thm_dlayer}, if we let $\epsilon_j = \frac{\epsilon_{j+1}}{2n}$ for any $1\leq j < d$ and use $\cond$ in Algo~\ref{alg:d_cond}, our algorithm will return a $\prod_{j=1}^{d}(1-O(\epsilon_j d))$-approximated solution when
$T=O(\frac{n\log(n/\epsilon_{max})}{\epsilon_{max}\epsilon_{min}} )$.

\begin{lemma}~\label{lem:d_running_time}
	To obtain a $(1-\epsilon)$-approximated solution for an $(s$-$t)$ $(d+1)$-layered graph, the running time of our algorithm is $O(d^2n^{d+2}\log(n/\epsilon)/\epsilon^2)$.
\end{lemma}
\begin{proof}
	Due to our construction of $\{\epsilon\}$, we have
	\[ \prod_{j=1}^{d}(1-\epsilon_j ) \leq 1-2\epsilon_d,  \]
	 indicating that we can obtain a $(1-O(\epsilon_d d))$-approximated solution when the number of iterations is $O(\frac{n^d\log(n/\epsilon_d)}{\epsilon_d^2} )$.
	
	For any $\epsilon\in (0,1)$, to achieve a $(1-O(\epsilon))$-approximated solution, we need to set $\epsilon_d = \frac{\epsilon}{d}$, so
	the number of iterations will be $O(\frac{d^2n^d\log(n/\epsilon)}{\epsilon^2})$.
	In each iteration, we compute $x_{u,v}$ for each edge in $G$ and update each vertex weight. Thus, the running time of an iteration is still $O(n^2)$, completing the proof.

\end{proof}

Note that the diameter of the above graph excluding $t$ is $d+1$. When this value is $d$, we obtain the running time $O(d^2n^{d+1}\log(n/\epsilon)/\epsilon^2)$, as claimed in Theorem~\ref{thm:weight_DAG}.

Finally, we consider the virtual-weight dependence of these weights.
\begin{theorem}\label{thm:d_layer_good_weights}
	Under the framework of Algo~\ref{alg:d_layer_framework}, if letting $\epsilon_j = \epsilon_{j+1}/(2n)$ for each $1\leq j \leq d-1$ and using $\cond$ in Algo~\ref{alg:d_cond}, given an reduction graph $\Gr$, for any two neighboring copies $\vr_j,\vr_{j+1}$ of any vertex $v$, we have
	\[\alpha_{\vr_{j+1}} = (\alpha_{\vr_j})^{2n} .\]
\end{theorem}

The proof of this theorem is similar to the proof of Theorem~\ref{thm:3_layer_good_weights}. For the copies of one vertex, they share the same decreasing steps according to $\cond$ in Algo~\ref{alg:d_cond}. Since for each $1\leq j \leq d-1$, $\epsilon_j = \epsilon_{j+1}/(2n)$, we can obtain $\alpha_{\vr_{j+1}} = (\alpha_{\vr_j})^{2n} $ by the same algebra.

%% file: apdxlearnability.tex
\section{The Learnability of Vertex Weights for Online Flow Allocation in general DAGs}\label{sec:apdxlearnability}

In this part, we consider the learnability of vertex weights in a DAG $G$ and give the whole proof of Theorem~\ref{thm:learnability}.
According to the reduction in Section~\ref{sec:existweight}, we can always assume that $G$ is a $d$-layered graph.
We use an inductive method to prove the learnability. As we shown in Section~\ref{sec:learnability}, the vertex weights are PAC-learnable when $d=2$. Now we claim that if the vertex weights are learnable in $d$-layered graphs, under some mild assumptions, for $(d+1)$-layered graphs, the vertex weights are also learnable.

The basic framework is the same as the previous. We still employ Algo~\ref{alg:learning} as our learning algorithm. Similarly, due to the definition of $\{\predw\}$, we have
\begin{equation}\label{eq:d_near_optimal}
R(\predw,\predI) \geq (1-\epsilon)\opt(\predI) \geq (1-\epsilon) R(\alpha^*,\predI),
\end{equation}
where $\opt(\predI)$ is the optimal value of instance $\predI$.

And we only need to focus on prove the following two inequalities:
\begin{equation}\label{eq:d_learn_1}
R(\alpha^*,\predI) \geq (1-O(\epsilon) ) \E_{I\sim \cD^m} [R(\alpha^*,I)],
\end{equation}
and
\begin{equation}\label{eq:d_learn_2}
\E_{I\sim \cD^m} [R(\predw,I)] \geq (1-O(\epsilon)) R(\predw,\predI),
\end{equation}

Define $R(\alpha,\E[I]) $ be the value of flow obtained by $\{\alpha\}$ on the expected impression set $\E[I]$.
Again due to the concavity of $\min$ function and and Jensen's inequality, Eq~\eqref{eq:d_learn_1} can obtained if we prove the following lemma:

\begin{lemma}\label{lem:d_sample_complexity}
	Given any $\epsilon >0$, $\delta \in (0,1]$ and vertex weights $\{ \alpha \}$, if the number of instances $s$ is no less than $ O(\frac{n^2}{\epsilon^2}\ln(\frac{n}{\delta}))$,
	with probability at least $1-\delta$, for each $a\in A$,
	\begin{equation}\label{eq:d_sample_comlexity}
|  R(\alpha,\E[I])  - R(\alpha,\predI)|  \leq O(\epsilon) R(\alpha,\E[I]) .
	\end{equation}
\end{lemma}

\begin{proof}
	Use $A_1$,...$A_d$ to denote the $d$ offline layers in the $(d+1)$-layered graph (Note that the layer $I$ is the online layer ).
	According to the proof of Lemma~\ref{lem:bi_sample_complexity}, we know in the first layer $A_1$, if the number of instances $s$ is no less than $ O(\frac{n}{\epsilon^2}\ln(\frac{n}{\delta}))$,
	with probability at least $1-\delta$, for each $a_1\in A_1$, we have
	\begin{equation}\label{eq:d_learn_1_2}
		\begin{aligned}
		&|\min (Alloc_{a_1}(\alpha,\predI),C_{a_1}) - \min(Alloc_{a_1}(\alpha,\E[I]),C_{a_1})| \\= &|\min (Alloc_{a_1}(\alpha,\predI),C_{a_1}) - \min(\E[Alloc_{a_1}(\alpha,I)],C_{a_1})| \\
		\leq &   \epsilon\min(\E[Alloc_{a_1}(\alpha,I)], C_{a_1} )+ \epsilon/n \\
		=& \epsilon\min(Alloc_{a_1}(\alpha,\E[I]), C_{a_1} )+ \epsilon/n
		\end{aligned}
	\end{equation}
	
	For each $a_2 \in A_2$, 
	\[Alloc_{a_2} (\alpha,\E[I]) = \sum \limits_{a_1\in N(a_2,A_1)}\min(Alloc_{a_1}(\alpha,\E[I]),C_{a_1}) y_{a_1,a_2} ,\]
	 where $y_{a_1,a_2}$ is the proportion of flow sent from $a_1$ to $a_2$, a fixed value if $\{\alpha\}$ is fixed.
	 Thus, if Eq~\eqref{eq:d_learn_1_2} holds for any $a_1\in A_1$, for each $a_2\in A_2$, we also have
	 \begin{equation}
	 	|\min (Alloc_{a_2}(\alpha,\predI),C_{a_2}) - \min(Alloc_{a_2}(\alpha,\E[I]),C_{a_2})| \leq \epsilon\min(Alloc_{a_2}(\alpha,\E[I]), C_{a_2} )+ O(\epsilon/n).
	 \end{equation}
	 
	 Inductively, this inequality holds for the all layers. By the definition of $R(\alpha,\E[I]) $, we have
	 \[ |  R(\alpha,\E[I])  - R(\alpha,\predI)|  \leq O(\epsilon) R(\alpha,\E[I]) . \]
	
\end{proof}

Now we focus on the second inequality.

\begin{lemma}\label{lem:d_learn_2}
	For any $\epsilon \in (0,1)$, 
	if for $d$-layered graphs, we have 
	\begin{equation}\label{eq:d_learn_2_1}
	\E[R(\predw,I)] \geq (1-O(\epsilon)) R(\predw,\predI)
	\end{equation}
	when $C_a = O(\frac{1}{\epsilon^2}(\ln\frac{1}{\epsilon}))$ for each vertex $a$,
	then for $(d+1)$-layered graphs, 
	\begin{equation}\label{eq:d_learn_2_2}
	\E[R(\predw,I)] \geq (1-O(\epsilon)) R(\predw,\predI)
	\end{equation}
	when $C_a = O(\frac{1}{\epsilon^2}(\ln\frac{1}{\epsilon}))$ for each vertex $a$ and in the optimal flow of instance $\predI$, the load of each vertex is at least $O(\frac{1}{\epsilon^2}(\ln\frac{1}{\epsilon}))$.
\end{lemma}

\begin{proof}
	
	Use $A_1$,...,$A_d$ to represent the $d$ offline layers. We discuss this problem in two cases:
	
	(1) Employing $\{ \predw \}$ to instance $\predI$, no vertex in $A_d$ has a load larger than its capacity.
	
	(2) Employing $\{\predw\}$ to instance $\predI$, at least one vertex in $A_d$ has a load larger than its capacity.
	
	\textbf{Case 1}:
	
	The basic idea of this proof is to construct a $d$-layered graph $G'$ with each vertex capacity at least $O(\frac{1}{\epsilon^2}(\ln\frac{1}{\epsilon}))$ and a set of weights $\{\alpha'\}$ for vertices in $G'$, such that 	
	
	(1) For any impression set $I$, $R(\alpha,I) \geq R'(\alpha',I)$, where $R'(\alpha',I)$ is the value of flow obtained by weights $\{ \alpha' \}$ in the $d$-layered instance $(I,G')$.
	
	(2) For the impression set $\predI$, $R(\alpha,\predI) = R'(\alpha',\predI)$.
	
	Assume there exists such a $d$-layered graph, this lemma can be proved directly:
	\begin{equation}
		\begin{aligned}
		\E[R(\predw,I)]  \geq& \E[ R'(\alpha',I) ] \\
		\geq & (1-O(\epsilon)) R'(\alpha',\predI) \\
		= & (1-O(\epsilon)) R(\predw,\predI) 
		\end{aligned}
	\end{equation}
	The first step and the third step holds due to the two properties. The second step holds due to our assumption for $d$-layered graphs.
	
	Now we construct the $d$-layered graph $G'$ and its capacity function $C'$. The structure of $G'$ is obtained by removing all vertices in the last layer $A_d$ of $G$. The capacity of each vertex in $A'_1$,...,$A'_{d-2}$ is the same as that in $G$ while the capacity of each vertex in $A'_{d-1}$ is set to be the feasible load in this vertex obtained by $\predw$ in instance $(\predI,G)$. According to our assumption about the optimal flow of instance $(\predI,G)$, the capacity of each vertex in $A'_{d-1}$ is at least $O(\frac{1}{\epsilon^2}(\ln\frac{1}{\epsilon}))$. Clearly, for each vertex $a\in G'$, the new capacity cannot be larger than its capacity in $G$: $C'_a \leq C_a$.
	
	The weights of each vertex in $G'$ are the same as that in $G$. 
	Since for each vertex $a$, $C'_a \leq C_a$ and for the last layer $A'_{d-1}$, its capacity function is a set of feasible loads, we have $R(\alpha,I) \geq R'(\alpha',I)$ for any impression set $I$.
	
	According to our construction, their performances are the same given the impression set $\predI$. Namely, $R(\alpha,I) = R'(\alpha',I)$, completing the proof of this case.
	
	\textbf{Case 2}:

	For each $a \in A_d$, let $Alloc_a(\predw)$ be its load. By the definition of $R(\alpha,I) $, we have
	\[  R(\alpha,I)  = \sum_{a\in A_d} \min(Alloc_a(\predw),C_a). \]
	
	Use $B$ to represent the set of vertices in $A_d$ with a load larger than the capacity.
	Now we construct a new capacity function $C'$ for the graph $G$ and use $R'(\predw, I)$ to denote the objective value under the capacity function $C'$.
	For each vertex not in $B$, its capacity is the same. 
	For each vertex $b\in B$, let $C'_b = Alloc_b(\predw)$.
	
	Clearly, if we use the new capacity function, no vertex in $A_d$ has a load larger than its capacity. According to the proof of Case 1, we have
	\begin{equation}\label{eq:d_learn_2_case2}
		\E[R'(\predw,I)] \geq (1-O(\epsilon)) R'(\predw,\predI) . 
	\end{equation}

	%By Lemma~\ref{lem:d_sample_complexity}, 
%	\[  \E[R'_{\predw}(I)] \geq (1-O(\epsilon)) R'_{\predw}(\E[I])   \]
	
	Now we analysis the changes of $\E[R(\predw,I)] $ and $R(\predw,\predI) $ if we decrease the capacity of each vertex $b\in B$ from $C'_b$ to $C_b$.
	 
	Do a expansion for $\E[R(\predw,I) ]$: 
	\begin{equation}
		\begin{aligned}
		\E[R(\predw,I) ]= & \sum_{a\in A_{d}} \E[ \min(Alloc_a(\predw,I),C_a) ]\\
		=& \sum_{a\in A_{d}} \sum_{I\in \Pi} \Pr[I] \min(Alloc_a(\predw,I),C_a) 
		\end{aligned}
	\end{equation} 
	
	Do a similar expansion for $R(\predw,\predI) $:
	\begin{equation}
	\begin{aligned}
	R(\predw,\predI) = & \sum_{a\in A_{d}} \min(Alloc_a(\predw,\predI),C_a) \\
	=& \sum_{a\in A_{d}} \sum_{I\in \Pi} \Pr[I] \min(Alloc_a(\predw,\predI),C_a) 
	\end{aligned}
	\end{equation} 
	
	Note that in the expansion of $R(\predw,\predI) $, $Alloc_a(\predw,\predI)$ is a fixed value. Thus, for each term in the sum $\Pr[I] \min(Alloc_b(\predw,\predI),C_b)$, if $C_b$ decreases, its value will definitely decrease.
	
	However, in the expansion of $\E[R(\predw,I) ]$, not all such terms decrease. If $Alloc_a(\predw,I)$ is small enough, the value of term $\Pr[I] \min(Alloc_a(\predw,I),C_a)$ will not decrease.
	
	Thus, if we decrease the capacity of each vertex $b\in B$ from $C'_b$ to $C_b$,  $\E[R(\predw,I)] $ decreases slower than $R(\predw,\predI) $. By Eq~\eqref{eq:d_learn_2_case2}, we can claim that under the original capacity function, 
	\begin{equation}
	\E[R(\predw,I)] \geq (1-O(\epsilon)) R(\predw,\predI) ,
	\end{equation}
	completing the proof of this case.

\end{proof}

%% file: robustness_app.tex
\section{Parameter Robustness of Predictions for Online Flow Allocation} \label{app:robustness}

%\subsection{Weight Predictions}
In this section, we focus on the parameter robustness. 
As mentioned above, for any online instance there exists a set of vertex weights for $V$ which can return a near optimal solution. Now we assume the online algorithm is given predictions of these vertex weights in the beginning. Our goal is to give an online algorithm based on these weights, which can obtain a competitive ratio related to the error of these predictions.

We first define the prediction error $\eta$.
Consider a prediction of vertex weight $\predw_v$ for each vertex $v\in V$. Due to scale invariance, we can assume that the minimum predicted vertex weight $\predw_{min}=1$.
Use $\{ \alpha_v^* \}_{v\in V}$ to represent the optimal vertex weights, namely, the weights that can achieve an $(1-\epsilon)$-approximate solution. Similarly, let $\alpha^*_{min}=1$.
Now define the prediction error \[\eta := \max_{v\in V} \left(\frac{\predw_v}{\alpha^*_v} ,\frac{\alpha_v^*}{\predw_v} \right).\]
We have the following claim:% that is proven in Appendix~\ref{app:robustness}

\begin{theorem}\label{thm:weight_predicitive_algorithm}
	Given any weight predictions $\{ \predw \}$, we can obtain a solution with competitive ratio
	\[ \max\left(\frac{1}{d+1},\frac{1-\epsilon}{\eta^{2d}}\right), \]
	where $d$ is the diameter of the graph $G$.
\end{theorem}

\begin{proof}
	Due to the reduction in the Appendix~\ref{sec:existweight}, we can assume w.l.o.g. that $G$ is $d$-layered in the following.
	According to Theorem~\ref{thm:maximal_bound}, in this proof, we only need to give an online algorithm with competitive ratio $\frac{1-\epsilon}{\eta^{2d}}$.
	%To prove the theorem above, we first give an online algorithm with competitive ratio $\frac{1-\epsilon}{\eta^{2d}}$, and then show that we can modify this algorithm such that its competitive ratio does not decrease and is always better than $\frac{1}{d+1}$.  
	
	%show that for any online algorithm, we can modify it such that its competitive ratio is always better than $\frac{1}{d+1}$ and then give an online algorithm with competitive ratio $\frac{1-\epsilon}{\eta^{2d}}$.
	
	\begin{lemma}\label{lem:predictive_d_layer}
		If route the flow according to the predicted weights $\{\predw\}$ directly, we can obtain a solution with competitive ratio $(1-\epsilon)/\eta^{2d}$.
	\end{lemma}
	
	\begin{proof}[Proof of Lemma~\ref{lem:predictive_d_layer}]
		We give an inductive proof. Consider the first layer $A_1$. For each impression $i$ and $a\in N(i,A_1)$, we have
		\begin{equation}
		\begin{aligned}
		\predx_{i,a} = &\frac{\predw_a}{\sum_{a'\in N(i,A_1)} \predw_{a'} }\\
		\geq& \frac{\alpha_a^*/ \eta}{\sum_{a'\in N(i,A_1)}  \alpha_{a'}^* \eta} \\
		= &  \frac{1}{\eta^2}\frac{\alpha_a^*}{\sum_{a'\in N(i,A_1)}  \alpha_{a'}^*}\\
		= & \frac{1}{\eta^2}x_{i,a}^*
		\end{aligned}
		\end{equation}
		Thus, for each $a\in A_1$, 
		\begin{equation}
		\min(\predA_a, C_a) \geq \frac{1}{\eta^2}\min({Alloc}_a^*, C_a) 
		\end{equation}
		
		Now we prove that if for each vertex $a\in A_j$, \[\min(\predA_a, C_a) \geq \frac{1}{\eta^{2j}}\min({Alloc}_a^*, C_a), \]
		then for each vertex $a\in A_{j+1}$, 
		\[\min(\predA_a, C_a) \geq \frac{1}{\eta^{2(j+1)}}\min({Alloc}_a^*, C_a). \]
		
		For each $a\in A_{j+1}$ and $a'\in N(a,A_j)$, similarly, we have $\predx_{a',a} \geq  \frac{1}{\eta^2}x_{a',a}^*.$ The contribution of vertex $a'$ to $a$ is 
		\begin{equation}
		\min(\predA_{a'},C_{a'}) \predx_{a',a} \geq \frac{1}{\eta^{2(j+1)}}\min({Alloc}_a^*, C_a) x_{a',a}^*.
		\end{equation}
		Thus, for each vertex $a\in A_{j+1}$, 
		\[\min(\predA_a, C_a) \geq \frac{1}{\eta^{2(j+1)}}\min({Alloc}_a^*, C_a). \]
		
		The value $\val$ of our solution is $(1-\epsilon)/\eta^{2d}$-approximated:
		\begin{equation}
		\begin{aligned}
		\val = &\sum_{a\in A_d} \min(\predA_a,C_a ) \\ 
		\geq &\frac{1}{\eta^{2d}}\sum_{a\in A_d} \min(Alloc_a^*,C_a )  \\
		\geq & \frac{1-\epsilon}{\eta^{2d}}\opt. 
		\end{aligned}
		\end{equation}
		
	\end{proof}
\end{proof}

%\subsubsection{The Bipartite Case}
\subsection{2-layered Graphs}\label{subsec:bi_robust}

If $d=1$, this problem is an online bipartite matching problem. In this special case, we can get an improved result with a more graceful degradation in the error $\eta$.  %Again the following result is proven in Appendix~\ref{app:robustness}.

\begin{theorem}\label{thm:online_algo}
	Assume that $\opt$ can assign all impressions.  For any given $\epsilon\in (0,1)$, There exists an online algorithm with a competitive ratio of
	\[ 1-4\epsilon\log_{(1+\epsilon)} \eta-3\epsilon. \]
\end{theorem}

\begin{proof}%[\textbf{Proof of Theorem~\ref{thm:online_algo} }]
	
	\begin{algorithm}[t]
		\caption{Online algorithm with predictions when $d=1$}
		\label{alg:online}
		\KwIn{$G= (I \cup A\cup \{t\},E)$ where $I$ and $E$ arrive online, $\{ C_a \}_{a \in A}$, parameter $\epsilon \in (0,1)$, prediction $\{\predw\}$}
		Create a same imaginary instance as the input instance, we use $x'$ and $Alloc'$ to denote the assignment in this instance. \
		
		\While{an impression $i$ comes }
		{	
			Initially, set $x_{i,a} = x'_{i,a}=0$ for all $a\in N_i$.

			\While{$\sum_{a\in N_i}x_{i,a}<1$}
			{
				For each $a\in N_i$, increase $x_{i,a}$ and $x'_{i,a}$ with rate $\frac{\predw_a}{\sum_{a' \in N_i} \predw_{a'} }$ until $\sum_{a\in N_i}x_{i,a}=1$ or there exists one advertiser $a$ such that $Alloc'_a = (1+\epsilon)^2C_a$.
				
				\While{there exists one advertiser $a$ such that $Alloc'_a \geq (1+\epsilon)^2C_a$}
				{
					For each $a \in A$, if $Alloc_a' \geq (1+\epsilon)^2C_a$, $\predw_a \leftarrow \predw_a / (1+\epsilon)$.\
					
					Update each $x_{i,a}'$ and $Alloc_a'$ according to new $\{ \predw_a\}$.			
				}		
			}
		}
	\end{algorithm}
	
	The algorithm is presented in Algo.~\ref{alg:online}.
	We first prove that our algorithm can terminate. Namely, when a new impression arrives, we can always find the weights $\{ \predw \}$ such that in the imaginary instance, $\forall a\in A$, $Alloc_a' < (1+\epsilon)^2C_a$.
	
	%Since we assume that all impressions can be assigned by a set of optimal weights $\{ \alpha^*\}$, $\forall a\in A$, 
	Let $\{ \alpha^*\}$ be the optimal weights. Since we assume that all impressions can be assigned, we can assume that for any advertiser $a$, we have $Alloc_a^* < (1+\epsilon)C_a$. We give the following claim to show that the algorithm will not fall into an endless loop:
	\begin{claim}\label{claim:terminate}
		In any time during the algorithm, for any $a\in A$, $\predw_a \geq \alpha_a^*/\eta$.
	\end{claim}
	\begin{proof}[Proof of Claim~\ref{claim:terminate}]
		Assuming that at some point, there existed some  $\predw_a < \alpha_a^*/\eta$.
		Use $b$ to represent a such vertex.
		Consider the first time $t$ that this event occurred. 
		Clearly, $Alloc_{b}'(t-1) \geq (1+\epsilon)^2C_{b}$. Since $t$ is the first time, we also have for any $a$, $\predw_{a}(t-1) \geq \alpha_{a}^* /\eta$ and $\predw_{b}(t-1) < \alpha_{b}^*(1+\epsilon) /\eta$. For any $x'_{i,b}$ in time $t-1$, we have that
		\begin{equation}
		x'_{i,b}(t-1) = \frac{\predw_{b}(t-1)}{\sum_{a\in N_i} \predw_a(t-1)} < \frac{(1+\epsilon) \alpha_{b}^*/\eta }{\sum_{a\in N_i} \alpha_a^*/\eta  }  = (1+\epsilon) x_{i,b}^*.
		\end{equation}
		Thus, we know in the time $t-1$, 
		\begin{equation}
		Alloc_{b}'(t-1) < (1+\epsilon) Alloc_{b}^* < (1+\epsilon)^2C_{b},
		\end{equation}
		contradicting to the fact that $Alloc_{b}'(t-1) \geq (1+\epsilon)^2C_{b}$. 
	\end{proof}
	
	According to Claim~\ref{claim:terminate}, we can make sure that our algorithm can terminate because we can not keep decreasing $\predw$. The value of ${\predw}$ will stop when it is close to the nearest ${\alpha^*/\eta}$.
	
	For any advertiser $a$, use $k$ to represent the number of times that its weight decreased and $w_1,w_2,...,w_k$ to denote $k$ different values of its weight. Our initial prediction $\predw_a$ is viewed as $w_0$. Since each time, the weight decreased by $(1+\epsilon)$, we have $w_k = w_0/(1+\epsilon)^k$.
	Due to Claim~\ref{claim:terminate}, $w_k \geq \beta_a^*$. Thus, 
	\begin{equation}
	\begin{aligned}
	\frac{w_0}{(1+\epsilon)^k}& \geq \alpha_a^*/\eta \\
	(1+\epsilon)^k &\leq \eta \frac{\predw_a}{\alpha_a^*} \\
	k \leq 2\log&_{(1+\epsilon)} \eta
	\end{aligned}
	\end{equation}
	Namely, the weight decreased at most $2\log_{(1+\epsilon)} \eta$ times.
	Use $Alloc_a^{(j)}$ to represent the number of impressions that assigned to $a$ when $a$'s weight is $w_j$. In the following, we will try to bound each $Alloc_a^{(j)}$ to give a upper bound of the final $Alloc_a$.
	
	\begin{lemma}~\label{lem:before_decrease}
		For any advertiser $a$, $Alloc_a^{(0)} \leq (1+\epsilon)^2 C_a$.
	\end{lemma}
	
	\begin{proof}[Proof of Lemma~\ref{lem:before_decrease}]
		This lemma can be proved very easily. For any advertiser $a$, if its weight has not decreased so far, we have $Alloc_a' < (1+\epsilon)^2C_a$.

		When one impression arrived, in the first step, the increments of $Alloc_a$ and $Alloc_a'$ are the same. In the second step, $Alloc_a'$ may change due to the decrease of other weights. Since $\predw_a$ has not decreased, $Alloc_a'$ cannot decrease in this step. Thus, we have $Alloc_a \leq Alloc_a' < (1+\epsilon)^2C_a$.
	\end{proof}

	\begin{lemma}~\label{lem:after_decrease}
		For any advertiser $a$ and any $j \geq 1$, $Alloc_a^{(j)} \leq 2\epsilon C_a$.
	\end{lemma}
	
	\begin{proof}[Proof of Lemma~\ref{lem:after_decrease}]
		As we mentioned above, if $\predw_a$ did not change, the increment of $Alloc_a$ is no more than the increment of $Alloc_a'$. Thus, if we prove that during the period that $\predw_a = w_j$, $Alloc_a'$ increased at most $3\epsilon C_a$, this lemma can be proved.
		
		Consider the moment that advertiser $a$'s weight decreased from $w_{j-1}$ to $w_j$. Use $Alloc_a'(-)$ and $Alloc_a'(+)$ to denote the allocation in the imaginary instance before and after this decrease. Clearly, we have $Alloc_a'(-)\geq (1+\epsilon)^2C_a$. When the weight decreased by $(1+\epsilon)$, the allocation decreased by at most $(1+\epsilon)$. So after this decrease, 
		\[ Alloc_a'(+) \geq Alloc_a'(-)/(1+\epsilon) \geq (1+\epsilon)C_a. \]
		
		During the period that $\predw_a = w_j$, $Alloc_a'$ cannot become larger than $(1+\epsilon)^2C_a$ according to our algorithm. Thus, the increment of $Alloc_a'$ in this period is at most 
		\[ (1+\epsilon)^2C_a - (1+\epsilon)C_a\leq 2\epsilon C_a, \]
		
		completing this proof.
		
	\end{proof}
	
	Combining Lemma~\ref{lem:before_decrease} and Lemma~\ref{lem:after_decrease}, we have 
	\begin{equation}\label{eq:log_eta}
	Alloc_a = \sum_{j} Alloc_a^{(j)} \leq (1+3\epsilon+4\epsilon \log_{(1+\epsilon)}\eta)C_a
	\end{equation}
	
	This equation indicates that for each advertiser $a$, its allocation $Alloc_a$ is at most $(1+3\epsilon+4\epsilon \log_{(1+\epsilon)}\eta) C_a$. If we increase the capacity of each advertiser $a$ from $C_a$ to $C'_a=(1+3\epsilon+4\epsilon \log_{(1+\epsilon)}\eta) C_a$, all $m$ impressions can be assigned. In other words,
	\[ \sum_a \min(Alloc_a,C'_a) =m \geq \opt. \]
	For each $a$, due to the definition of $C_a'$, we have
	\[ \min (Alloc_a,C_a) = \min( Alloc_a,  \frac{1}{1+3\epsilon+4\epsilon \log_{(1+\epsilon)}\eta  } C_a' ) \geq \frac{1}{1+3\epsilon+4\epsilon \log_{(1+\epsilon)}\eta  } \min (Alloc_a, C_a').\]
	
	Combining the above two inequalities, we have
	\[ \sum_a \min(Alloc_a,C_a) \geq  \frac{m}{1+3\epsilon+4\epsilon \log_{(1+\epsilon)}\eta  }  \geq (1-3\epsilon-4\epsilon \log_{(1+\epsilon)}\eta  ) \opt,\]
	completing the proof of Theorem~\ref{thm:online_algo}.

	Note that according to Theorem~\ref{thm:maximal_bound}, we can also come up with an algorithm with a competitive ratio of $\max(1-4\epsilon\log_{(1+\epsilon)} \eta-3\epsilon, 1/2)$. When the predictions are nearly correct, we can obtain a near optimal solution. When $\eta$ is large, the ratio will not be worse than $1/2$.

\end{proof}

Additionally, we can show that in some sense this is the best you can do in this setting (up to constant factors).  %The following upper bound result is shown in Appendix~\ref{app:robustness}.

\begin{theorem}\label{thm:bound}
	Consider the online flow allocation problem. For any online algorithm with weight predictions, even if $d=1$, its competitive ratio is not better than $1-\Omega(\log\eta)$.
\end{theorem}

\begin{proof}
	
	The basic idea is to construct a set of impressions and predictions such that $\eta$ is very small and the expectation of any algorithm's competitive ratio is at most $1-\log \eta$, indicating the worst ratio among these instances cannot be better than this expectation value.  
	
	More specifically, the graph $G$ has $n$ advertisers $(a_1,...,a_n)$ and $n$ impressions $(i_1,...,i_n)$. Sample a uniform random permutation $\pi$ of set $[n]$. Given any parameter $0< s < n$, define the edge set to be
	\[ \E = \{ (a_{\pi(j)}, i_k)  | j,k \leq s  \} \cup \{ (a_{\pi(j)}, i_k) | j >s \text{ and } 1\leq k \leq n \}. \]
	Note that this parameter $s$ will be served as an bridge between the competitive ratio and the prediction error.
	
	The impressions arrive in the order $i_1,i_2,...,i_n$. For each advertiser $a$, we set its capacity $C_a$ to be $1$ and its predictive weight $\predw_a$ to be $1$.
	
	We can see Fig~\ref{fig:upperbound} as an illustration. Given the parameter $s$, we can partition all advertisers into two sets according to the permutation $\pi$. The first $s$ impressions are adjacent to all advertisers, while the last $n-s$ impressions only connect to all purple advertisers.
	
	\begin{figure}[t]
		\centering
		\captionsetup{justification=centering}
		\includegraphics[width=0.6\linewidth]{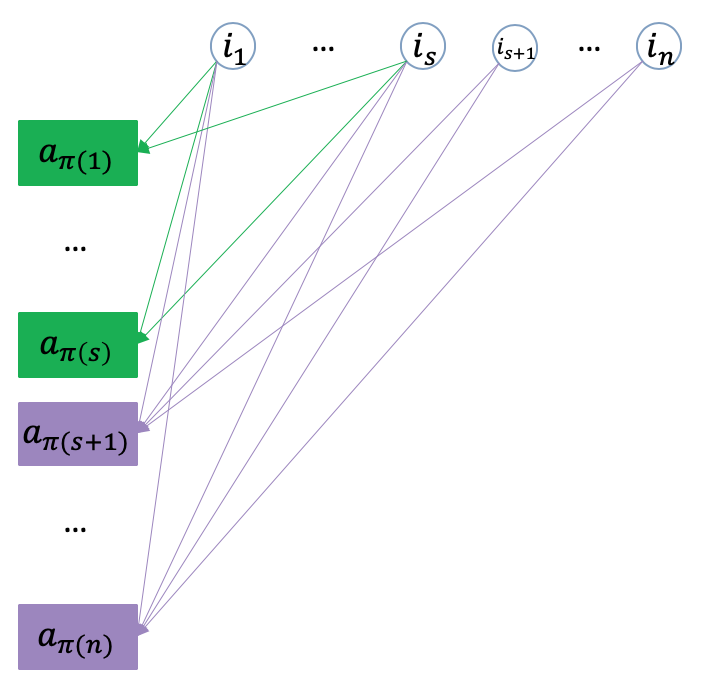}
		\caption{An illustration of the constructed instance. }
		\label{fig:upperbound}
	\end{figure}
	
	We first give the expected competitive ratio of this instance, then analyze $\{ \beta^* \}$ and give $\eta$. As mentioned above, both the ratio and $\eta$ will be related to parameter $s$. Thus, we draw the conclusion that if $s$ satisfies some conditions, the ratio will be at most $1-\log \eta$.
	
	\begin{lemma}\label{lem:bound_ratio}
		Let $p$ be $\frac{n-s}{n}$,  the expected competitive ratio $r$ of any online algorithm is at most $1-p(1-p)$.
	\end{lemma}
	\begin{proof}[Proof of Lemma~\ref{lem:bound_ratio}]
		Clearly, for any permutation $\pi$, there is always a perfect matching in this graph. Namely, the expected optimal value $\opt$ is $n$. Consider the expected value of each $x_{a_{\pi(j)},i_k}$. When $1\leq k\leq s$, this value is at most $\frac{1}{n}$ because for any two advertisers $a_{\pi(j_1)}$ and $a_{\pi(j_2)}$, we have $E(x_{a_{\pi(j_1)},i_k}) = E(x_{a_{\pi(j_2)},i_k})$ by symmetry, and the sum of all advertisers' values is at most $1$. Similarly, when $s<k<n$, $x_{a_{\pi(j)},i_k}$ is $0$ if $j \leq s$ and at most $\frac{1}{n-s}$ if $j>s$. Thus, each green advertiser matches at most $s/n$ impressions while each purple advertiser matches at most $1$ impressions. The expected competitive ratio of any online fractional matching algorithm can be bounded:
		
		\begin{equation}
		\begin{aligned}
		r &= \frac{\sum_{j=1}^{n} \sum_{k=1}^{n}E(x_{a_{\pi(j),i_k}}) }{n}  \\
		&\leq \frac{1}{n}(\frac{s}{n}\cdot s + n-s)\\
		& = 1-\frac{(n-s)s}{n^2} \\
		& = 1- p(1-p)
		\end{aligned}
		\end{equation}
	\end{proof}
	
	\begin{lemma}\label{lem:bound_error}
		Given any permutation $\pi$ and any $\epsilon \in (0,1)$, there exists a set of weights $\{ \beta^* \}$ with $\eta_{\beta^*} = p/\epsilon$ such that it can achieve a $(1-\epsilon)$-approximated fractional matching.  
	\end{lemma}
	
	\begin{proof}[Proof of Lemma~\ref{lem:bound_error}]
		Since the first $s$ advertisers in the permutation are equivalent, we let $\beta^*_{a_{\pi(1)}} = \beta^*_{a_{\pi(2)}} = ... = \beta^*_{a_{\pi(s)}} = w_1.$ Similarly, we have $\beta^*_{a_{\pi(s+1)}} = \beta^*_{a_{\pi(s+2)}} = ... = \beta^*_{a_{\pi(n)}}=w_2.$ Clearly, the last $n-s$ impressions will fill out the last $n-s$ advertisers. For each one in the first $s$ impressions, the unmatched proportion is the proportion that assigned to the last $n-s$ advertisers. Letting $w_2=1$, this proportion is $\frac{(n-s)}{sw_1+(n-s)}$. Thus, we can compute the size of the fractional matching obtained by these weights:
		\begin{equation}
		\begin{aligned}
		\val = n- \frac{s(n-s)}{sw_1+(n-s)}
		\end{aligned}
		\end{equation}
		We desire that these weights can obtain a $(1-\epsilon)$-approximated fractional matching. Thus, we have
		\begin{equation}
		\begin{aligned}
		n- \frac{s(n-s)}{sw_1+(n-s)} \geq (1-\epsilon)n 
		\end{aligned}
		\end{equation}
		Solving the inequality above, we have
		\begin{equation}
		\begin{aligned}
		w_1 \geq \frac{p}{\epsilon} -\frac{p}{1-p}
		\end{aligned}
		\end{equation}
		If $w_1=\frac{p}{\epsilon}$, we can achieve a $(1-\epsilon)$-approximated fractional matching. Since all predicted weights equal one, when $p > \epsilon$, the error is $p/\epsilon$, completing the proof.
	\end{proof}
	
	According to Lemma~\ref{lem:bound_error} and the definition of $\eta$, we have $\eta \leq p/\epsilon$. When $0<\log (\frac{p}{\epsilon}) \leq p(1-p)$, we can bounded the expected competitive ratio:
	\begin{equation}
	r \leq 1-p(1-p) \leq 1-\log (\frac{p}{\epsilon})  \leq 1- \log \eta,
	\end{equation}
	completing the proof of Theorem~\ref{thm:bound}.
	
\end{proof}

\subsection{A Worst-case Bound}\label{subsec:worst_case_bound}

To show that our algorithm is competitive, we presents a worst-case bound in the subsection:

\begin{theorem}\label{thm:worst_case_bound}
	Considering the integral version of the online flow allocation problem, for any deterministic algorithm, its competitive ratio cannot be better than $1/(d+1)$.
\end{theorem}

%The detail of this proof is provided in Appendix~\ref{app:robustness}

%Now we present the proofs of the results stated in Section~\ref{sec:robustness}.

%The above simple algorithm is $(1-\epsilon)$-consistent, but not robust. Now we improve it to obtain both consistency and robustness.
%give another algorithm to obtaining both consistency and robustness.

%Combing Lemma~\ref{lem:predictive_d_layer}, Lemma~\ref{lem:improved_online} and Lemma~\ref{lem:maximal_flow}, Theorem~\ref{thm:weight_predicitive_algorithm} can be proved directly.

%\subsection{Proof of Theorem~\ref{thm:bound}}

%This bound indicates that our algorithm is near optimal. 

%\subsection{Proof of Theorem~\ref{thm:worst_case_bound}}\label{subsec:worst_case_bound}

\begin{figure}[t]
	\centering
	\includegraphics[width=0.6\linewidth]{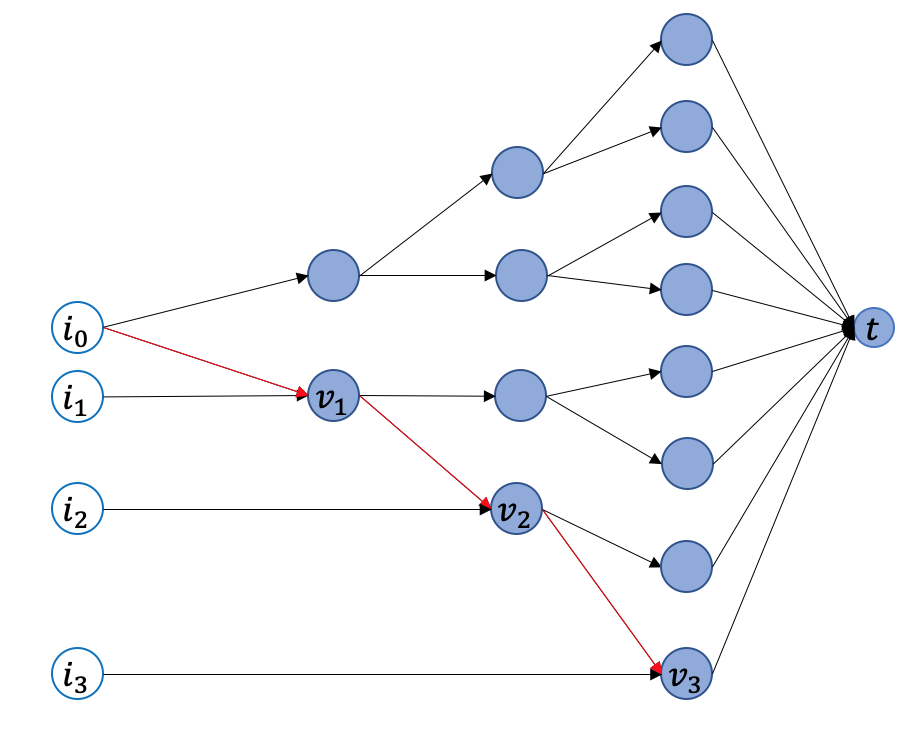}
	\caption{An illustration of the constructed graph with $d=3$. The red path is the path selected by an online algorithm $\cA$ when the first impression arrives. Clearly, the optimal solution sends 4 units flow to $t$ but algorithm $\cA$ only sends one unit. }
	\label{fig:upperbounddag}
\end{figure}

\begin{proof}
	We first construct a directed graph $G(V\cup \{t\},E )$, and then show that given any online algorithm $\cA$, there is a set of impressions such that its competitive ratio is $1/(d+1)$.
	
	As shown in Fig.~\ref{fig:upperbounddag}, the graph has $d+1$ layers. There are exactly two vertices and one vertex (the sink $t$) in the first layer and the last layer respectively. Each vertex except those in the first layer or the last layer is pointed by exactly one vertex in its previous layer, while each vertex except those in the last two layers is adjacent to exactly two vertices in its next layer. All vertices in the $d$-th layer are adjacent to $t$. The capacity of each vertex other than $t$ is $1$.
	
	Clearly, the graph $G$ excluding $t$ consists of two complete binary trees. Now we construct the impression set. The first impression $i_0$ is adjacent to the two vertices in the first layer. For any online algorithm $\cA$, it should select a path to $t$ for this impression. Otherwise, let no impressions arrive any more and the competitive ratio is $0$. Use $(v_1,v_2,...,v_d,t)$ to represent the selected path. Then $d$ impressions $(i_1,i_2,...,i_d)$ arrives sequentially. Each impression $i_k$ is only adjacent to the vertex $v_k$. Since the capacity of $v_k$ is only $1$, $\cA$ sends only one unit of flow to $t$.
	
	However, the optimal solution satisfies all $d+1$ impressions. When impression $i_k$ arrives, the optimal solution can always pick a path that does not contain any vertex in $\{v_{k+1},...,v_d\}$ for it. Thus, the competitive ratio of this instance is $1/(d+1)$, completing this proof. 
\end{proof}

%% file: apdx_scheduling.tex
\section{Learnable and Instance-Robust Predictions for Online Load Balancing} \label{sec:apdx_scheduling}

In this section we prove the results about online load balancing with restricted assignments stated in Section~\ref{sec:load_balancing_results}.  Recall that \cite{PropMatchAgrawal,DBLP:conf/soda/LattanziLMV20} show the existence of useful predictions (machine weights) for this problem.  Thus we focus on studying the instance robustness and learnability of these predictions.

\subsection{Instance Robustness}

%\tl{Chenyang, can you fill this part in?}
Recall that the theorem we want to prove is the following:
\begin{theorem}[Theorem~\ref{thm:makespan-robustness} Restated] \label{thm:makespan-robustness_restate}
	For any instance $S$ and $\epsilon >0$, let $w$ be weights such that $\alg(S,w) \leq (1+\eps)\opt(S)$.  Then for any instance $S'$ we have $\alg(S',w) \leq (1+\eps)^2\eta(S,S')^2\opt(S')$.  
\end{theorem}

\begin{proof}
	The basic idea of this proof is analyzing how much the performance of a set of weights changes when the instance changes.
	For each job type $j$ and machine $i\in N_j$, let $x_{ij}(w)$ be the proportion of this type of job that assigned to machine $i$ using weights $w$.	
	In instance $S$, for each machine $i$, let $L_i(w,S)$ be the load on machine $i$ using weights $w$. Clearly,
	\[ L_i(w,S) = \sum_{j} S_jx_{ij}(w). \]
	
	Now if we turn instance $S$ into a new instance $S'$, the load of each machine increases at most $\eta$:
	\[ L_i(w,S') = \sum_{j} S_j'x_{ij}(w) \leq \eta \sum_{j} S_jx_{ij}(w) = \eta L_i(w,S).\]
	
	Let $w'$ be a $(1+\epsilon)$-approximated weights of the instance $S'$. Similarly, we have
	\[ L_i(w',S) = \sum_{j} S_jx_{ij}(w') \leq \eta \sum_{j} S_j'x_{ij}(w') = \eta L_i(w',S'). \]
	
	According to the near-optimality of weights $w$ on the instance $S$, we have
	\[ \max_{i} L_i(w,S) \leq \max_i L_i(w',S). \]
	
	Thus,
	\begin{equation*}
	\begin{aligned}
	\max_{i} L_i(w,S')& \leq \eta \max_{i} L_i(w,S)\\
	& \leq (1+\epsilon)\eta \max_i L_i(w',S) \\
	&\leq (1+\epsilon)\eta^2 \max_i L_i(w',S')\\
	&\leq (1+\epsilon)^2\eta^2 T, 
	\end{aligned}
	\end{equation*}
	
	completing this proof.

\end{proof}

\subsection{Learnability}

We show that machine weights for makespan minimization are learnable from data in the following formal sense.  There is an unknown distribution $\cD$ over instances of the problem.  A sample $S \sim \cD$ consists of $n$ jobs, where job $j$ has size $p_j$ and neighborhood $N(j) \subseteq [m]$ of machines.  For simplicity, we assume $\cD = \prod_{j=1}^n \cD_j$, i.e. each job is sampled independently from it's own ``private'' distribution and that $p_j = 1$ for all jobs.  Later we show how to generalize to different sizes.  Let $\alg(w,S)$ be the fractional makespan on instance $S$ with weights $w$.  We want to show that we can find weights $w$ given $s$ samples $S_1,S_2,\ldots,S_s$ from $\cD$ such that $\E_{S \sim \cD}[ \alg(w,S)] \leq (1+O(\epsilon)) \E[\opt(S)]$ with high probability (i.e. probability at least $1-\delta$ for $\delta > 0$.  Here $\opt(S)$ is the optimal (fractional) makespan on job set $S$.  Note that such a result also implies that these weights $w$ also satisfy
$\E_{S \sim \cD}[ \alg(w,S)] \leq (1+O(\epsilon)) \min_{w'} \E_{S \sim \cD}[ \alg(w',S)]$ 
with high probability, i.e. they are comparable to the best set of weights for the distribution $\cD$.  Ideally, we want $s = \poly(m,\frac{1}{\epsilon},\frac{1}{\delta})$ number of samples, and lower is better.

\subsubsection{Preliminary Results on Proportional Weights}

We need the following prior results about the weights.  Recall that given a set of jobs $S$ and weights $w \in \R_+^m$ we consider the following fractional assignment rule for job $j$ and $i \in N(j)$.
\begin{equation} \label{eqn:weights_assign}
    x_{ij}(w) := \frac{w_i}{\sum_{i' \in N(j)} w_{i'} }
\end{equation}

For ease of notation we assume that $x_{ij} = 0$ whenever $i \notin N(j)$.  We would like to find weights $w$ such that $x_{ij}(w)$ approximately solves the following LP.

\begin{equation} \label{eqn:matching_lp}
    \begin{array}{ccc}
        \text{maximize} & \displaystyle \sum_{i} \sum_j x_{ij} &  \\
         & \displaystyle \sum_{j} x_{ij} \leq T_i & \forall i \in [m]\\
         & \displaystyle \sum_{i} x_{ij}  \leq 1  & \forall j \in S \\
         & x \geq 0
    \end{array}
\end{equation}

Here, the right hand side values $T_i$ are inputs and can be thought as all being set to the optimal makespan.  Given an assignment via the weights $x_{ij}(w)$ via weights $w$, we can always convert it to a feasible solution to LP~\eqref{eqn:matching_lp} in the following way.  For all $i \in [m]$ let $O_i = \max\{\sum_j x_{ij}(w) / T_i ,1\}$.  It is easy to see that $x'$ is feasible for LP~\eqref{eqn:matching_lp} and that the amount lost is exactly the overallocation $\sum_i \max\{\sum_{j} x_{ij}(w) - T_i,0\}$.  We can then take $x_{ij}' = x_{ij}(w) / O_i$ for all $i,j$.  The following theorem is adapted from Agrawal et al.

\begin{theorem}[Theorem 1 in Agrawal et al.] \label{thm:weights_thm}
For any $\delta \in (0,1)$, there exists an algorithm which finds weights $w$ such that a downscaling of $x_{ij}(w)$ is a $1-\delta$-approximation to LP~\eqref{eqn:matching_lp}.  The algorithm operates in $R = O(\frac{1}{\delta^2} \log( m / \delta))$ iterations and produces weights of the form $w_i = (1+\epsilon)^k$ for $k \in [0,R]$.
\end{theorem}

Using this theorem, we get the following result as simple corollary.  Again let $S$ be set of $n$ jobs that we want to schedule on $m$ machines to minimize the makespan.  Let $T$ be the makespan of an optimal schedule

\begin{corollary} \label{cor:makespan_weights}
For any $\epsilon > 0$, there exists weights $w \in \R_+^m$ such that $x_{ij}(w)$ yields a fractional schedule with makespan at most $(1+\epsilon)T$.  The weights are computed by running for $R= O( m^2 \log(m / \epsilon) / \eps^2$ iterations and produces weights of the form $w_i = (1+\eps)^k$ for $k \in [0,R]$.
\end{corollary}
\begin{proof}
Consider running the algorithm of Theorem~\ref{thm:weights_thm} with $\delta = \eps / m$ and $T_i = T$ for all $i$.  The optimal value of \eqref{eqn:matching_lp} on this instance is exactly $n$ since $T$ is the optimal makespan and thus we are able to assign all the jobs.  After scaling down to be feasible, the solution has value at least $(1- \epsilon / m) n$.  We only scaled down the assignment on machines for which its assignment was greater than $T$, and the amount we lost in this scaling down was at most $\eps n / m$.  Thus in the worst case, any machines assignment using the weights is at most $T + \eps n / m \leq (1+\eps)T$, since $T \geq n / m$.
\end{proof}

\subsubsection{Learning the Weights}

Now we show that computing the weights on a ``stacked'' instance is a reasonable thing to do.  Let's set up some more notation.  Let $\cW(R)$ be the set of possible weights output by $R$ iterations of the proportional algorithm.  Let $\opt(S)$ be the optimal fractional makespan on job set $S$.  We are interested in the case when $\E_{S \sim \cD}[\opt(S)] = \Omega(\log m )$.  Let $L_i(w,S)$ be the fractional load of machine $i$ on instance $S$ with weights $w$.  Thus we have $\alg(w,S) = \max_i L_i(w,S)$.  Note that $L_i(w,S) = \sum_{j \in S} x_{ij}(w)$.  Our first lemma shows that $\E_{S \sim \cD}[\alg(w,S)] \approx \max_i \E_{S \sim \cD}[L_i(w,S)]$.  When it is clear, we will suppress $S \sim \cD$ for ease of notation.

\begin{lemma} \label{lem:commute_avg_max}
    Let $\epsilon > 0$ be given.  If $\E_{S \sim \cD}[\opt(S)] \geq \frac{4+2\eps}{\eps^2} \log( \frac{m}{\sqrt{\eps}})$, then for all $R$ and all weights $w \in \cW(R)$, we have $\E_{S \sim \cD}[\alg(w,S)] \leq (1+ 2\epsilon)\max_i \E_{S \sim \cD}[L_i(w,S)]$.
\end{lemma}
\begin{proof}
Fix any $R$ and $w \in \cW(R)$.  We have the following simply bound on $\alg(w,S)$.  It can either be at most $(1+\eps)\max_i \E[L_i(w,S)]$, or it is larger in which case it is at most $n$.  Thus we have:
\[
\begin{split}
    \E[\alg(w,S)]  \leq & (1+\epsilon) \max_i \E[L_i(w,S)]  \\ & + n\Pr[\alg(w,S) > (1+\epsilon)\max_{i} \E[L_i(w,S)]] \\
     \leq & (1+\epsilon) \max_i \E[L_i(w,S)] + n \sum_i \Pr[L_i(w,S) \geq (1+\epsilon) \E[L_i(w,S)]]
\end{split}
\]
Now we claim that for each $i$, $\Pr[L_i(w,S) \geq (1+\epsilon) \E[L_i(w,S)]] \leq \epsilon / m^2$.  Indeed, if this is the case then we see that 
\[
\E[\alg(w,S)] \leq (1+\epsilon) \max_i \E[L_i(w,S)] + \frac{\epsilon n}{m} \leq (1+2\eps)\max_i \E[L_i(w,S)]  
\]
since $\max_i \E[L_i(w,S)] \geq n / m$, and thus proving the lemma.  Thus we just need to show the claim.  Recall that $L_i(w,S) = \sum_{j \in S} x_{ij}(w)$ and that each job $j$ is chosen to be part of $S$ independently from distribution $\cD_j$.  Thus $x_{ij}(w)$ is an independent random variable in the interval $[0,1]$ for each $j$.  Applying Theorem~\ref{thm:upper_chernoff} to $L_i(w,S)$ with $\mu = \max_{i'} \E[L_{i'}(w,S)]$, we see that since $\mu \geq \E[\opt(S)] \geq \frac{4+2\eps}{\eps^2} \log(\frac{m}{\sqrt{\eps}})$, we have
\[
\Pr[L_i(w,S) > (1+\eps) \mu] \leq \exp \left( - \frac{\eps^2}{2+\eps} \mu\right) \leq \exp \left( - \frac{\eps^2}{2+\eps} \E[\opt(S)] \right) \leq   \frac{\epsilon}{m^2}
\]
completing the proof of the claim.
\end{proof}

Now that we have  this lemma, we can show that computing the weights on a ``stacked'' instance suffices to find weights that generalize for the distribution.  The result we want to prove is the following.

\begin{theorem} \label{thm:learn_the_weights}
Let $\eps,\delta \in (0,1)$ and $R = O(\frac{m^2}{\eps^2}\log(\frac{m}{\eps}))$ be given and let $\cD = \prod_{j=1}^n \cD_j$ be a distribution over $n$-job restricted assignment instances such that $\E_{S \sim \cD}[\opt(S)] \geq \Omega(\frac{1}{\eps^2} \log(\frac{m}{\eps}))$.  There exists an algorithm which finds weights $w \in \cW(R)$ such that $\E_{S \sim \cD}[\alg(w,S)] \leq (1+O(\eps)) \min_{w' \in \cW(R)}\E_{S \sim \cD}[\alg(w',S)]$ when given access to $s = \poly(m,\frac{1}{\eps},\frac{1}{\delta})$ independent samples $S_1,S_2,\ldots,S_s \sim \cD$.  The algorithm succeeds with probability at least $1-O(\delta)$ over the random choice of samples.
\end{theorem}

We will show that \emph{uniform convergence} occurs when we take $s = \poly(m,\frac{1}{\eps},\frac{1}{\delta})$ samples.  This means that for all $i \in [m]$ and all $w \in \cW(R)$ simultaneously, we have with probability $1-\delta$ that $\frac{1}{s} \sum_{\alpha} L_i(w,S_\alpha) \approx \E_S[L_i(w,S)]$.  Intuitively this should happen because the class of weights $\cW(R)$ is not too complex.  Indeed we have that $|\cW(R)| = R^m$, and thus the pseudo-dimension is $\log(|\cW(R)|)= m \log(R)= O(m \log m)$ when $R = O(m^2 \log m)$.  Once we have established uniform convergence, setting up the algorithm and analyzing it will be quite simple.  We start with some lemmas showing uniform convergence.

% This version is slightly different to get better sampling
\begin{lemma} \label{lem:loads_converge}
Let $\eps,\delta \in (0,1)$ and $S_1,S_2,\ldots,S_s \sim \cD$ be independent samples.  If $s \geq 
\frac{m^2}{\eps^2}\log(\frac{2|\cW(R)|m}{\delta})$, 
then with probability at least $1-\delta$ for all $i \in [m]$ and $w \in \cW(R)$ we have 
\[
\left| \frac{1}{s} \sum_{\alpha} L_i(w,S_\alpha) - \E_S[L_i(w,S)] \right| \leq \epsilon \max_{i'} \E_S[L_{i'}(w,S)]
\]
\end{lemma}
\begin{proof}
Fix a machine $i \in [m]$ and $w \in \cW(R)$.  
We have that $\frac{1}{s}L_i(w,S_\alpha)$ is an independent random variable in $[0,n/s]$ for each $\alpha \in [s]$.  Moreover we have that $\E[ \frac{1}{s}\sum_\alpha L_i(w,S_\alpha)] = \E[L_i(w,S)]$.  Applying Theorem~\ref{thm:hoeffding} to $\frac{1}{s}\sum_{\alpha} L_i(w,S_\alpha)$ with $t = \eps \max_{i'}\E[L_{i'}(w,S)$, we have

\[
\Pr\left[ \left|\frac{1}{s}\sum_{\alpha} L_i(w,S_\alpha) -\E[L_i(w,S)] \right| \geq t \right] \leq 2\exp\left(- s\frac{\eps^2 (\max_{i'}\E[L_{i'}(w,S)])^2}{n^2} \right).
\]
We claim that if $s \geq \frac{m^2}{\eps^2}\log(\frac{2|\cW(R)|m}{\delta})$, then this probability is at most $\frac{\delta}{|\cW(R)|m}$.  Indeed, this claim follows if $m \geq \frac{n}{\max_{i'}\E[L_{i'}(w,S)]}$, which is true since $\max_{i'}\E[L_{i'}(w,S)] \geq n /m$.  Finally, the lemma follows by union bounding over all $i \in [m]$ and $w \in \cW(R)$.
\end{proof}

\begin{lemma} \label{lem:opts_converge}
Let $\eps, \delta \in (0,1)$ and $S_1,S_2,\ldots,S_s \sim \cD$ be independent samples.  For each $\alpha \in [s]$ let $T_\alpha = \opt(S_\alpha)$.  If $s \geq \frac{m^2}{\eps^2}\log(2/\delta)$ then with probability at least $1-\delta$ we have
\[
(1-\eps)\E[\opt(S)] \leq \frac{1}{s} \sum_{\alpha} T_\alpha \leq (1+\eps)\E[\opt(S)]
\]
\end{lemma}
\begin{proof}
For each $\alpha \in [s]$ we have $\frac{1}{s}T_\alpha$ is an independent random variable in $[0,n/s]$.  Moreover, we have $\E[\frac{1}{s}\sum_\alpha T_\alpha] = \E[\opt(S)]$.  Applying Theorem~\ref{thm:hoeffding} to $\frac{1}{s} \sum_\alpha T_\alpha$ we have
\[
\Pr\left[ |\frac{1}{s} \sum_\alpha T_\alpha - \E[\opt(S)]| \geq \eps\E[\opt(S)] \right] \leq 2 \exp \left( - s\frac{\eps^2\E[\opt(S)]^2}{n^2} \right).
\]
Now since $\E[\opt(S)] \geq n / m$ we have this probability is at most $2 \exp(- s\eps^2 / m^2)$.
Thus whenever $s \geq \frac{m^2}{\eps^2}\log(2/\delta)$, this probability becomes at most $\delta$, completing the proof.
\end{proof}

\subsubsection{The Learning Algorithm}

We can now describe and analyze the algorithm.  Set $R = O(\frac{m^2}{\eps^2}\log(m /\eps)$.  We sample independent instances $S_1,S_2,\ldots,S_s \sim \cD$ for 
%$s \geq \max\{\frac{m^2}{\eps^2}\log(\frac{2|\cW(R)|m}{\delta}) , \frac{m^2}{\eps^2}\log(2/\delta) \}$.  
$s \geq \frac{m^2}{\eps^2}\log(\frac{2|\cW(R)|m}{\delta})$
Next we set up a stacked instance consisting of all the jobs in these samples.  Next we set $T_\alpha = \opt(S_\alpha)$ and $T = \sum_{\alpha}T_\alpha$.  We run the algorithm of Corollary~\ref{cor:makespan_weights} on the stacked instance with right hand side bounds $T_i = T$ for all $i$.  The algorithm should run for $R$ rounds and produce weights $w \in \cW(R)$ such that $\sum_{\alpha} L_i(w,S_\alpha) \leq (1+\eps) T$ for all $i$.  We can now prove Theorem~\ref{thm:learn_the_weights}.

\begin{proof}[Proof of \textnormal{\bf Theorem~\ref{thm:learn_the_weights}}]
Let $w \in \cW(R)$ be the weights output by the algorithm above.  Now for a new randomly sampled instance $S\sim \cD$, by Lemma~\ref{lem:commute_avg_max} we have that
\[
\E[\alg(w,S)] \leq (1+2\eps) \max_i \E[L_i(w,S)].
\]
By Lemma~\ref{lem:loads_converge}, we have $\max_i \E[L_i(w,S)] \leq (1+O(\eps)) \max_i \frac{1}{s} \sum_{\alpha}L_i(w,S_\alpha)$ with probability at least $1-\delta$.  By construction of our algorithm, we have $\sum_{\alpha} L_i(w,S_\alpha) \leq (1+\eps)T =(1+\eps)\sum_{\alpha}T_\alpha$ for all $i$.  It thus follows that $\max_i \E[L_i(w,S)] \leq(1+O(\eps)) \frac{1}{s} \sum_\alpha T_{\alpha}$ with probability at least $1-\delta$.  Next we have that $\frac{1}{s}\sum_{\alpha}T_\alpha \leq (1+\eps)\E[\opt(S)]$ with probability at least $1-\delta$.  Finally, with probability at least $1-2\delta$, by chaining these inequalities together we get
\[
\E[\alg(w,S)] \leq (1+O(\eps)) \E[\opt(S)] \leq (1+O(\eps)) \E[\alg(w^*,S)]
\]
where $w^* = \arg\min_{w' \in \cW(R)} \E[\alg(w^*,S)]$.  Since $R = O(\frac{m^2}{\eps^2}\log(\frac{m}{ \epsilon}))$, we have that $\log(|\cW(R)|) = m\log(R) = O(m\log (m /\eps))$.  Thus we can take $s = \poly(m,\frac{1}{\eps},\frac{1}{\delta})$ to get the result.  This completes the proof.
\end{proof}

\subsubsection{Handling Different Sizes}

Now we give a sketch of how to handle the case when each job has an integer size $p_j > 0$.  For this we need a slightly different version of Theorem~\ref{thm:weights_thm} and Corollary~\ref{cor:makespan_weights}.  Consider the following variant of LP~\ref{eqn:matching_lp}:

\begin{equation} \label{eqn:sizes_matching_lp}
    \begin{array}{ccc}
        \text{maximize} & \displaystyle \sum_{j}p_j \sum_i x_{ij} &  \\
         & \displaystyle \sum_{j} p_j x_{ij} \leq T_i & \forall i \in [m]\\
         & \displaystyle \sum_{i} x_{ij}  \leq 1  & \forall j \in S \\
         & x \geq 0
    \end{array}
\end{equation}
Again we simplify notation and assume $x_{ij} = 0$ whenever $i \notin N(j)$.  The following is a corollary of Theorem~\ref{thm:weights_thm}.  Let $T$ be the optimal makespan for a set of jobs

\begin{corollary} \label{cor:makespan_weights_sizes}
For any $\epsilon > 0$, there exists weights $w \in \R_+^m$ such that $x_{ij}(w)$ yields a fractional schedule with makespan at most $(1+\epsilon)T$.  The weights are computed by running a variant of the algorithm of Theorem~\ref{thm:weights_thm} for $R= O( m^2 \log(m / \epsilon) / \eps^2$ iterations and produces weights of the form $w_i = (1+\eps)^k$ for $k \in [-R,R]$.
\end{corollary}
\begin{proof}
Consider creating $p_j$ unit-sized copies of each job $j$.  Note that this only needs to be done conceptually.  It is easy to see that writing down LP~\ref{eqn:matching_lp} for this conceptual instance is a relaxation of LP~\ref{eqn:sizes_matching_lp}.  Consider running the Algorithm of Theorem~\ref{thm:weights_thm} with $\delta = \eps / m$, $T_i = T$ for all $i \in [m]$ and for $R = O(\frac{1}{\delta^2} \log(m / \delta)$ iterations.  Note that since $T$ is the optimal makespan, there exists a solution with value $\sum_j p_j$.  Thus since the algorithm returns a $(1-\delta)$-approximation, we get a solution with value at least $(1- \delta) \sum_j p_j$.  The amount that we lose in the objective is exactly the total amount over-allocated in the solution given by the weights.  Thus for all $i$, since $T \geq \sum_j p_j / m$. we have
\[
\sum_j p_j x_{ij}(w) \leq T + \delta \sum_j p_j = T+ \eps \frac{\sum_j p_j}{m} \leq (1+\eps)T.
\]
\end{proof}

Our learning algorithm will be the same as before, just the jobs will now have sizes.  We go through each lemma above and prove an analogous version for when there are job sizes.  For a job set $S$ and weights $w \in \cW(R)$ let $L_i(w,S) = \sum_j p_j x_{ij}(w)$.  Let $p_{\max} = \max_j p_j$ be the maximum job size.  For this case our assumption becomes $\E_{S \sim \cD}[\opt(S)] \geq \frac{4+2\eps}{\eps^2}p_{\max} \log( \frac{m}{\sqrt{\eps}})$.

\begin{lemma} \label{lem:commute_avg_max_sizes}
    Let $\epsilon > 0$ be given.  If $\E_{S \sim \cD}[\opt(S)] \geq \frac{4+2\eps}{\eps^2}p_{\max} \log( \frac{m}{\sqrt{\eps}})$, then for all $R$ and all weights $w \in \cW(R)$, we have $\E_{S \sim \cD}[\alg(w,S)] \leq (1+ 2\epsilon)\max_i \E_{S \sim \cD}[L_i(w,S)]$.
\end{lemma}
\begin{proof}
Fix any $R$ and $w \in \cW(R)$.  We have the following simply bound on $\alg(w,S)$.  It can either be at most $(1+\eps)\max_i \E[L_i(w,S)]$, or it is larger in which case it is at most $\sum_j p_j$.  Thus we have:
\[
\begin{split}
    \E[\alg(w,S)]  \leq & (1+\epsilon) \max_i \E[L_i(w,S)]  \\ & +\sum_j p_j\Pr[\alg(w,S) > (1+\epsilon)\max_{i} \E[L_i(w,S)]] \\
     \leq & (1+\epsilon) \max_i \E[L_i(w,S)] + \sum_j p_j \sum_i \Pr[L_i(w,S) \geq (1+\epsilon) \E[L_i(w,S)]]
\end{split}
\]
Now we claim that for each $i$, $\Pr[L_i(w,S) \geq (1+\epsilon) \E[L_i(w,S)]] \leq \epsilon / m^2$.  Indeed, if this is the case then we see that 
\[
\E[\alg(w,S)] \leq (1+\epsilon) \max_i \E[L_i(w,S)] + \frac{\epsilon \sum_j p_j}{m} \leq (1+2\eps)\max_i \E[L_i(w,S)]  
\]
since $\max_i \E[L_i(w,S)] \geq \sum_j p_j / m$, and thus proving the lemma.  Thus we just need to show the claim.  Recall that $L_i(w,S) = \sum_{j \in S} p_jx_{ij}(w)$ and that each job $j$ is chosen to be part of $S$ independently from distribution $\cD_j$.  Thus $p_jx_{ij}(w)$ is an independent random variable in the interval $[0,p_{\max}]$ for each $j$.  Applying Theorem~\ref{thm:upper_chernoff} to $L_i(w,S) / p_{\max}$ with $\mu = \max_{i'} \E[L_{i'}(w,S)] / p_{\max}$, we see that since $\mu \geq \E[\opt(S)] / p_{\max} \geq \frac{4+2\eps}{\eps^2} \log(\frac{m}{\sqrt{\eps}})$, we have
\[
\Pr\left[\frac{L_i(w,S)}{p_{\max}} > (1+\eps) \mu \right] \leq \exp \left( - \frac{\eps^2}{2+\eps} \mu\right) \leq \exp \left( - \frac{\eps^2}{2+\eps} \E[\opt(S)] \right) \leq   \frac{\epsilon}{m^2}
\]
which implies the claim.
\end{proof}

Modifying the remaining lemmas is simple.  We can do this by replacing most instances of $n$ in the proofs with $\sum_j p_j$.

\subsubsection{Inequalities}

\begin{theorem}[Upper Chernoff Bound] \label{thm:upper_chernoff}
Let $X_1,X_2,\ldots,X_n$ be independent random variables with $X_i \in [0,1]$ for $1 \leq i \leq n$.  Let $X = \sum_i X_i$ and $\mu \geq E[X]$, then for all $\eps > 0$ we have
\[
\Pr[X \geq (1+\eps)\mu] \leq \exp \left(- \frac{\epsilon^2}{2+\epsilon} \mu \right)
\]
\end{theorem}

\begin{theorem}[Two-Sided Hoeffding Bound] \label{thm:hoeffding}
Let $X_1,X_2,\ldots,X_n$ be independent random variables with $a_i \leq X_i \leq b_i$ for $1 \leq i \leq n$.  Let $X = \sum_i X_i$ and $\mu = \E[X]$.  Then for all $t > 0$ we have
\[
\Pr[ |X - \mu| \geq t] \leq 2\exp \left( - \frac{t^2}{\sum_i (b_i-a_i)^2}\right)
\]
\end{theorem}